\documentclass[letterpaper]{article} 
\usepackage{aaai2026}  
\usepackage{times}  
\usepackage{helvet}  
\usepackage{courier}  
\usepackage[hyphens]{url}  
\usepackage{graphicx} 
\usepackage{natbib}  
\usepackage{caption} 
\usepackage{algorithm}

\usepackage{newfloat}
\usepackage{listings}
\DeclareCaptionStyle{ruled}{labelfont=normalfont,labelsep=colon,strut=off} 
\floatstyle{ruled}
\newfloat{listing}{tb}{lst}{}
\floatname{listing}{Listing}
\usepackage{amsmath,amsthm,bm}
\usepackage{amssymb}
\usepackage{dsfont}
\usepackage{mathtools}
\usepackage{multirow}
\usepackage{algorithm}%
\usepackage{algpseudocode}
\usepackage[table,xcdraw]{xcolor}
\usepackage{booktabs}
\usepackage{makecell}

\definecolor{darkyellow}{RGB}{249, 199, 12} 
\title{Beyond Hard Constraints: Budget-Conditioned Reachability \\ For 
Safe Offline Reinforcement Learning}
\author{
      Janaka Brahmanage and Akshat Kumar 
}
\affiliations{
School of Computing and Information Systems, Singapore Management University.\\
janakat.2022@phdcs.smu.edu.sg, akshatkumar@smu.edu.sg
}

\usepackage{bibentry}

\theoremstyle{plain}
\newtheorem{theorem}{Theorem}[section]

\newtheorem{lemma}[theorem]{Lemma}

\newtheorem{definition}[theorem]{Definition}

\theoremstyle{remark}

\newcommand{\blue}[1]{\textcolor{blue}{#1}}
\newcommand{\modified}[1]{#1}
\newcommand{\green}[1]{\textcolor{green}{#1}}
\newcommand{\gray}[1]{\textcolor{gray}{#1}}
\newcommand{\mc}{\mathcal} 
\newcommand{\cta}{\ct_{\text{init}}} 
\newcommand{\ct}{\delta}
\newcommand{\ctmax}{\delta_{\text{max}}}
\newcommand{\qc}{Q^{C}}
\newcommand{\qr}{Q^{R}}
\newcommand{\vc}{V^{C}}
\newcommand{\vr}{V^{R}}

\newcommand{\loss}{\mathcal{L}}

\newcommand{\dataset}{\mathcal{D}}

\newcommand{\tauc}{\tau_{C}}
\newcommand{\taur}{\tau_{R}}

\newcommand{\cmax}{c_{\text{max}}}

\DeclareMathOperator*{\E}{\mathop{\mathbb{E}}}

\newcommand{\our}{BCRL}
\newcommand{\ourmdp}{BAMDP}

\newcommand{\T}{T}
\newcommand{\vcs}{V_C^*}
\newcommand{\qcs}{Q_C^*}
\newcommand{\Sp}{S_P}
\newcommand{\Ap}{A_P}
\newcommand{\PiP}{\Pi_P}
\newcommand{\SbP}{\bar{S}_P}
\newcommand{\MD}{\bar{\mc{M}}_{\text{D}}}
\newcommand{\MS}{\bar{\mc{M}}_{\text{S}}}

\definecolor{darkyellow}{RGB}{249, 199, 12}

\def\1{\bm{1}}
\def\sR{{\mathbb{R}}}

\begin{document}

\maketitle

\begin{abstract}
Sequential decision-making using Markov Decision Process underpins many real-world applications. Both model-based and model-free methods have achieved strong results in these settings. However, real-world tasks must balance reward maximization with safety constraints, often conflicting objectives, that can lead to unstable min–max, adversarial optimization. A promising alternative is \textit{safety reachability} analysis, which precomputes a forward-invariant safe state–action set, ensuring that an agent starting inside this set remains safe indefinitely. Yet, most reachability-based methods address only hard safety constraints, and little work extends reachability to cumulative cost constraints. To address this, \textit{first}, we define a safety-conditioned reachability set that \textit{decouples} reward maximization from  cumulative safety cost constraints. \textit{Second}, we show how this set enforces safety constraints without unstable min–max or Lagrangian optimization, yielding a novel offline safe RL algorithm that learns a safe policy from a fixed dataset without environment interaction. \textit{Finally}, experiments on standard offline safe-RL benchmarks, and a real-world maritime navigation task demonstrate that our method matches or outperforms state-of-the-art baselines while maintaining safety.
\end{abstract}


\section{Introduction}

Markov Decision Processes (MDPs) have demonstrated impressive success modeling across a range of domains, including robotics~\cite{tang_deep_2024}, algorithm discovery~\cite{fawzi_discovering_2022}, classical board games~\cite{silver_mastering_2016}, and Atari games~\cite{mnih_playing_2013}.
Despite these achievements, deploying agents in real-world settings poses significant challenges. In practice, most of the applications utilize a Reinforcement Learning (RL) framework to train agents, and these agents must not only maximize cumulative reward but also operate safely~\cite{altman_constrained_2021}, requiring policies that strike a balance between performance and safety.
Additionally, constructing fast and accurate simulators for complex environments is often computationally prohibitive. Offline safe reinforcement learning (\textit{offline safe RL}) offers a practical solution by enabling agents to learn from pre-collected datasets without further environment interaction~\cite{liu_datasets_2024}. This work aims to develop methods that address the dual goals of reward optimization and safety in offline RL, an important problem in RL research~\cite{gu_review_2024}.

\paragraph{Approach overview}

\modified{
Our goal is to solve the standard CMDP problem~\cite{altman_constrained_2021}, where the agent must maximize reward subject to a constraint on the expected cumulative cost. We propose an approach that augments the state with a dynamic budget. This dynamic budget is a quantity that controls the conservativeness of the policy; a smaller budget implies a more conservative policy. For example, we can start with the total cost threshold of the CMDP and track the exact remaining budget by subtracting immediate costs incurred as the policy executes. In a deterministic setting, we can enforce the safety constraint by tracking such a budget.
However, this does not work in a stochastic environment. Instead, we show that tracking a quantity other than the exact remaining budget can enforce safety in stochastic settings. We then estimate a budget-conditioned \textit{persistent} safety set: the set of state-action pairs from which there exists a policy that keeps future costs within the remaining budget. By restricting the agent's actions to this safety set, we decouple safety enforcement from reward optimization, avoiding the instability of Lagrangian or min-max methods.
}

\paragraph{Related work} 
Constrained Markov decision processes (CMDPs) \cite{altman_constrained_2021} are a standard framework for safe RL. Prior offline safe RL methods often struggle with optimization instability or high computational overhead. Lagrangian methods like BCQ-Lagrangian~\cite{liu_datasets_2024} suffer from tuning difficulties and unstable learning. Methods like COptiDICE~\cite{lee_coptidice_2022} perform distribution correction but struggle empirically on recent benchmarks~\cite{liu_datasets_2024}. Recent approaches learn generative models (VAEs) as a pre-training step~\cite{koirala_latent_2025,guo_constraint-conditioned_2025}, which incurs significant computational overhead. 
While Hamilton-Jacobi reachability~\cite{yu_reachability_2022} enforces safety independently of the policy for hard constraints, action-constrained RL ensures per-step feasibility through optimization solvers~\cite{lin_escaping_2021} or generative models~\cite{brahmanage_flowpg_2023,brahmanage_leveraging_2024}.
Another related approach, Saute RL~\cite{sootla_saute_2022}, enforces per-step constraints but requires online rollouts to track budgets. In contrast, our approach prunes unsafe actions offline using a dynamic budget without requiring a generative model or online samples. We provide an extended discussion of related work in Appendix~\ref{app:related_work}.

\paragraph{Adaptive Safty Budgets for RL} 
\modified{
Standard safe RL methods enforce static, episode-level constraints (i.e., \textit{expected} cost over all the episodes is less than a budget). In contrast, we introduce a step-wise budget that dynamically adjusts during policy execution. A key benefit is that using this step-wise budget, we can prune unsafe actions at each time step and guide value estimation (for both reward and cost), enabling policy that adapts over time. 
}

\noindent\textbf{Our contributions are:}
\begin{itemize}
\item We propose \textbf{Budget-Conditioned Reachability}, a framework that applies reachability analysis to CMDPs with cumulative cost constraints. It uses dynamic budgets to estimate persistently safe state--action sets, enabling reward-maximizing policy learning within this safe region. We also provide rigorous theoretical justification for the framework.

\item We introduce two variants of our method to address both deterministic and stochastic CMDP settings, ensuring broad applicability across different problem structures.

\item Our method integrates seamlessly with existing offline RL algorithms such as IQL~\cite{wang_offline_2023}, XQL~\cite{garg_extreme_2023}, and SparseQL~\cite{xu_offline_2023}, yielding a safe offline RL approach, \textbf{BCRL} (Budget-Conditioned Reachability RL). The resulting algorithms are easy to implement, require no min--max adversarial training, never query out-of-distribution actions, and generalize to any budget.

\item We evaluate \our\ on interpretable grid-world environments, DSRL benchmarks~\cite{liu_datasets_2024}, and a real-world maritime navigation task, where agents are trained using historical ship trajectory data to navigate safely in high-traffic conditions. 
Extensive ablations show consistent improvements over state-of-the-art baselines. Our code is available at\footnote{\url{https://janakact.github.io/bcrl}}.
\end{itemize}

\section{Preliminaries}\label{sec:preliminaries}

\subsection{Constrained Markov Decision Process}\label{sec:cmdp}
A constrained Markov Decision Process (CMDP) is defined as a tuple \( \mc{M} := \langle S, A, \T, r, c, \gamma, \mu_0, \cta \rangle \). Here, \( S \) and \( A \) represent the state and action spaces, respectively. The transition dynamics are given by \( \T(s' \!\mid\! s, a) \), specifying the probability of transitioning from state \( s \) to \( s' \) under action \( a \). The reward function is {\( r(s, a): S\times A \rightarrow \mathbb{R}  \)}, and the cost function is {\( c(s, a) : S\times A \rightarrow \mathbb{R}^+ \)}. The discount factor is \( \gamma \in [0, 1) \), and the initial state distribution is \( \mu_0: S \rightarrow [0, 1] \). The $\cta$ is the cost threshold over the total expected cost agent can incur under any policy. 
\textbf{The objective} is to determine a policy \( \pi : S \to \Delta(A) \) that maximizes the expected discounted reward  
\(
J_R(\pi) 
\), subject to the cost constraint $J_C(\pi)\leq \cta$ where, 
{
\begin{equation}
\label{eq:cmdp_cost}
J_R(\pi)\! =\! \E_{\pi} 
\sum_{t=0}^\infty \gamma^t r(s_t, a_t) 
 ;\;\; 
 J_C(\pi) \!=\! \E_{\pi} 
\sum_{t=0}^\infty \gamma^t c(s_t, a_t) 
\end{equation}
}
and \( \cta \in [0, \ctmax] \) is the cost (or safety budget) threshold. Here $\ctmax = \frac{\cmax}{1-\gamma}$, where $\cmax$ is the highest cost the environment can incur at a single step.

\begin{figure}
    \centering
    \includegraphics[width=0.8\linewidth]{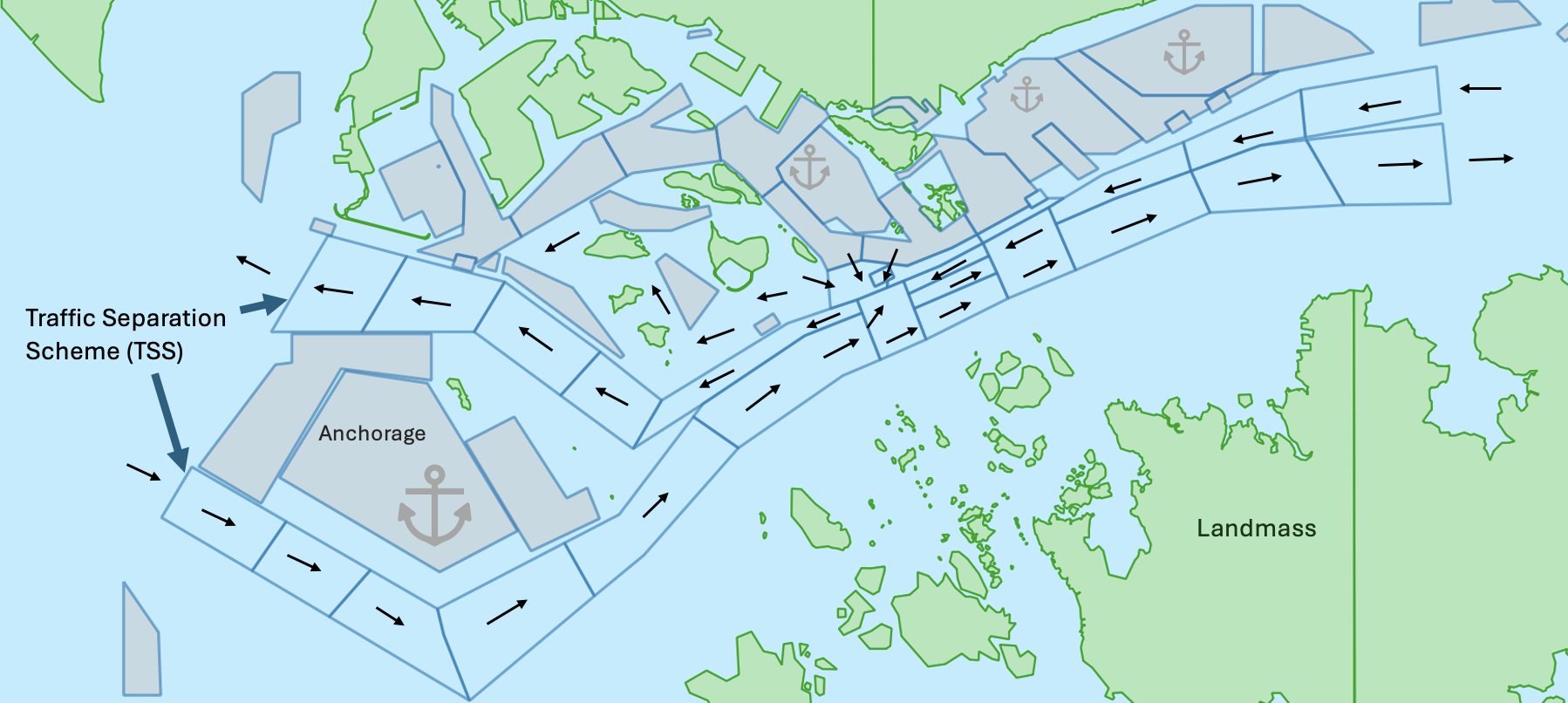}
    \caption{\small Electronic navigation chart for maritime traffic navigation in Singapore strait}
    \label{fig:enc}
\end{figure}

\paragraph{Learning from offline data} For CMDPs, getting faithful domain model (e.g., transition, reward, cost functions) for realistic problems is challenging. Thus, we focus on the offline RL setting where we learn from a pre-collected transition dataset \( \dataset = \{(s_i, a_i, r(s_i,a_i), c(s_i,a_i), s'_i)\}_{i=1}^N \) generated by one or more behavioral policies $\pi_\beta$. Offline safe RL aims to learn a policy \( \pi \) that optimizes CMDP objective subject to cost constraints using this dataset alone, without further environment interaction. As a real world example, consider the safe maritime navigation problem in figure~\ref{fig:enc}. The Singapore strait is a high traffic, constrained waterway, with vessels navigating narrow lanes, TSS corridors, anchorage zones, and areas near landmasses. Ships must make sequential decisions--adjusting speed or heading--while avoiding restricted zones, minimizing collision risk, and following navigation rules.
Because unsafe exploration is impossible in real waters, learning through trial-and-error is not feasible. Yet, the sea traffic generates extensive AIS (automatic identification system) data on vessel positions, speeds, headings, and interactions. This makes safe offline RL a natural fit for such problems where safe policies are learned solely from historical data without real-world experimentation.

\subsection{In-Sample Q-learning Algorithms}
\label{sec:insample_rl}
\modified{For unconstrained reinforcement learning (i.e., standard RL without any cost constraints), several algorithms use asymmetric loss functions to estimate the value function instead of the exact maximum, learning policies from offline data without querying Q-values of unseen actions. Methods such as IQL~\cite{kostrikov_offline_2021}, XQL~\cite{garg_extreme_2023}, and SparseQL~\cite{xu_offline_2023} employ asymmetric regression to fit the Bellman optimal value, replacing the max operator while still following temporal-difference learning. The critic \(Q_\theta(s,a)\), target critic \(Q_{\theta_T}(s,a)\), and value function \(V_\psi(s)\), parameterized by \(\theta\), \(\theta_T\), and \(\psi\), are trained by minimizing the following losses:
}
{
\begin{equation}
\mc{L}^V(\psi) = \E_{(s,a)\sim \dataset}
[\mathcal{L}^*(Q_{\theta_T}(s,a) - V_{\psi}(s))]
\label{eq:iql_v}
\end{equation}
\begin{equation}
\label{eq:iql_q}
\mc{L}^Q(\theta) = 
\E_{(s,a,s')\sim \dataset}
[(r(s,a) + \gamma V_\psi(s') - Q_{\theta}(s,a))^2]
\end{equation}
}

\modified{Here, $L^*$ is an asymmetric loss function that penalizes positive and negative errors differently, weighting overestimation more heavily to approximate the maximum Q-value. In IQL~\cite{kostrikov_offline_2021}, $L^*(u)=L^\tau(u)=|\tau - \1_{\{u < 0\}}|u^2$ is the expectile loss~\cite{newey_asymmetric_1987}, with $\tau > 0.5$ pushing $V_\psi(s)$ toward the maximum. Similarly, XQL~\cite{garg_extreme_2023} uses Gumbel regression, and SparseQL~\cite{xu_offline_2023} introduces a robust asymmetric loss.}
Target parameters are updated via $\theta_T \leftarrow (1-\alpha)\theta_T + \alpha \theta$. 
These methods are computationally efficient, avoid querying OOD actions, and have strong theoretical guarantees~\cite{kostrikov_offline_2021, xu_offline_2023}. For policy extraction, they use in-sample methods similar to Advantage Weighted Regression (AWR)~\cite{peng_advantage-weighted_2019}. In this work, we adopt primarily in-sample methods to avoid OOD actions.

\subsection{Persistent Safety with Value Functions}\label{sec:pre-persistent}
In safe RL, we sometimes have access to a safety indicator for a given state, which can be viewed as a state-wise safety constraint. However, it is not sufficient to ensure only \emph{instantaneous safety}. When an agent takes an action, we also need to guarantee that in the future the agent will always have at least one safe action available. In other words, the agent should not be driven into an unrecoverable situation or a dead end.
Formally, a state is unrecoverable if there exists no sequence of future actions that can prevent an eventual safety violation.
A safety signal that indicates the agent is not in an unrecoverable state is referred to as a \emph{persistent safety signal}~\cite{yu_reachability_2022}.
Hamilton–Jacobi (HJ) reachability-based RL methods~\cite{zheng_safe_2024,yu_reachability_2022,ganai_iterative_2023} address this by estimating a \emph{forward-invariant safety set}. This set consists of states from which the agent can remain safe indefinitely, provided it follows an appropriate policy. 
For example, in FISOR~\cite{zheng_safe_2024}, the immediate safety indicator $h(s)$ defines the safe state set
\(
S_f := \{ s \in S \mid h(s) \leq 0 \}.
\)
A learned value function is then used to estimate
\begin{align}
V^*_h(s) &= \min_{\pi} V_h^{\pi}(s) 
= \min_{\pi} \max_{t \in \mathbb{N}} h(s_t) 
\end{align}
If $V^*_h(s) \leq 0$, this implies that there exists a policy $\pi$ such that the future cost satisfies $V_h^{\pi}(s) \leq 0$. Therefore, $V_h^*$ acts as a persistent safety indicator. The set
\(
S_p := \{ s \in S \mid V_h^*(s) \leq 0 \}
\)
is then used to define the persistent safety set.

\section{Budget-Conditioned Reachability}

Prior reachability and persistent safety estimation methods are mostly limited to hard constraints (i.e., whether safety indicator $h(s)\leq 0$). They do not extend trivially to cumulative cost based safety constraints in CMDPs defined in Section~\ref{sec:cmdp}. Thus, we aim to estimate the persistent safety set \textit{independently} of the policy being learned, using it to enforce cumulative cost constraints. Decoupling safety estimation from policy optimization ensures the learned policy remains safe while making learning more stable and tractable, without using Lagrangian or min-max adversarial optimization. Our approach builds on three key ideas: \textbf{(1)} augmenting the CMDP to explicitly track the remaining budget, \textbf{(2)} estimating a budget-conditioned persistent safety set for \textit{any} remaining budget, and \textbf{(3)} solving the augmented CMDP.

\noindent \modified{ \textbf{Methodological Intuition:} Rather than enforcing a static constraint over an entire episode, we augment the state with a remaining budget that dynamically depletes as costs are incurred. By pre-computing whether this budget is sufficient for safe completion (the persistent safety set), we can actively prune unsafe actions at each step before the agent acts, entirely decoupling safety enforcement from reward maximization.
}

\subsection{A Budget-Conditioned Persistent Safety Set}

Similar to the definition of persistent safety discussed in Section~\ref{sec:pre-persistent}, we now define the budget-conditioned safety set that indicate feasibility based on the discounted future cost. 
\begin{definition}[The optimal cost-value function] The optimal state-cost-value function $\vcs$ and the optimal action-cost-value function $\qcs$ are defined as:
\begin{align}
\vcs(s) &= \min_{\pi} V_C^{\pi}(s) 
= \min_{\pi} \E_{
\substack{
\pi \\
s_0 = s}
} \sum_{t \in \mathbb{N}} \gamma^t c(s_t,a_t),  \label{eq:def_vcs}\\
\qcs(s,a) &= \min_{\pi} Q_C^{\pi}(s,a) 
= \min_{\pi} 
 \E_{
\substack{
 \pi \\
s_0 = s\\ a_0=a}
}
\sum_{t \in \mathbb{N}} \gamma^t c(s_t,a_t) \label{eq:def_qcs}
\end{align}
\end{definition}

Following a similar reasoning as in Section~\ref{sec:pre-persistent}, we claim that, given a budget $\ct$, if $\vcs(s) \leq \ct$, then there exists a policy $\pi$ such that the future discounted cost satisfies $V_C^{\pi}(s) \leq \ct$. Therefore, $\vcs(s) \leq \ct$ serves as a persistent safety indicator for the given budget, guaranteeing that there exists a policy under which the future cost remains within the budget. This allows us to define the {largest set} of feasible states, as well as the feasible actions for each state, that remain within the given budget.

\begin{definition}[Budget-Conditioned Persistent Safety Sets]
Given a budget $\ct$, the \emph{budget-conditioned persistent safety set} is defined as:
\begin{align}\label{eq:def_sp}
\Sp(\ct) := \{ s \in S \mid \vcs(s) \leq \ct \}.
\end{align}
Similarly, the \emph{budget-conditioned persistent safe action set} for a given state $s$ is defined as:
\begin{align}\label{eq:def_ap}
\Ap(s,\ct) := \{ a \in A \mid \qcs(s,a) \leq \ct \}.
\end{align}
\end{definition}

Intuitively, the budget-conditioned persistent safety set consists of states where a policy exists to keep the total discounted future cost within the budget $\ct$. 

\begin{lemma}[Safe Actions Always Exist for Persistent Safety States]\label{lemma:safe-actions-exist}
Given $\ct \in \sR^+$; for any state $s \in \Sp(\ct)$, the budget-conditioned persistent safe action set $\Ap(s,\ct)$ is non-empty:
\(
\Ap(s,\ct) \neq \emptyset.
\)
\end{lemma}

\begin{proof}
Let $s \in \Sp(\ct)$, which means by definition
\(
\vcs(s) \leq \ct.
\)
By the definition of the value function,
{\small
\[
\vcs(s) = \min_{\pi} \mathbb{E}\Big[\sum_{t=0}^\infty \gamma^t c(s_t, a_t) \,\big|\, s_0 = s, \pi\Big] = \min_{a \in A} \qcs(s,a),
\]
}
so there exists at least one action $a^* \in A$ that attains this minimum. Therefore,
\(
\qcs(s, a^*) = \vcs(s) \leq \ct,
\)
which implies $a^* \in \Ap(s, \ct)$. Hence,
\(
\Ap(s, \ct) \neq \emptyset.
\)
\end{proof}

In the following sections, we discuss how the budget-conditioned persistent safety sets defined above can be used to enforce the cumulative cost constraint of the CMDP.

\newcommand{\ctinit}{f}
\newcommand{\ctupdate}{g}
\newcommand{\ctdomain}{\mathbb{R}^+}
\newcommand{\bS}{\bar{S}}
\newcommand{\bs}{\bar{s}}
\newcommand{\bT}{\bar{\T}}
\newcommand{\br}{\bar{r}}
\newcommand{\bc}{\bar{c}}
\newcommand{\bmu}{\bar{\mu}_0}
\newcommand{\bM}{\bar{\mc{M}}}

\subsection{Budget Adaptive MDPs}
Since the definition of the budget-conditioned persistent safety set depends on the remaining budget, it is necessary to make the budget an explicit part of the system’s dynamics. To achieve this, we extend the CMDP by augmenting its state space with a dynamic budget variable. This augmentation allows the agent to reason about safety not only as a function of the environment state but also as a function of the available safety budget over time. {One key contribution is our method for initializing and updating the budget via functions $\ctinit$ and $\ctupdate$. Specially for stochastic MDPs, these updates are non-trivial and crucial for ensuring the theoretical guarantees of our approach.}

\begin{definition}[Budget-Adaptive MDP]\label{def:bmdp}
A \emph{Budget-Adaptive Markov Decision Process} (\ourmdp) constructed from a CMDP $\mc{M}$
is a tuple
\[
\bar{\mc{M}}(\mc{M}, f, g) := \langle \bS, A, \bT, \br, \bc, \gamma, \bmu, \cta \rangle.
\]
where the augmented state space is
\(
\bS := S \times \ctdomain,
\) and the action space $A$ remain unchanged.
The reward and cost functions are unchanged from the original CMDP. That is,
\(
\br((s,\ct),a) := r(s,a), 
\;
\bc((s,\ct),a) := c(s,a).
\)

The \emph{initial budget function}
\(
\ctinit: S \times \ctdomain \to \ctdomain
\)
assigns the initial budget based on the CMDP budget $\cta$ and the starting state $s_0$, such that the augmented initial state distribution is:
\[
\bmu := \{ (s_0, \ct_0) \mid s_0 \sim \mu_0, \;\ct_0 = \ctinit(s_0, \cta) \}.
\]

The \emph{budget update function}
\(
\ctupdate: S \times A \times S \times \ctdomain \to \sR^+
\)
updates the budget after each transition. In the most general form considered here, it depends on the current state $s$, action $a$, next state $s'$, and the current budget $\ct$, such that
\(
\ct' = \ctupdate(s, a, s', \ct).
\)
This results in the following augmented transition function:
\[
\bT\big((s',\ct') \mid (s,\ct), a\big) :=
\T(s' \mid s,a)\cdot \1_{\{\ct' = \ctupdate(s,a,s',\ct)\}},
\]
where $\1_{\{\cdot\}}$ is the indicator function.
\end{definition}

\noindent\textbf{Intuition of $f$ and $g$}  
The role of $f$ and $g$ is to initialize and track a remaining budget (or a similar quantity) while executing the policy, and then use it to prune the action space, as discussed later in this section. The simplest setup is to define  
\(
f(s_0, \cta) := \cta, 
\quad 
g(s,a,s',\ct) := \frac{\ct - c(s,a)}{\gamma},
\)  
where the budget is initialized as the CMDP cost constraint. This formulation keeps track of the exact remaining budget while dividing by $\gamma$ to account for discounting. However, this is not the only possible formulation. We can also track other quantities as the budget, as long as they help enforce the original CMDP cost constraint. 
In Section~\ref{sec:stochastic}, we show how such a quantity can be utilized in stochastic CMDPs.

While we present our formulation with a single cost objective, our approach naturally extends to multiple cost types by augmenting the state space with a budget vector. We provide the detailed formulation for multiple cost constraints in Appendix~\ref{app:multiple_costs}.

\begin{definition}[Budget-Restricted Policy Set]\label{def:policy_set}
Let $\Ap(s,\ct)$ denote the persistent safety action set, as defined in the Eq~\eqref{eq:def_ap}. The \emph{budget-restricted policy set} is:
{
\small
\[
\PiP := \Big\{ \pi : \bS \to \Delta(A) \;\Big|\; 
\pi(a \mid (s,\ct)) > 0\! \implies\! a \in \Ap(s,\ct) 
\Big\}
\]
}
\end{definition}

{The goal of defining $\PiP$ is that, with an appropriate choice of $f$ and $g$, any policy in $\PiP$ will satisfy the CMDP constraint, as discussed in Section~\ref{sec:stochastic}.}

\paragraph{The objective of the \ourmdp\ } is to find a policy in $\PiP$ that maximizes the reward. 
{\small
\begin{equation}
   \pi^* = \max_{\pi \in \PiP} \E_{\pi} 
   \sum_{t=0}^\infty \gamma^t \br((s_t,\ct_t), a_t)
   \label{eq:baobjective}
\end{equation}
}

\paragraph{Feasible state space and $\boldsymbol{\PiP}$}
A policy in $\PiP$ may not be defined for all states, since there may exist states where $\qcs(s,a) > \ct$ for all actions $a \in A$, leaving no admissible action. The following definition identifies the subset of states where the budget-restricted policy set $\PiP$ can be well-defined.

\begin{definition}[Feasible State Subspace]
The \emph{feasible state subspace} is
\(
\SbP := \Big\{ (s,\ct) \;\Big|\; \ct \in \sR^+,\; s \in \Sp(\ct) \Big\}.
\)
{where $\Sp(\ct)$ is the persistent safety set in~\eqref{eq:def_sp}
}.
\end{definition}

\begin{figure}
    \centering
    \includegraphics[width=0.95\linewidth]{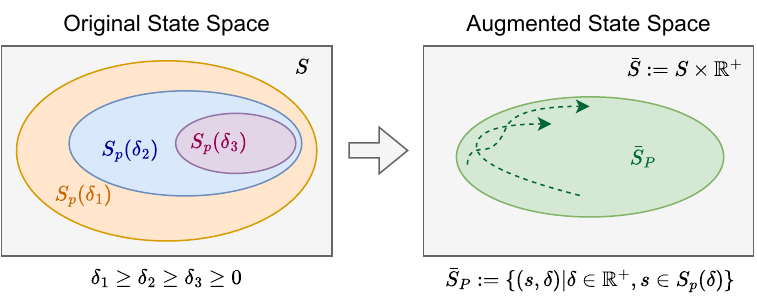}
    \vskip -5pt
    \caption{\small \normalfont\small State augmentation with budget. {A desirable property is set $\SbP$ to be persistent (trajectories starting in $\SbP$ must end in $\SbP$)}}
    \label{fig:augmented-space}
    \vskip -10pt
\end{figure}

\paragraph{Augmented State Space:}
Figure~\ref{fig:augmented-space} illustrates the modified state space and the feasible subspace defined within it. The feasible subspace represents the set of augmented states ($(s, \ct)$) where each state ($s$) remains feasible under its corresponding budget ($\ct$). As shown in Lemma~\ref{lemma:safe-actions-exist}, for every state $s \in \Sp(\ct)$, the safe action set $\Ap(s,\ct)$ is non-empty. This guarantees that the policy set $\PiP$ is well-defined over $\SbP$, since there is always at least one admissible action for each $(s,\ct) \in \SbP$.

Intuitively, some key desirable properties associated with $\SbP$ and $\PiP$ are: (a) the set $\SbP$ should be persistent. That, is an agent starting in $s\in \SbP$ remains in $\SbP$ following a policy $\pi\in\PiP$; (b) For any policy $\pi \in \PiP$, it is guaranteed that the original CMDP constraint $J_C(\pi)\le\cta$ is satisfied. 
In the following, we discuss benefits of defining $\PiP$ and $\SbP$, and then show how to choose appropriate budget update functions for properties (a) and (b) to hold.

\paragraph{Benefits of defining $\PiP$ and $\SbP$:}
Both sets are constructed from the persistent safety sets in Eq.~\ref{eq:def_sp} and Eq.~\ref{eq:def_ap}, which depend only on the conservative cost critic and are independent of both the reward signal and the reward-optimizing policy. This makes enforcing $\pi \in \PiP$ far simpler than satisfying the original CMDP constraint in Eq.~\ref{eq:cmdp_cost}, which depends on the policy being learned and creates a difficult min--max coupling between the cost critic, reward critic, and policy. Our formulation removes this circular dependency, yielding a more stable learning process.

\subsection{Budget Update for CMDPs}\label{sec:stochastic}
The most intuitive budget update function is to track the remaining cost budget exactly, where 
\(
    \ct_0 = \cta,
\)
and define the update rule as:
\(
    \ct_{t+1} = \frac{\ct_t - c(s_t, a_t)}{\gamma}.
\)
In fact, such a budget update rule can result in safety for fully deterministic CMDPs, where, transition function and policy both are deterministic. We formally define and prove it in Appendix~\ref{app:det-proof}.
However, tracking such an exact budget and using it to restrict the policy to $\PiP$ does not guarantee safety in stochastic CMDPs. 
In a stochastic setting, the cost-to-go satisfies
\[
\qcs(s,a) = c(s,a) + \gamma \mathbb{E}_{s'}[\vcs(s')],
\]
rather than the deterministic form 
\[
\qcs(s,a) = c(s,a) + \gamma \vcs(s').
\]
As a result, the next-state value $\vcs(s')$ can take a range of values, and the budget update defined for the deterministic setting no longer guarantees that $(s', \ct')$ lies inside $\SbP$. In such cases, we may need to follow a different policy—possibly the most conservative policy at these states—but doing so does not ensure that the agent will satisfy the CMDP constraint in Eq.~\eqref{eq:cmdp_cost}.
To address these issues, we propose an alternative budget update formulation, called \emph{soft budget-tracking}.

\begin{definition}[Soft Budget-Tracking]
Let $\MS$ denote the Soft Budget-Tracking \ourmdp\ defined by
\[
\MS(\mc{M}) := \bar{\mc{M}}(\mc{M}, f, g),
\]
with
\begin{align}
f(s_0, \cta) = \vcs(s_0)+ \cta - \E_{s \sim \mu_0}[\vcs(s)],
\\
g(s,a,s',\ct) = \vcs(s') + \frac{\ct - \qcs(s,a)}{\gamma} \label{eq:ms-fg}
\end{align}
\end{definition}

The advantage of the above formulation is that it ensures that after any transition, the next state $(s',\ct')\in\SbP$. Moreover, we can show that it satisfies the CMDP constraint. Although $\MS$ may seem unintuitive, its budget update functions reduce to the deterministic budget update in fully deterministic settings, and we formally show this in Appendix~\ref{app:stoch-proof}.

With the soft budget-tracking \ourmdp,  we can reason about the behavior of policies constrained to remain within the persistent safety set $\SbP$, even in stochastic environments. Under the reasonable assumption Eq.~\eqref{eq:assumption-soft} on the initial cost threshold, we can formally show that these policies maintain feasibility: they remain inside the persistent safety set and satisfy the CMDP cumulative cost constraint in expectation. The following theorem summarizes these key properties.

\newcommand{\thmStochMain}[1]{
Let $\MS$ be the soft budget-tracking \ourmdp. Assume the CMDP is feasible, i.e., the cost threshold $\cta$ satisfies
\begin{equation}
\cta \ge \mathbb{E}_{s \sim \mu_0}[V_C^*(s)].
\ifthenelse{\equal{#1}{1}}{\label{eq:assumption-soft}}{\nonumber 
}
\end{equation}

Let $\pi \in \PiP$, and define the expected discounted cumulative cost from an augmented state $(s,\ct) \in \bS$ by
\[
J_C^\pi((s,\ct)) := \E_{\pi}\Big[\sum_{t=0}^\infty \gamma^t c(s_t,a_t) \ \big|\ (s_0,\ct_0) = (s,\ct)\Big].
\]
Then the following properties hold:
\begin{enumerate}
    \item Any policy $\pi \in \PiP$ induces trajectories that start in and remain entirely within the feasible subspace $\SbP$.
    
\item For any $(s,\ct) \in \SbP$ and any policy $\pi \in \PiP$,
\(
J_C^\pi((s,\ct)) \le \ct.
\)
\item Any policy $\pi \in \PiP$ satisfies the CMDP cumulative cost constraint (Eq.~\ref{eq:cmdp_cost}): \(J_C(\pi) \le \cta\).
\end{enumerate}
}

\begin{theorem}[Properties of Policies in $\PiP$ under Soft Budget-Tracking]
\thmStochMain{1}
\label{thm:stoch-main}
\end{theorem}

The proof of Theorem~\ref{thm:stoch-main} is provided in Appendix~\ref{app:stoch-proof}. 
The theorem ensures that any policy in~$\PiP$ remains within~$\SbP$ and that the expected cumulative cost from any augmented state~$(s,\ct)$ does not exceed its budget~$\ct$, thereby satisfying the CMDP constraint. 
Unlike the deterministic case, we cannot guarantee that all CMDP-feasible policies lie within~$\PiP$, so the resulting constraint may be stricter. 
Nevertheless, our empirical results show that it still yields high-quality policies.

\noindent\textbf{CMDP Instantiation} {We next discuss how to formulate and solve \ourmdp\ from a given CMDP with known model. \ourmdp\ is formed by augmenting the state with a discretized budget dimension and defining transitions using \(f\) or \(g\), depending on whether the setting is stochastic or deterministic. We mask states outside \(\SbP\) and restrict actions to the support of \(\PiP\). The resulting finite MDP can then be solved with a standard LP solver. Further details appear in Appendix~\ref{app:cmdp-extension}}.

In real-world settings where dynamics are unknown, offline RL is often used to learn from pre-collected data. Next, we show how our framework integrates with standard offline RL methods to solve \textit{constrained} RL problems.

\section{Safe Offline-RL With \ourmdp\ }

In this section we discuss how our approach can be used with existing \textit{unconstrained} offline RL algorithms. Our approach can be seen as a three step process:
\modified{
\begin{enumerate}
\item We first estimate the cost-to-go functions $\vcs$ and $\qcs$ using offline RL algorithms, ignoring the reward, using the original dataset $\dataset$.
\item Then we use it create a new dataset $\bar{\dataset}$ from $\dataset$ that represent valid transitions of the augmented  MDP $\bar{\mc{M}}$ with the extended state space $\bS$ and dynamic budget.
\item Then we train an Offline RL agent, using the augmented dataset $\bar{\dataset}$ to maximize the reward for the augmented MDP $\bar{\mc{M}}$, while ensuring that the resulting policy respects the persistent safety constraints.
\end{enumerate}
}

The first step is performed on the original MDP without any modifications to the original MDP, while the third step is performed on the \ourmdp.

\paragraph{Learning the Persistent Safety Set } 
For the first step, any offline RL algorithm can be used; the only difference is that we minimize the cost instead of maximizing the reward. In practice, in-sample RL algorithms such as IQL~\cite{kostrikov_offline_2021}, XQL~\cite{garg_extreme_2023}, or SparseQL~\cite{xu_offline_2023} are particularly convenient, as they do not require learning a policy to estimate Q-values. However, if one wishes to use algorithms that rely on out-of-distribution actions, such as SAC+BC~\cite{fujimoto2021minimalistapproachofflinereinforcement}, CQL~\cite{kumar_conservative_2020}, our approach is fully compatible and does not impose any restrictions. 
For convenience and stability, we use IQL in our implementation. The following loss functions can then be used, with $\tauc \leq 0.5$ to minimize the cost (details in~\cite{kostrikov_offline_2021}):

{
\begin{equation}
\mc{L}^V_C(\psi) = \E_{(s,a)\sim \dataset}
[\mathcal{L}^{\tauc}(\qc_{\theta_T}(s,a) - \vc_{\psi}(s))]
\label{eq:loss_vcs}
\end{equation}
\begin{equation}
\label{eq:loss_qcs}
\mc{L}^Q_C(\theta) =
\E_{(s,a,s')\sim \dataset}
[(c(s,a) + \gamma \vc_\psi(s') - \qc_{\theta}(s,a))^2]
\end{equation}
}

\paragraph{Offline RL within the Augmented MDP}
In this step, we train an RL algorithm within the persistent safety set of the augmented MDP $\bar{\mc{M}}(\mc{M},f,g)$, ensuring the policy selects only actions from $\Ap$ by restricting the MDP to $\SbP$.
\noindent To facilitate training, we define an augmented dataset as follows:
{
\small
\[
\bar{\dataset} := \Bigg\{ \big(\underbrace{(s,\ct)}_{\bs}, a, \underbrace{(s',\ct')}_{\bs'}\big) \;\Big|\;
\substack{
(s,a,s') \sim \dataset, \\
\ct \sim \mc{U}_{[\qcs(s,a),\ctmax]}, \;
\ct' = g(s,a,s',\ct)
}
\Bigg\}
\]
}

\modified{
Unlike a fixed dataset, this augmented dataset is not pregenerated; instead, we sample mini-batches from it dynamically during training. Here, the budget $\ct$ is sampled from a uniform distribution $\mc{U}_{[\qcs(s,a), \ctmax]}$ to generate augmented transitions within the persistent safety set. Sampling $\ct \ge \qcs(s,a)$ ensures that the dataset extends only to the persistent safety set $\SbP$ and not to the full augmented state space $\bS$.
Crucially, this process acts as data augmentation rather than filtering. A valid budget range $[\qcs(s,a), \ctmax]$ always exists for any state-action transition $(s,a,s')$ in the offline dataset. Thus, we can always sample a valid budget $\ct$ and construct an augmented transition $((s,\ct), a, (s',\ct'))$. This ensures that the training process utilizes the entire support of the original offline dataset without discarding any samples.
The rewards $r(s,a)$ and $c(s,a)$ remain unchanged, as discussed in the Definition~\ref{def:bmdp}.
}

\begin{algorithm}[t]
\small
   \caption{\small \our\ : IQL Instantiation Algorithm }
   \label{alg:siql}
\begin{algorithmic}
\small 
   \State {\bfseries Initialize Parameters:} 
   ${\theta}$,   
   ${\theta_T}$,   
   ${\psi}$,   
   $\hat{\theta}$,   
   $\hat{\theta}_T$,   
   $\hat{\psi}$,   
   ${\phi}$
   \State {\bfseries Inputs:} Offline data $\dataset$, and other hyper parameters $\tau_c,\tau_r,\alpha,\beta$
   \For{each gradient step} 
   \State {$(s,a,s')\sim \dataset$} \Comment{sample mini batch of transitions}
   \State   ${\theta} \leftarrow {\theta} - \lambda_Q^C\nabla_{{\theta}} \loss_{Q}^C({\theta})$ \Comment{update cost-critic Eq.\eqref{eq:loss_qcs}} 
   \State   ${\psi} \leftarrow {\psi} - \lambda_V^C\nabla_{{\psi}} \loss_{V}^C({\psi})$ 
   \Comment{update cost-value Eq.\eqref{eq:loss_vcs}}
   \State   $\theta_T \leftarrow (1-\alpha){\theta}_T + \alpha{\theta}$ \Comment{soft update cost target-net.}
  \\
   \State {$\ct \sim \mc{U}_{[\qc_{\theta}(s,a),\ctmax]}$} \Comment{sample budget from uniform dist.}   \State {$\ct' = g(s,a,s',\ct)$} \Comment{compute next budget}
   \State {$\bs = (s,\ct)$ and $\bs' = (s',\ct')$} \Comment{compute augmented states}
   \State   $\hat\theta \leftarrow \hat\theta - \lambda_Q^R \nabla_{\hat\theta} \loss_{Q}^R({\hat\theta})$ 
   \Comment{update reward critic Eq.\eqref{eq:loss_qr}}
   \State   $\hat\psi \leftarrow \hat\psi - \lambda_V^R\nabla_{\hat\psi} \loss_{V}^R({\hat\psi})$
 \Comment{update value function Eq.\eqref{eq:loss_vr}}
   \State   $\hat{\theta}_T \leftarrow (1-\alpha)\hat{\theta}_T + \alpha\hat{\theta}$
   \Comment{soft update reward target-net.} 
   \State   $\phi \leftarrow \phi - \lambda_{\pi} \nabla_{\phi} \mc{L}^{\pi_R}({\phi})$
    \Comment{policy extraction (AWR) Eq.\eqref{eq:loss_pi}}
   \EndFor
\end{algorithmic}
\end{algorithm}

\noindent \textbf{Implicit Enforcement of Persistent Safety} Offline RL naturally stays close to the behavior policy, thus restricting the dataset to the persistent safety set already biases the learned policy to remain within it. In practice, we found no extra penalties are required. 

\noindent \textbf{IQL-Instantiation}
We instantiate our approach using Implicit Q-Learning (IQL) by extending the offline dataset to include only transitions within the persistent safety set $\SbP$. This modification naturally integrates our safety constraint into the learning process. The value and Q-functions are trained using the standard expectile regression (with $\taur \geq0.5$) and Bellman objectives defined over the augmented dataset $\bar{\dataset}$:
{
\begin{equation}
\mc{L}^{V}_R(\hat\psi) =
\E_{(\bs,a,\bs')\sim \bar\dataset}
\big[\mathcal{L}^{\taur}\big(\qr_{\hat\theta_T}(\bs,a) - \vr_{\hat\psi}(\bs)\big)\big]
\label{eq:loss_vr}
\end{equation}
}
{
\small
\begin{equation}
\mc{L}^{Q}_R(\hat\theta) =\!\!
\E_{
(\bs,a,\bs')\sim \bar\dataset}
\big[(\br(\bs,a) + \gamma \vr_{\hat\psi}(\bs') - \qr_{\hat\theta}(\bs,a))^2\big]
\label{eq:loss_qr}
\end{equation}
}
Here, $\theta_T$ denotes the target network parameters, which are updated using a soft update. For policy extraction, we follow Advantage-Weighted Regression (AWR) as in IQL, where the policy is trained to maximize the advantage-weighted likelihood of actions under the Q-function:

{
\begin{equation}
\mc{L}^{\pi}(\phi) =
-\!\!\!\!\E_{(\bs,a)\sim \bar\dataset}
\left[\exp{\left(\frac{A(\bs,a)}{\beta}\right)}
\log \pi_\phi(a \!\mid \!\bs)\right]
\label{eq:loss_pi}
\end{equation}
}

where $A(\bs,a) = \qr_{\hat\theta}(\bs,a) - \vr_{\hat\psi}(\bs)$ denotes the advantage, and $\beta$ controls the temperature of the weighting distribution. The exact learning process is outlined in the Algorithm~\ref{alg:siql}. In our method, called BCRL, cost critics ($Q^C, V^C$) can be learned by ignoring rewards altogether, and policy extraction only depends on reward critics. This is significantly more stable that standard Lagrangian relaxation methods.

\section{Experiments}

\begin{figure}[t]
    \centering
    \includegraphics[width=0.99\linewidth]{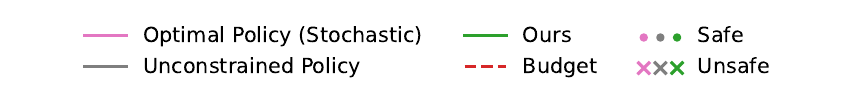} \\
    \includegraphics[width=0.46\linewidth]{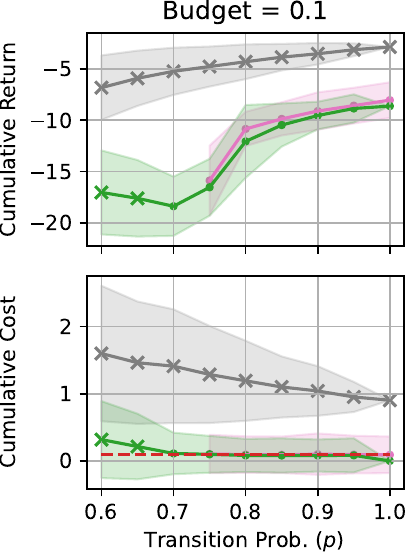}
    \hspace{0.05\linewidth}
    \includegraphics[width=0.46\linewidth]{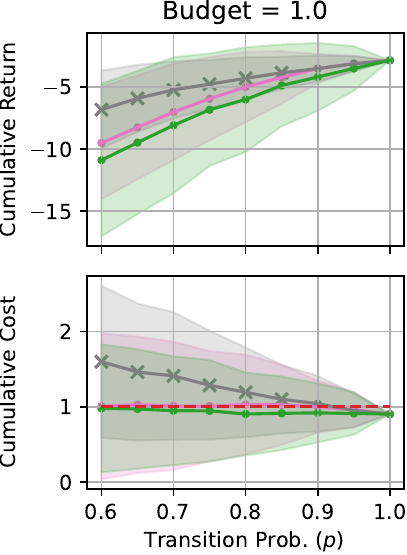}
\caption{
\textbf{\small Comparison with Optimal Solution:} Grid-world results with X-axis as the intended-movement probability $p$ (higher values indicate less noise). Top plots show total return; bottom plots show cost for two budget levels.
}
    \label{fig:synthetic}
    \vskip -10pt
\end{figure}

\begin{table*}[t]
    \centering
\resizebox{\textwidth}{!}
{
\small 
\begin{tabular}{lllllllllll}
 & \multicolumn{2}{c}{CDT} & \multicolumn{2}{c}{CAPS} & \multicolumn{2}{c}{CCAC} & \multicolumn{2}{c}{LSPC} & \multicolumn{2}{c}{\our\ (Ours)} \\
 & reward $\uparrow$ & cost $\downarrow$ & reward $\uparrow$ & cost $\downarrow$ & reward $\uparrow$ & cost $\downarrow$ & reward $\uparrow$ & cost $\downarrow$ & reward $\uparrow$ & cost $\downarrow$ \\
\hline
PointButton1 & \gray{0.53$_{\pm 0.01}$} & \gray{1.68$_{\pm 0.13}$} & \textbf{0.04$_{\pm 0.15}$} & \textbf{0.58$_{\pm 1.42}$} & \gray{0.68$_{\pm 0.14}$} & \gray{4.31$_{\pm 3.20}$} & \textbf{\blue{0.09$_{\pm 0.20}$}} & \textbf{\blue{0.69$_{\pm 1.03}$}} & \textbf{0.07$_{\pm 0.18}$} & \textbf{0.79$_{\pm 2.10}$} \\
PointButton2 & \gray{0.46$_{\pm 0.01}$} & \gray{1.57$_{\pm 0.10}$} & \textbf{\blue{0.12$_{\pm 0.20}$}} & \textbf{\blue{0.96$_{\pm 2.35}$}} & \gray{0.63$_{\pm 0.11}$} & \gray{4.82$_{\pm 3.61}$} & \gray{0.19$_{\pm 0.21}$} & \gray{1.39$_{\pm 2.53}$} & \textbf{\blue{0.12$_{\pm 0.22}$}} & \textbf{\blue{0.81$_{\pm 1.36}$}} \\
PointCircle1 & \textbf{\blue{0.59$_{\pm 0.00}$}} & \textbf{\blue{0.69$_{\pm 0.04}$}} & \textbf{0.50$_{\pm 0.12}$} & \textbf{0.68$_{\pm 1.51}$} & \gray{0.58$_{\pm 0.13}$} & \gray{2.22$_{\pm 1.05}$} & \gray{0.56$_{\pm 0.35}$} & \gray{3.49$_{\pm 3.55}$} & \textbf{0.55$_{\pm 0.17}$} & \textbf{0.86$_{\pm 0.21}$} \\
PointCircle2 & \gray{0.64$_{\pm 0.01}$} & \gray{1.05$_{\pm 0.08}$} & \textbf{0.51$_{\pm 0.21}$} & \textbf{0.76$_{\pm 0.47}$} & \gray{0.18$_{\pm 0.50}$} & \gray{1.42$_{\pm 1.09}$} & \gray{0.67$_{\pm 0.24}$} & \gray{4.18$_{\pm 4.56}$} & \textbf{\blue{0.58$_{\pm 0.10}$}} & \textbf{\blue{0.74$_{\pm 0.34}$}} \\
PointGoal1 & \gray{0.69$_{\pm 0.02}$} & \gray{1.12$_{\pm 0.07}$} & \textbf{0.46$_{\pm 0.28}$} & \textbf{0.63$_{\pm 0.81}$} & \gray{0.55$_{\pm 0.26}$} & \gray{1.88$_{\pm 1.72}$} & \textbf{0.26$_{\pm 0.27}$} & \textbf{0.23$_{\pm 0.45}$} & \textbf{\blue{0.56$_{\pm 0.25}$}} & \textbf{\blue{0.70$_{\pm 0.65}$}} \\
PointGoal2 & \gray{0.59$_{\pm 0.03}$} & \gray{1.34$_{\pm 0.05}$} & \textbf{0.34$_{\pm 0.23}$} & \textbf{0.80$_{\pm 0.74}$} & \gray{0.41$_{\pm 0.36}$} & \gray{3.60$_{\pm 3.37}$} & \textbf{0.23$_{\pm 0.24}$} & \textbf{0.50$_{\pm 0.79}$} & \textbf{\blue{0.38$_{\pm 0.21}$}} & \textbf{\blue{0.77$_{\pm 0.68}$}} \\
CarCircle1 & \gray{0.60$_{\pm 0.01}$} & \gray{1.73$_{\pm 0.04}$} & \gray{0.55$_{\pm 0.16}$} & \gray{1.54$_{\pm 1.63}$} & \gray{0.49$_{\pm 0.14}$} & \gray{5.13$_{\pm 3.53}$} & \gray{0.46$_{\pm 0.25}$} & \gray{1.86$_{\pm 3.14}$} & \textbf{\blue{0.50$_{\pm 0.12}$}} & \textbf{\blue{0.48$_{\pm 0.43}$}} \\
CarPush2 & \gray{0.19$_{\pm 0.01}$} & \gray{1.30$_{\pm 0.16}$} & \textbf{0.05$_{\pm 0.14}$} & \textbf{0.71$_{\pm 1.66}$} & \gray{0.12$_{\pm 0.31}$} & \gray{6.06$_{\pm 9.53}$} & \textbf{\blue{0.07$_{\pm 0.15}$}} & \textbf{\blue{0.80$_{\pm 1.41}$}} & \textbf{0.05$_{\pm 0.17}$} & \textbf{0.52$_{\pm 1.07}$} \\
SwimmerVelocity & \textbf{\blue{0.66$_{\pm 0.01}$}} & \textbf{\blue{0.96$_{\pm 0.08}$}} & \gray{0.44$_{\pm 0.22}$} & \gray{1.66$_{\pm 2.37}$} & \textbf{-0.08$_{\pm 0.02}$} & \textbf{0.00$_{\pm 0.01}$} & \gray{0.46$_{\pm 0.21}$} & \gray{1.28$_{\pm 2.60}$} & \textbf{0.49$_{\pm 0.18}$} & \textbf{0.99$_{\pm 0.37}$} \\
HopperVelocity & \textbf{0.63$_{\pm 0.06}$} & \textbf{0.61$_{\pm 0.08}$} & \textbf{0.44$_{\pm 0.26}$} & \textbf{0.75$_{\pm 0.83}$} & \textbf{0.46$_{\pm 0.38}$} & \textbf{0.50$_{\pm 0.55}$} & \gray{0.21$_{\pm 0.04}$} & \gray{1.25$_{\pm 1.20}$} & \textbf{\blue{0.66$_{\pm 0.23}$}} & \textbf{\blue{0.80$_{\pm 0.35}$}} \\
HalfCheetahVelocity & \textbf{\blue{1.00$_{\pm 0.01}$}} & \textbf{\blue{0.01$_{\pm 0.01}$}} & \textbf{0.94$_{\pm 0.04}$} & \textbf{0.77$_{\pm 0.13}$} & \textbf{0.94$_{\pm 0.04}$} & \textbf{0.94$_{\pm 0.07}$} & \textbf{0.97$_{\pm 0.02}$} & \textbf{1.00$_{\pm 1.04}$} & \textbf{0.93$_{\pm 0.07}$} & \textbf{0.62$_{\pm 0.27}$} \\
\hline
SafetyGym Average (21 tasks) & \gray{0.54$_{\pm 0.21}$} & \gray{1.06$_{\pm 0.59}$} & \textbf{0.36$_{\pm 0.34}$} & \textbf{0.76$_{\pm 1.43}$} & \gray{0.41$_{\pm 0.46}$} & \gray{3.18$_{\pm 5.24}$} & \gray{0.34$_{\pm 0.19}$} & \gray{1.10$_{\pm 1.75}$} & \textbf{\blue{0.38$_{\pm 0.17}$}} & \textbf{\blue{0.57$_{\pm 0.68}$}} \\
\hline
\hline
BallRun & \gray{0.39$_{\pm 0.09}$} & \gray{1.16$_{\pm 0.19}$} & \textbf{0.18$_{\pm 0.09}$} & \textbf{0.94$_{\pm 0.73}$} & \textbf{\blue{0.40$_{\pm 0.12}$}} & \textbf{\blue{0.84$_{\pm 0.31}$}} & \textbf{0.15$_{\pm 0.01}$} & \textbf{0.00$_{\pm 0.00}$} & \textbf{0.20$_{\pm 0.02}$} & \textbf{0.06$_{\pm 0.10}$} \\
CarRun & \textbf{\blue{0.99$_{\pm 0.01}$}} & \textbf{\blue{0.65$_{\pm 0.31}$}} & \textbf{0.97$_{\pm 0.01}$} & \textbf{0.23$_{\pm 0.25}$} & \textbf{0.97$_{\pm 0.04}$} & \textbf{0.68$_{\pm 0.22}$} & \textbf{0.98$_{\pm 0.01}$} & \textbf{0.84$_{\pm 1.04}$} & \textbf{0.98$_{\pm 0.01}$} & \textbf{0.98$_{\pm 0.72}$} \\
DroneRun & \textbf{\blue{0.63$_{\pm 0.04}$}} & \textbf{\blue{0.79$_{\pm 0.68}$}} & \gray{0.45$_{\pm 0.10}$} & \gray{2.91$_{\pm 4.46}$} & \gray{0.46$_{\pm 0.18}$} & \gray{4.36$_{\pm 4.32}$} & \textbf{0.59$_{\pm 0.06}$} & \textbf{0.93$_{\pm 1.51}$} & \textbf{0.58$_{\pm 0.03}$} & \textbf{0.54$_{\pm 0.91}$} \\
AntRun & \textbf{\blue{0.72$_{\pm 0.04}$}} & \textbf{\blue{0.91$_{\pm 0.42}$}} & \textbf{0.61$_{\pm 0.13}$} & \textbf{0.78$_{\pm 0.59}$} & \textbf{0.12$_{\pm 0.09}$} & \textbf{0.04$_{\pm 0.08}$} & \gray{0.66$_{\pm 0.13}$} & \gray{1.43$_{\pm 1.25}$} & \textbf{0.68$_{\pm 0.11}$} & \textbf{0.95$_{\pm 0.29}$} \\
BallCircle & \gray{0.77$_{\pm 0.06}$} & \gray{1.07$_{\pm 0.27}$} & \textbf{0.69$_{\pm 0.10}$} & \textbf{0.59$_{\pm 0.23}$} & \textbf{\blue{0.78$_{\pm 0.05}$}} & \textbf{\blue{0.46$_{\pm 0.31}$}} & \gray{0.62$_{\pm 0.28}$} & \gray{1.17$_{\pm 1.95}$} & \textbf{0.69$_{\pm 0.11}$} & \textbf{0.97$_{\pm 0.27}$} \\
CarCircle & \textbf{\blue{0.75$_{\pm 0.06}$}} & \textbf{\blue{0.95$_{\pm 0.61}$}} & \textbf{0.70$_{\pm 0.10}$} & \textbf{0.66$_{\pm 0.27}$} & \textbf{0.72$_{\pm 0.04}$} & \textbf{0.77$_{\pm 0.44}$} & \gray{0.78$_{\pm 0.12}$} & \gray{1.88$_{\pm 3.23}$} & \textbf{0.53$_{\pm 0.17}$} & \textbf{0.51$_{\pm 0.51}$} \\
DroneCircle & \textbf{\blue{0.63$_{\pm 0.07}$}} & \textbf{\blue{0.98$_{\pm 0.27}$}} & \textbf{0.55$_{\pm 0.06}$} & \textbf{0.65$_{\pm 0.29}$} & \textbf{0.29$_{\pm 0.27}$} & \textbf{0.63$_{\pm 0.70}$} & \gray{0.60$_{\pm 0.02}$} & \gray{1.31$_{\pm 0.86}$} & \textbf{0.42$_{\pm 0.12}$} & \textbf{0.52$_{\pm 0.32}$} \\
AntCircle & \gray{0.54$_{\pm 0.20}$} & \gray{1.78$_{\pm 4.33}$} & \textbf{0.37$_{\pm 0.18}$} & \textbf{0.15$_{\pm 0.25}$} & \textbf{\blue{0.61$_{\pm 0.14}$}} & \textbf{\blue{0.75$_{\pm 0.90}$}} & \textbf{0.44$_{\pm 0.14}$} & \textbf{0.53$_{\pm 0.93}$} & \textbf{0.56$_{\pm 0.17}$} & \textbf{0.79$_{\pm 0.60}$} \\
\hline
BulletGym Average (8 tasks) & \gray{0.68$_{\pm 0.19}$} & \gray{1.04$_{\pm 1.65}$} & \textbf{0.57$_{\pm 0.25}$} & \textbf{0.86$_{\pm 1.82}$} & \gray{0.54$_{\pm 0.30}$} & \gray{1.07$_{\pm 2.04}$} & \gray{0.60$_{\pm 0.10}$} & \gray{1.01$_{\pm 1.35}$} & \textbf{\blue{0.58$_{\pm 0.09}$}} & \textbf{\blue{0.67$_{\pm 0.46}$}} \\
\hline
\hline
easysparse & \textbf{0.17$_{\pm 0.14}$} & \textbf{0.23$_{\pm 0.32}$} & \textbf{0.12$_{\pm 0.20}$} & \textbf{0.37$_{\pm 0.43}$} & \textbf{-0.06$_{\pm 0.00}$} & \textbf{0.10$_{\pm 0.05}$} & \gray{0.79$_{\pm 0.14}$} & \gray{1.34$_{\pm 1.43}$} & \textbf{\blue{0.70$_{\pm 0.19}$}} & \textbf{\blue{0.94$_{\pm 0.17}$}} \\
easymean & \textbf{0.45$_{\pm 0.11}$} & \textbf{0.54$_{\pm 0.55}$} & \textbf{0.02$_{\pm 0.07}$} & \textbf{0.21$_{\pm 0.23}$} & \textbf{-0.06$_{\pm 0.01}$} & \textbf{0.07$_{\pm 0.08}$} & \gray{0.82$_{\pm 0.06}$} & \gray{1.33$_{\pm 0.86}$} & \textbf{\blue{0.73$_{\pm 0.17}$}} & \textbf{\blue{0.94$_{\pm 0.13}$}} \\
easydense & \textbf{0.32$_{\pm 0.18}$} & \textbf{0.62$_{\pm 0.43}$} & \textbf{0.11$_{\pm 0.15}$} & \textbf{0.16$_{\pm 0.17}$} & \textbf{-0.06$_{\pm 0.00}$} & \textbf{0.07$_{\pm 0.04}$} & \gray{0.84$_{\pm 0.10}$} & \gray{1.88$_{\pm 1.20}$} & \textbf{\blue{0.70$_{\pm 0.17}$}} & \textbf{\blue{0.90$_{\pm 0.16}$}} \\
mediumsparse & \gray{0.87$_{\pm 0.11}$} & \gray{1.10$_{\pm 0.26}$} & \textbf{0.59$_{\pm 0.36}$} & \textbf{0.74$_{\pm 0.97}$} & \textbf{-0.08$_{\pm 0.00}$} & \textbf{0.07$_{\pm 0.04}$} & \gray{0.97$_{\pm 0.02}$} & \gray{1.24$_{\pm 0.68}$} & \textbf{\blue{0.94$_{\pm 0.07}$}} & \textbf{\blue{0.82$_{\pm 0.32}$}} \\
mediummean & \textbf{0.45$_{\pm 0.39}$} & \textbf{0.75$_{\pm 0.83}$} & \textbf{0.66$_{\pm 0.35}$} & \textbf{0.90$_{\pm 0.89}$} & \textbf{-0.07$_{\pm 0.01}$} & \textbf{0.02$_{\pm 0.03}$} & \textbf{\blue{0.93$_{\pm 0.12}$}} & \textbf{\blue{0.78$_{\pm 0.51}$}} & \textbf{0.92$_{\pm 0.11}$} & \textbf{0.77$_{\pm 0.16}$} \\
mediumdense & \gray{0.88$_{\pm 0.12}$} & \gray{2.41$_{\pm 0.71}$} & \textbf{0.75$_{\pm 0.29}$} & \textbf{0.65$_{\pm 0.57}$} & \textbf{-0.08$_{\pm 0.00}$} & \textbf{0.06$_{\pm 0.03}$} & \gray{0.82$_{\pm 0.30}$} & \gray{1.07$_{\pm 0.71}$} & \textbf{\blue{0.78$_{\pm 0.28}$}} & \textbf{\blue{0.68$_{\pm 0.31}$}} \\
hardsparse & \textbf{0.25$_{\pm 0.08}$} & \textbf{0.41$_{\pm 0.33}$} & \textbf{0.45$_{\pm 0.15}$} & \textbf{0.75$_{\pm 0.56}$} & \textbf{-0.05$_{\pm 0.00}$} & \textbf{0.07$_{\pm 0.04}$} & \textbf{\blue{0.54$_{\pm 0.09}$}} & \textbf{\blue{0.91$_{\pm 0.51}$}} & \textbf{0.49$_{\pm 0.13}$} & \textbf{0.76$_{\pm 0.44}$} \\
hardmean & \textbf{0.33$_{\pm 0.21}$} & \textbf{0.97$_{\pm 0.31}$} & \textbf{0.29$_{\pm 0.13}$} & \textbf{0.28$_{\pm 0.31}$} & \textbf{-0.05$_{\pm 0.00}$} & \textbf{0.06$_{\pm 0.03}$} & \gray{0.48$_{\pm 0.14}$} & \gray{1.16$_{\pm 1.19}$} & \textbf{\blue{0.46$_{\pm 0.15}$}} & \textbf{\blue{0.84$_{\pm 0.54}$}} \\
harddense & \textbf{0.08$_{\pm 0.15}$} & \textbf{0.21$_{\pm 0.42}$} & \textbf{0.36$_{\pm 0.18}$} & \textbf{0.66$_{\pm 0.92}$} & \textbf{-0.03$_{\pm 0.01}$} & \textbf{0.11$_{\pm 0.08}$} & \gray{0.47$_{\pm 0.19}$} & \gray{1.43$_{\pm 0.93}$} & \textbf{\blue{0.44$_{\pm 0.17}$}} & \textbf{\blue{0.63$_{\pm 0.24}$}} \\
\hline
MetaDrive Average (9 tasks) & \textbf{0.42$_{\pm 0.31}$} & \textbf{0.80$_{\pm 0.61}$} & \textbf{0.38$_{\pm 0.34}$} & \textbf{0.54$_{\pm 0.69}$} & \textbf{-0.06$_{\pm 0.02}$} & \textbf{0.07$_{\pm 0.06}$} & \gray{0.74$_{\pm 0.13}$} & \gray{1.24$_{\pm 0.89}$} & \textbf{\blue{0.69$_{\pm 0.16}$}} & \textbf{\blue{0.81$_{\pm 0.27}$}} \\
\hline
\bottomrule
\end{tabular}
}
\caption{
\normalfont
\small 
\textbf{Results for normalized cost and normalized reward:} \(\uparrow\) indicates that a higher value is better, while \(\downarrow\) indicates that a lower value is preferable. \textbf{Bold} text denotes safe policies, \gray{gray} indicates unsafe policies, and \blue{\textbf{blue}} highlights the highest reward among the safe policies for each task.  
 Each result is obtained by running three random seeds across three distinct cost thresholds and evaluating over 20 episodes. {Some tasks are omitted to save space. Complete results are provided in Table~\ref{app:main_table} of Appendix~\ref{app:exp}.}}
\label{tab:main-table}
\end{table*}

\begin{figure*}[bt]
    \centering
    \begin{minipage}{0.55\textwidth}
    \centering

\resizebox{\textwidth}{!}
{
\begin{tabular}{lccccc}
\toprule
 & \makecell{ADE \\ ($m$) $\downarrow$ }  & \makecell{Close\\quaters\\ Rate $\downarrow$} & \makecell{Avg.\\ Acceleration \\ ($ms^{-2}$)} & \makecell{Avg. \\ Speed \\ ($ms^{-1}$)} & \makecell{Success \\ Rate $\uparrow$ }\\
\midrule
EXPERT & 0.0 & 0.3 & 0.0001 & 4.44  & 1.0  \\
CAPS & 87.3 ± 19.82 & 0.23 ± 0.05 & 0.0034 ± 0.0006 & 5.67 ± 0.07 & 0.73 ± 0.16 \\
CCAC & 371.07 ± 26.74 & \textbf{0.11 ± 0.06} & 0.7549 ± 0.1339 & 21.41 ± 0.0 & 0.0 ± 0.0 \\
LSPC & 251.61 ± 22.68 & 0.17 ± 0.02 & 0.0087 ± 0.0157 & 8.47 ± 0.37 & 0.23 ± 0.1 \\
\hline
\textbf{\our} & \textbf{52.08 ± 10.64} & 0.26 ± 0.03 & -0.0006 ± 0.0003 & 3.65 ± 0.14 & \textbf{0.88 ± 0.03} \\
\hline
\bottomrule
\end{tabular}
}
\newline \newline
    \captionof{table}{\normalfont\small Performance metrics in the Maritime Navigation Task (section~\ref{exp:real})}
    \label{tab:ship-nav}

    \end{minipage}
    \hfill
    \begin{minipage}{0.40\textwidth}
        \centering
    \includegraphics[width=0.9\linewidth]{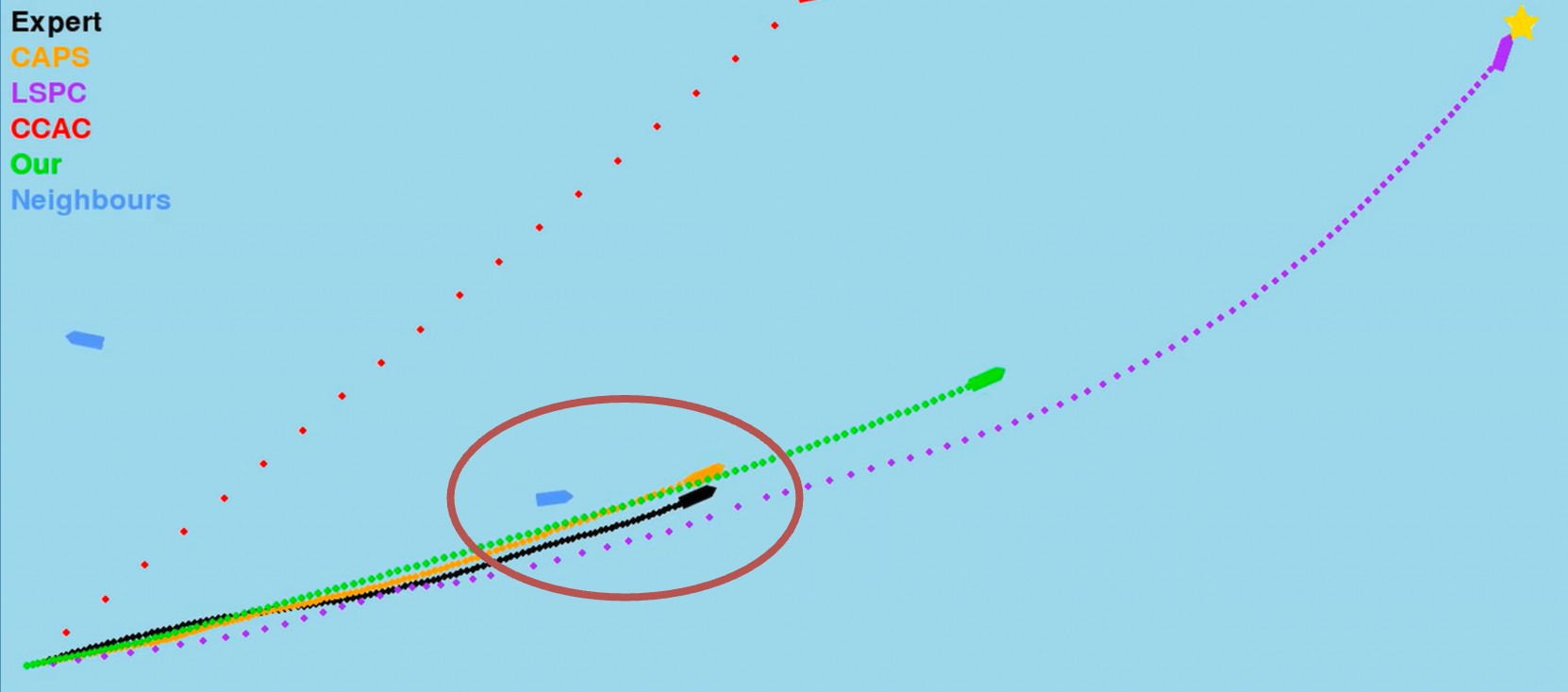}
    \caption{
\normalfont\small Expert, learned trajectories in marine navigation
    }
    \label{fig:ship-nav}
    \end{minipage}
    \vskip -10pt
\end{figure*}

We evaluate \our\ on DSRL benchmarks~\cite{liu_datasets_2024} and a real-world maritime navigation task. We aim to answer following questions:
{\textbf{(Q1)} Does \our\ create a notable gap from the optimal CMDP ({model known}) solution, particularly under stochastic conditions? (Section \ref{sec:comp-optimal})}
\textbf{(Q2)} Can \our\ outperform existing state-of-the-art offline ({model-free}) safe RL baselines on standard benchmarks? (Section~\ref{exp:dsrl})
\textbf{(Q3)} How sensitive is \our\ to different hyperparameter settings, and how do these choices affect its performance and safety behavior?
\textbf{(Q4)} How does \our\ behave on real-world-data, both qualitatively (in terms of safe and interpretable trajectories) and quantitatively (in terms of performance and safety metrics)? (Section~\ref{exp:real}).

\noindent\textbf{Baselines and benchmarks:} For our main results, we compare against state-of-the-art baselines that also generalize to different cost budgets similar to our method. These methods include \textbf{CDT}~\cite{liu_constrained_2023}, \textbf{CAPS}~\cite{chemingui_constraint-adaptive_2025}, \textbf{CCAC}~\cite{guo_constraint-conditioned_2025}, and \textbf{LSPC}~\cite{koirala_latent_2025}. An additional comparison with methods that do not adapt to different cost budgets is provided in the supplementary. We use the DSRL~\cite{liu_datasets_2024} dataset and environments. We evaluate our approach on all $38$ tasks, which include three types of  environments: \emph{SafetyGym}, \emph{BulletGym}, and \emph{MetaDrive}. The evaluation metrics include \emph{Normalized-Reward} and \emph{Normalized-Cost}, where the cost is normalized with respect to the cost threshold; that is, \emph{Normalized-Cost}~$> 1$ indicates an unsafe policy. More details on these metrics are in the Appendix. 

\subsection{Synthetic CMDPs}\label{sec:comp-optimal}
Under deterministic dynamics, our method matches the optimal solution; in stochastic settings, it can be slightly conservative. We test this on a discrete grid world where the optimal CMDP policy is computed via Linear Programming (LP) solvers. The agent starts at a fixed point and must reach a goal (appendix~\ref{app:grid-world} provides environment details.). Figure~\ref{fig:synthetic} compares our method with the LP solution and an unconstrained baseline. The X-axis is the probability of executing the intended action; lower values introduce more randomness. When this probability is too low, no feasible policy satisfies the budget, so the LP solver yields no solution (shown as missing plot segments). For all the cases, both methods satisfy the cost constraint, and our approach incurs only a small reward gap. As the environment becomes less noisy, our solution converges to the LP optimum. This shows that even in a highly stochastic setting, BCRL is fairly accurate.

\subsection{Offline Safe RL}\label{exp:dsrl}

\paragraph{Main Results.}
Table~\ref{tab:main-table} reports the main results under the standard DSRL evaluation settings. All algorithms are trained following the official DSRL benchmark protocols, with additional implementation details provided in the Appendix. We instantiate \our\ using IQL. In stochastic domains (e.g., SafetyGym, Bullet-SafetyGym), we use the stochastic variant, while for deterministic environments like MetaDrive, we report the deterministic version. Both variant results are detailed in the Appendix for completeness.
Overall, the results show \our\ consistently outperforms baselines while satisfying safety across all domains. \our\ produces safe policies in all 38 out of 38 tasks (indicated in \textbf{bold}), outperforming all baselines in 16 out of 38 tasks, and achieving strongly consistent results as highlighted in \blue{blue}. \our\ attains \textbf{highest average performance across all three benchmarks}.

Table~\ref{tab:main-table} focuses on baselines capable of adapting to different cost budgets. Additional comparisons with non–budget-conditioned baselines such as BC-Safe \cite{liu_datasets_2024}, BCQ-Lag, COptiDICE \cite{lee_coptidice_2022}, and CPQ \cite{xu_constraints_2022} are in the Appendix~\ref{app:exp}.
\textbf{Runtime:} Our algorithm was tested on an NVIDIA RTX 3090 GPU. As shown in Table~\ref{tab:run_time} (in appendix), it completes training and evaluation within a few minutes---significantly faster than baseline methods that require 2–3 hours under the same conditions. Although runtime comparisons are approximate, the results indicate notable efficiency gains of our method.

\noindent\textbf{Ablations:}
Our IQL variant is governed by three hyperparameters: the cost critic expectile ($\tauc$), the reward critic expectile ($\taur$), and the AWR policy temperature ($\beta$). Among these, the expectile values influence performance more, as also noted in~\cite{kostrikov_offline_2021}. Expectile ablation results are in Table~\ref{tab:expectile}, appendix~\ref{app:hyperparam}. Additionally, the appendix discusses how the quality of the learned cost critic affects overall performance and how inclusive the learned persistent safety set is.

\subsection{Evaluation on Real-World Maritime Data}\label{exp:real}

We test \our\ on a real-world {maritime navigation} task, learning a policy that imitates expert captains navigating congested routes while minimizing collision risk. 
We use the maritime traffic simulator, and the imitation-learning dataset from~\cite{pham_shipnavisim_2025}, augmented with reward and cost signals. The agent receives a sparse reward of $100$ for reaching the goal, and a cost of $1$ for entering a \textbf{close-quarter scenario} (getting close to another vessel within 555,m~\cite{pham_shipnavisim_2025}).

The dataset contains $\sim$2 years of AIS trajectories from vessels operating in the Singapore Strait. We simulate a \textbf{multi-agent environment} where surrounding vessels replay their historical (log-play) paths, while the RL-controlled ego agent starts from a historical position but must follow its learned policy. The agent observes the past five steps of its own and nearby vessels’ states and uses a delta-action space~\cite{gulino_waymax_2023}. We evaluate performance using the \textbf{average displacement error (ADE)}, the \textbf{close-quarter rate}, and the \textbf{success rate}. To model dense traffic, we select the top 20\% most congested scenarios. We additionally report \textbf{average acceleration and speed} to assess how closely the policy matches real-world maneuvering. Further environment and dataset details appear in Appendix~\ref{app:maritime}.

As shown in Table~\ref{tab:ship-nav}, \our\ reduces close-quarter events from 30\% to 26\%—a meaningful improvement given the high-risk nature of close encounters involving large cargo and tanker vessels. \our\ also achieves the \textbf{lowest ADE} ($\sim$52\,m) and the highest \textbf{success rate} (88\%). Some baselines reduce close-quarter rates further but at the cost of unrealistic deviations from expert trajectories or reduced task success. In contrast, \our\ maintains expert-like \textbf{acceleration and speed} profiles. Results are averaged over 3 seeds with 100 evaluation episodes each.

Figure~\ref{fig:ship-nav} illustrates qualitative behavior: the expert’s trajectory (dotted \textbf{black}), neighboring vessels (\blue{blue}), and learned-policy rollouts—\green{\our}, \textcolor{purple}{LSPC}, \textcolor{red}{CCAC}, and \textcolor{orange}{CAPS}. The red ellipse shows a close-quarter scenario where the expert gets too close to a \blue{neighbor}. \our\ avoids near-misses through smooth path and speed adjustments, whereas LSPC and CCAC exhibit unrealistic deviations. CAPS performs more reasonably but shows higher acceleration and lower success than \our. Additional results and videos are included in the supplementary material.

\section{Conclusion} 
\modified{
We introduced Budget-Conditioned Reachability (BCR), a novel framework that fundamentally decouples reward maximization from cumulative safety constraints in Constrained Markov Decision Processes (CMDPs). By defining a dynamic, safety-conditioned reachability set---tracked through a step-wise remaining budget---our approach prunes unsafe actions at each timestep and ensures reliable constraint satisfaction without resorting to unstable min-max adversarial optimization or computationally heavy generative models.
}

\modified{
A significant advantage of our framework is its plug-and-play compatibility with any standard offline RL algorithm (e.g., IQL), enabling the cost critics to be learned independently of the reward for improved training stability. Through extensive cross-domain validation---ranging from interpretable grid-worlds and the high-dimensional DSRL benchmark to a complex, real-world maritime navigation task---\our\ consistently matches or exceeds state-of-the-art performance while reliably ensuring safety.
}



\bibliography{content/zotero}
\onecolumn
\appendix
\section{Theoretical Results}

\subsection{Budget Update Functions for Fully Deterministic CMDPs}
\label{app:det-proof}

In fully deterministic CMDPs, where transitions and trajectories are uniquely determined by the policy and environment dynamics, the evolution of the budget can be defined directly through a budget update rules below.

\begin{definition}[Direct Budget-Tracking]
Let $\MD$ denote the Direct Budget Tracking \ourmdp\ defined by
\[
\MD(\mc{M}) := \bar{\mc{M}}(\mc{M}, f, g).
\]
Budget update functions are:
\[
f(s_0, \cta) = \cta,
\qquad
g(s,a,s',\ct) = \frac{\ct - c(s,a)}{\gamma}.
\]
\end{definition}

This formulation keeps track of the remaining budget, and the division by $\gamma$ accounts for the discounting effect of the CMDP constraint.
In a deterministic environment, tracking such a budget and restricting actions to the persistent safe action set $\Ap$ is both necessary and sufficient for satisfying the original cumulative cost constraint. That is, enforcing $\pi \in \PiP$ in $\MD$ is equivalent to satisfying the CMDP constraint in Eq.~\eqref{eq:cmdp_cost}. The formal theorem and proof are provided in Appendix~\ref{app:det-proof}. Therefore, during training we can restrict the policy to select actions only from the persistent safe action set $\Ap$ \eqref{eq:def_ap}, without explicitly enforcing the cumulative cost constraint \eqref{eq:cmdp_cost}.

\newcommand{\thmDetMain}{
Let $\MD(\mc{M})$ be the Direct Budget-Tracking \ourmdp\ defined above, and assume the underlying CMDP $\mc{M}$ is deterministic with cost budget $\cta \ge \vcs(s_0)$. Then,  
\begin{enumerate}
    \item Any policy $\pi \in \PiP$ induces trajectories that start in and remain entirely within the feasible subspace $\SbP$.
    \item Any deterministic policy $\pi$ satisfying the CMDP cost constraint (Eq.~\ref{eq:cmdp_cost}) belongs to $\PiP$.
    \item Conversely, if $\pi \in \PiP$, then it satisfies the cumulative cost constraint of the underlying CMDP $\mc{M}$.
\end{enumerate}
Therefore, in deterministic environments, it is sufficient to enforce $\pi \in \PiP$ instead of explicitly enforcing the cumulative cost constraint.
}

\begin{theorem}\label{thm:det-main}
\thmDetMain
\end{theorem}

\begin{proof}
Let $\bM_{\mathrm{Det}}(\mc{M})$ and notation be as above, and assume $\mc{M}$ is deterministic. Denote by $(s_t,a_t)_{t\ge0}$ the state–action sequence induced by a deterministic policy $\pi$ started at $s_0$, and let $c_t := c(s_t,a_t)$. Recall the direct budget update
\(
\ct_{t+1} = \frac{\ct_t - c_t}{\gamma}, \qquad \ct_0 = \cta.
\)

\paragraph{(1) Any $\pi \in \PiP$ induces trajectories that remain in $\SbP$.}
Let $\pi \in \PiP$ and consider the trajectory starting from $(s_0,\ct_0)$. By assumption, $\vcs(s_0) \leq \cta$, and since $\ct_0 = \cta$, it follows that $s_0 \in \Sp(\ct_0)$. 
Since $\pi \in \PiP$ we get $\qcs(s_0,a_0)\le \ct_0$. Using the bellman-identity $\qcs(s_0,a_0)=c_0+\gamma\vcs(s_1)$ we obtain
\[
\vcs(s_1)=\frac{\qcs(s_0,a_0) - c_0}{\gamma} \le \frac{\ct_0 - c_0}{\gamma} =\ct_1.
\]
Since cost is always positive, $\vcs(s_1) \geq 0$, which implies $\ct_1 \in \sR^+$ and $s_1 \in \Sp(\ct_1)$. Therefore, by definition, $(s_1,\ct_1) \in \SbP$. Moreover, there exists at least one action at $s_1$ with $\qcs(\cdot) \leq \ct_1$. Repeating the same argument at each step shows that all subsequent augmented states $(s_t,\ct_t)$ remain in $\SbP$. Thus, trajectories under $\pi$ never leave $\SbP$.

\paragraph{(2) If $\pi$ satisfies the cumulative cost constraint then $\pi\in\PiP$.}
Assume the policy $\pi$ satisfies the cumulative cost constraint in $\mc{M}$, i.e.
\[
J(\tau)=\sum_{t=0}^\infty \gamma^t c_t \le \cta = \ct_0.
\]
Unrolling the budget recursion gives the identity
\[
\gamma^t \ct_t = \ct_0 - \sum_{i=0}^{t-1} \gamma^i c_i \qquad (t\ge1),
\]
hence, since $J(\tau)\le \ct_0$,
\[
\gamma^t \ct_t \;\ge\; \sum_{k=t}^\infty \gamma^k c_k \;=\; \gamma^t \sum_{k=t}^\infty \gamma^{k-t} c_k.
\]
Dividing by $\gamma^t$ yields
\[
\ct_t \ge \sum_{k=t}^\infty \gamma^{k-t} c_k \quad (=:\; R_t),
\]
where $R_t$ is the actual remaining discounted cost from time $t$. For the deterministic MDP we have for the chosen action
\[
\qcs(s_t,a_t) = c(s_t,a_t) + \gamma \vcs(s_{t+1}),
\]
and since $\vcs(s_{t+1})$ is the \emph{minimal} possible remaining cost from $s_{t+1}$, it holds that $\vcs(s_{t+1}) \le \sum_{k=t+1}^\infty \gamma^{k-(t+1)} c(s_k,a_k)$. Therefore
\[
\qcs(s_t,a_t) \le c(s_t,a_t) + \gamma \sum_{k=t+1}^\infty \gamma^{k-(t+1)} c_k = R_t \le \ct_t.
\]
Thus at every visited augmented state $(s_t,\ct_t)$ the chosen action satisfies $\qcs(s_t,a_t)\le\ct_t$, so $\pi\in\PiP$.

\paragraph{(3) If $\pi\in\PiP$ then $\pi$ satisfies the cumulative cost constraint.}
Assume $\pi\in\PiP$, i.e.\ for every $t$ the chosen action satisfies $\qcs(s_t,a_t)\le\ct_t$. By the direct budget update we have equality
\[
c(s_t, a_t) = \ct_t - \gamma \ct_{t+1},
\]
so multiplying by $\gamma^t$ and summing over $t\ge0$ gives the telescoping series
\[
\sum_{t=0}^T \gamma^t c(s_t,a_t) = \sum_{t=0}^T \big(\gamma^t \ct_t - \gamma^{t+1}\ct_{t+1}\big) = \ct_0 - \gamma^{T+1}\ct_{T+1}.
\]
Since $\ct_{T+1}\ge0$ for all $T$, letting $T\to\infty$ yields
\[
\sum_{t=0}^\infty \gamma^t c(s_t,a_t) \le \ct_0,
\]
so the cumulative cost constraint holds.

\paragraph{Conclusion.} Combining (1)--(3) we obtain the claimed equivalence in the deterministic setting: enforcing $\pi\in\PiP$ is both necessary and sufficient for satisfying the cumulative cost constraint, and any policy in $\PiP$ produces trajectories that stay in $\SbP$. Therefore, in deterministic environments it suffices to enforce $\pi\in\PiP$ instead of directly enforcing the cumulative cost constraint.
\end{proof}

\subsection{Stochastic Case}\label{app:stoch-proof}

\begin{lemma}[Reduction to Deterministic Budget-Tracking]
In a fully deterministic setting, the soft budget-tracking \ourmdp\ $\MS$ simplifies to the direct budget-tracking \ourmdp\ $\MD$.
\end{lemma}

\begin{proof}
In a deterministic MDP, a well trained cost critic satisfies
\(
 \qcs(s,a) = c(s,a) + \gamma \vcs(s') 
\)
since there is no stochasticity and the Bellman equation reduces to the deterministic form.  
Substituting this into the soft budget-tracking update:
\[
g(s,a,s',\ct) = \vcs(s') + \frac{\ct - \qcs(s,a)}{\gamma}
= \frac{\ct - c(s,a)}{\gamma},
\]
which is exactly the budget update defined in the direct budget-tracking CMDP $\MD$.  

Similarly in the deterministic setting $\vcs(s_0) = \mathbb{E}_{s \sim \mu_0}[\vcs(s)]$. Therefore, the initial budget:
\[
f(s_0, \cta) = \vcs(s_0) + \cta - \mathbb{E}_{s \sim \mu_0}[\vcs(s)] = \cta,
\]

Thus, $\MS$ reduces to $\MD$ in the deterministic setting.
\end{proof}

\begingroup
\renewcommand\thetheorem{\ref{thm:stoch-main}}
\begin{theorem}
\thmStochMain{0}
\end{theorem}
\addtocounter{theorem}{-1}
\endgroup

\begin{proof}
We prove each property separately:

\paragraph{1. Trajectories remain in $\SbP$.}  
Consider any policy $\pi \in \PiP$ and the augmented state trajectory $\{(s_t,\ct_t)\}_{t\ge 0}$ it induces.  
From Assumption~\ref{eq:assumption-soft}, we have $\cta \ge \E_{s \sim \mu_0}[V_C^*(s)]$. 
Substituting this into the budget initialization function $f$ gives, \[
\ct_0 = \vcs(s_0) + \cta - \E_{s \sim \mu_0}[V_C^*(s)] \ge V_C^*(s_0),
\] and therefore $s_0 \in \Sp$, and $(s_0,\ct_0) \in \SbP$.

At each time step $t$, the policy selects an action $a_t \in \Ap(\ct_t)$, which ensures $Q_C^*(s_t,a_t)\le \ct_t$.  
The budget update in $\MS$ gives
\[
\ct_{t+1} = V_C^*(s_{t+1}) + \frac{\ct_t - Q_C^*(s_t,a_t)}{\gamma} \ge V_C^*(s_{t+1}),
\]
so $(s_{t+1},\ct_{t+1}) \in \SbP$.  

By induction, all subsequent augmented states remain in $\SbP$. Therefore, trajectories under any $\pi \in \PiP$ start in and remain entirely within the feasible subspace $\SbP$.

\paragraph{2. Pointwise budget satisfaction.}
Fix an augmented state \((s,\ct)\in\SbP\) and a policy \(\pi\in\PiP\). By \((s,\ct)\in\SbP\) we have $s\in \Sp(\ct)$ and \(\ct \ge V_C^*(s)\), and by \(\pi\in\PiP\) the support condition
\(
\mathrm{supp}\,\pi(\cdot\mid(s,\ct)) \subseteq \{a:\; Q_C^*(s,a)\le\ct\}
\)
holds. 

Define the finite-horizon truncated action-value in the augmented CMDP:
\[
Q_C^{\pi,h}((s,\ct),a)
:= \E_{\pi}\Big[\sum_{t=0}^{h} \gamma^t c(s_t,a_t)\ \Big|\ (s_0,\ct_0)=(s,\ct),\,a_0=a\Big],
\]
and the truncated state-value
\[
V_C^{\pi,h}((s,\ct)) := \E_{a\sim\pi(\cdot\mid(s,\ct))}\big[ Q_C^{\pi,h}((s,\ct),a)\big].
\]

We show by induction on the horizon \(h\) that for every action \(a\) and every \(h\ge0\),
\begin{equation}\label{eq:trunc-bound}
Q_C^{\pi,h}((s,\ct),a) \le \ct.
\end{equation}

\paragraph{Base ($h=0$).} \(Q_C^{\pi,0}((s,\ct),a)=c(s,a)\). Since \(c(s,a)\le Q_C^*(s,a)\) and actions in the support satisfy \(Q_C^*(s,a)\le\ct\), we get \(Q_C^{\pi,0}((s,\ct),a)\le\ct\).

\paragraph{Inductive step.} Assume \eqref{eq:trunc-bound} holds for horizon \(h\). Consider horizon \(h+1\). Using the recursion for truncated values and the soft budget update
\(
\ct' \;=\; V_C^*(s') \;+\; \frac{\ct - Q_C^*(s,a)}{\gamma},
\)
we have
\begin{align*}
Q_C^{\pi,h+1}((s,\ct),a)
&= c(s,a) + \gamma\,\E_{s'\sim \T(\cdot\mid s,a)}\big[ V_C^{\pi,h}((s',\ct')) \big] \\
&\le c(s,a) + \gamma\,\E_{s'}\big[ \ct' \big] 
\qquad\text{(by the inductive hypothesis applied at each }(s',\ct'))\\
&= c(s,a) + \gamma\,\E_{s'}\!\Big[ V_C^*(s') + \frac{\ct - Q_C^*(s,a)}{\gamma} \Big] \\
&= c(s,a) + \gamma\,\E_{s'}[V_C^*(s')] + (\ct - Q_C^*(s,a)) \\
&= Q_C^*(s,a) + (\ct - Q_C^*(s,a)) \qquad\text{(Bellman equation for }Q_C^*)\\
&= \ct,
\end{align*}
so \eqref{eq:trunc-bound} holds for \(h+1\).

\paragraph{Pass to the limit.} Letting \(h\to\infty\) (justified by \(\gamma\in[0,1)\) and nonnegative bounded costs) yields
\[
Q_C^{\pi}((s,\ct),a) \le \ct \qquad\forall a\in A,
\]
and therefore
\begin{equation}\label{eq:point-budget-satisfy}
V_C^\pi((s,\ct)) = \E_{a\sim\pi(\cdot\mid(s,\ct))}\big[Q_C^\pi((s,\ct),a)\big] \le \ct.
\end{equation}
By definition \(J_C^\pi((s,\ct)) = V_C^\pi((s,\ct))\), hence \(J_C^\pi((s,\ct)) \le \ct\). This holds for every \((s,\ct)\in\SbP\), proving the theorem.

\paragraph{3. Global CMDP feasibility}
Based on the budget initialization function for every $s_0$ the initialized budget
\[
\ct_0 \;=\; V_C^*(s_0) + \cta - \E_{s \sim \mu_0}[V_C^*(s)]
\]
Taking expectation over $s_0\sim\mu_0$ gives
\begin{align}
J_C(\pi) &= \E_{(s_0, \ct_0)\sim \bmu}\big[J_C^\pi((s_0,\ct_0))\big] \\
&= \E_{(s_0, \ct_0)\sim \bmu}\big[V_C^\pi((s_0,\ct_0))\big] \\
&\le \E_{(s_0, \ct_0)\sim \bmu}[\ct_0] \quad \text{From Eq.~\ref{eq:point-budget-satisfy}} \\
&= \E_{s_0 \sim \mu_0}\big[\vcs(s_0) + \cta - \E_{s \sim \mu_0} [\vcs(s)]\big]\\
&=  \cta.
\end{align}
This completes the proof and shows that $\pi \in \PiP$ satisfy the cumulative cost constraint in the underlying CMDP.
\end{proof}


\subsection{Extension to Multiple Cost Objectives}
\label{app:multiple_costs}

In many real-world applications, agents must satisfy multiple distinct safety constraints simultaneously (e.g., avoiding collisions while maintaining battery levels). Our \textit{Budget-Conditioned Reachability} (\our) framework naturally extends to this setting by augmenting the state space with a \textit{budget vector}. Let there be $N$ distinct cost functions $\{c_i\}_{i=1}^N$, where each $c_i: S \times A \rightarrow \mathbb{R}^+$, and a corresponding vector of initial cost thresholds $\boldsymbol{\delta}_{init} = [\delta_{init, 1}, \dots, \delta_{init, N}]^T$.

\paragraph{Augmented State Space}
The augmented state space is defined as $\bar{S} := S \times (\mathbb{R}^+)^N$. A state is represented as $(s, \boldsymbol{\delta})$, where $\boldsymbol{\delta} \in \mathbb{R}^N$ tracks the remaining budget for each constraint independently.

\paragraph{Persistent Safety Set Intersection}
The budget-conditioned persistent safety set is defined as the intersection of the feasible sets for each individual cost objective. For a budget vector $\boldsymbol{\delta}$, the safe state set is:
\begin{equation}
    S_P(\boldsymbol{\delta}) := \bigcap_{i=1}^N \{ s \in S \mid V_{C_i}^*(s) \leq \delta_i \}
\end{equation}
Similarly, the safe action set $A_P(s, \boldsymbol{\delta})$ is the intersection of actions safe for all objectives:
\begin{equation}
    A_P(s, \boldsymbol{\delta}) := \bigcap_{i=1}^N \{ a \in A \mid Q_{C_i}^*(s, a) \leq \delta_i \}
\end{equation}

\paragraph{Vectorized Budget Initialization and Update}
The budget is initialized and updated using vector-valued functions $\mathbf{f}$ and $\mathbf{g}$, which generalize the scalar functions from the main text.

The initialization is given by $\boldsymbol{\delta}_0 = \mathbf{f}(s_0, \boldsymbol{\delta}_{init})$, where $\mathbf{f}: S \times (\mathbb{R}^+)^N \rightarrow (\mathbb{R}^+)^N$. Typically, this operates element-wise for each objective $i$:
\begin{equation}
    \delta_{0,i} = f_i(s_0, \delta_{init, i})
\end{equation}

The budget update rule is defined by $\boldsymbol{\delta}' = \mathbf{g}(s, a, s', \boldsymbol{\delta})$, where $\mathbf{g}: S \times A \times S \times (\mathbb{R}^+)^N \rightarrow (\mathbb{R}^+)^N$. For independent constraints, this update is applied element-wise using the scalar update function $g_i$ for each cost component:
\begin{equation}
    \delta'_i = g_i(s, a, s', \delta_i)
\end{equation}

\paragraph{Training Considerations}
We learn $N$ independent cost-critics $\{Q_{C_i}^*, V_{C_i}^*\}_{i=1}^N$, one for each cost type. Since the persistent safety sets are computed independently of the policy, these critics can be trained in parallel. During the policy learning phase, we sample a budget vector $\boldsymbol{\delta}$ such that $\delta_i \sim \mathcal{U}_{[Q_{C_i}^*(s,a), \delta_{max, i}]}$ for all $i$, ensuring the augmented transition remains within the intersection of all persistent safety sets.

\section{Additional Experiments}\label{app:exp}

In additional experiments, we:  
\textbf{(1)} main results for both stochastic and deterministic approaches;  
\textbf{(2)} include an additional comparison with algorithms that are not adaptive to different cost budgets; 
\textbf{(2)} more details on the real-world maritime navigation experiment;
\textbf{(4) } provide an ablation study on the expectile parameters for reward and cost ($\tau_R$, $\tau_C$), the temperature parameter $\beta$ of IQL instantiation and the number of budget samples $n$ used during training;  
\textbf{(5)} evaluate the quality of the learned feasible action set; 
\textbf{(6)}compare our results with the best policy obtained by LSPC to address the discrepancy discussed in Section~\ref{app:lspc}.

\subsection{Main Results with Stochastic and Deterministic Versions of Our Algorithm}

\begin{table*}[t]
    \centering
\resizebox{\textwidth}{!}
{
\small 
\begin{tabular}{lllllllllll}
\hline
 & \multicolumn{2}{c}{CDT} & \multicolumn{2}{c}{CAPS} & \multicolumn{2}{c}{CCAC} & \multicolumn{2}{c}{LSPC} & \multicolumn{2}{c}{\our\ (Ours)} \\
 & reward $\uparrow$ & cost $\downarrow$ & reward $\uparrow$ & cost $\downarrow$ & reward $\uparrow$ & cost $\downarrow$ & reward $\uparrow$ & cost $\downarrow$ & reward $\uparrow$ & cost $\downarrow$ \\
\hline
PointButton1 & \gray{0.53$_{\pm 0.01}$} & \gray{1.68$_{\pm 0.13}$} & \textbf{0.04$_{\pm 0.15}$} & \textbf{0.58$_{\pm 1.42}$} & \gray{0.68$_{\pm 0.14}$} & \gray{4.31$_{\pm 3.20}$} & \textbf{\blue{0.09$_{\pm 0.20}$}} & \textbf{\blue{0.69$_{\pm 1.03}$}} & \textbf{0.07$_{\pm 0.18}$} & \textbf{0.79$_{\pm 2.10}$} \\
PointButton2 & \gray{0.46$_{\pm 0.01}$} & \gray{1.57$_{\pm 0.10}$} & \textbf{\blue{0.12$_{\pm 0.20}$}} & \textbf{\blue{0.96$_{\pm 2.35}$}} & \gray{0.63$_{\pm 0.11}$} & \gray{4.82$_{\pm 3.61}$} & \gray{0.19$_{\pm 0.21}$} & \gray{1.39$_{\pm 2.53}$} & \textbf{\blue{0.12$_{\pm 0.22}$}} & \textbf{\blue{0.81$_{\pm 1.36}$}} \\
PointCircle1 & \textbf{\blue{0.59$_{\pm 0.00}$}} & \textbf{\blue{0.69$_{\pm 0.04}$}} & \textbf{0.50$_{\pm 0.12}$} & \textbf{0.68$_{\pm 1.51}$} & \gray{0.58$_{\pm 0.13}$} & \gray{2.22$_{\pm 1.05}$} & \gray{0.56$_{\pm 0.35}$} & \gray{3.49$_{\pm 3.55}$} & \textbf{0.55$_{\pm 0.17}$} & \textbf{0.86$_{\pm 0.21}$} \\
PointCircle2 & \gray{0.64$_{\pm 0.01}$} & \gray{1.05$_{\pm 0.08}$} & \textbf{0.51$_{\pm 0.21}$} & \textbf{0.76$_{\pm 0.47}$} & \gray{0.18$_{\pm 0.50}$} & \gray{1.42$_{\pm 1.09}$} & \gray{0.67$_{\pm 0.24}$} & \gray{4.18$_{\pm 4.56}$} & \textbf{\blue{0.58$_{\pm 0.10}$}} & \textbf{\blue{0.74$_{\pm 0.34}$}} \\
PointGoal1 & \gray{0.69$_{\pm 0.02}$} & \gray{1.12$_{\pm 0.07}$} & \textbf{0.46$_{\pm 0.28}$} & \textbf{0.63$_{\pm 0.81}$} & \gray{0.55$_{\pm 0.26}$} & \gray{1.88$_{\pm 1.72}$} & \textbf{0.26$_{\pm 0.27}$} & \textbf{0.23$_{\pm 0.45}$} & \textbf{\blue{0.56$_{\pm 0.25}$}} & \textbf{\blue{0.70$_{\pm 0.65}$}} \\
PointGoal2 & \gray{0.59$_{\pm 0.03}$} & \gray{1.34$_{\pm 0.05}$} & \textbf{0.34$_{\pm 0.23}$} & \textbf{0.80$_{\pm 0.74}$} & \gray{0.41$_{\pm 0.36}$} & \gray{3.60$_{\pm 3.37}$} & \textbf{0.23$_{\pm 0.24}$} & \textbf{0.50$_{\pm 0.79}$} & \textbf{\blue{0.38$_{\pm 0.21}$}} & \textbf{\blue{0.77$_{\pm 0.68}$}} \\
PointPush1 & \textbf{\blue{0.24$_{\pm 0.02}$}} & \textbf{\blue{0.48$_{\pm 0.05}$}} & \textbf{0.15$_{\pm 0.17}$} & \textbf{0.33$_{\pm 0.66}$} & \gray{0.02$_{\pm 0.33}$} & \gray{1.42$_{\pm 5.37}$} & \textbf{0.14$_{\pm 0.18}$} & \textbf{0.37$_{\pm 0.77}$} & \textbf{0.14$_{\pm 0.16}$} & \textbf{0.47$_{\pm 0.79}$} \\
PointPush2 & \textbf{\blue{0.21$_{\pm 0.04}$}} & \textbf{\blue{0.65$_{\pm 0.03}$}} & \textbf{0.12$_{\pm 0.19}$} & \textbf{0.60$_{\pm 1.12}$} & \textbf{-0.18$_{\pm 0.66}$} & \textbf{1.00$_{\pm 1.70}$} & \textbf{0.12$_{\pm 0.17}$} & \textbf{0.80$_{\pm 2.65}$} & \textbf{0.14$_{\pm 0.18}$} & \textbf{0.71$_{\pm 1.11}$} \\
CarButton1 & \gray{0.21$_{\pm 0.02}$} & \gray{1.60$_{\pm 0.12}$} & \textbf{\blue{-0.02$_{\pm 0.17}$}} & \textbf{\blue{0.26$_{\pm 0.53}$}} & \gray{0.45$_{\pm 0.26}$} & \gray{11.31$_{\pm 8.91}$} & \textbf{-0.03$_{\pm 0.23}$} & \textbf{0.59$_{\pm 1.42}$} & \textbf{-0.06$_{\pm 0.18}$} & \textbf{0.34$_{\pm 0.76}$} \\
CarButton2 & \gray{0.13$_{\pm 0.01}$} & \gray{1.58$_{\pm 0.02}$} & \textbf{\blue{-0.07$_{\pm 0.18}$}} & \textbf{\blue{0.59$_{\pm 2.07}$}} & \gray{0.53$_{\pm 0.26}$} & \gray{9.82$_{\pm 9.00}$} & \textbf{-0.13$_{\pm 0.28}$} & \textbf{0.93$_{\pm 2.06}$} & \textbf{-0.20$_{\pm 0.26}$} & \textbf{0.27$_{\pm 0.62}$} \\
CarCircle1 & \gray{0.60$_{\pm 0.01}$} & \gray{1.73$_{\pm 0.04}$} & \gray{0.55$_{\pm 0.16}$} & \gray{1.54$_{\pm 1.63}$} & \gray{0.49$_{\pm 0.14}$} & \gray{5.13$_{\pm 3.53}$} & \gray{0.46$_{\pm 0.25}$} & \gray{1.86$_{\pm 3.14}$} & \textbf{\blue{0.50$_{\pm 0.12}$}} & \textbf{\blue{0.48$_{\pm 0.43}$}} \\
CarCircle2 & \gray{0.66$_{\pm 0.00}$} & \gray{2.53$_{\pm 0.03}$} & \gray{0.50$_{\pm 0.18}$} & \gray{1.56$_{\pm 2.32}$} & \gray{0.54$_{\pm 0.08}$} & \gray{2.75$_{\pm 4.36}$} & \gray{0.43$_{\pm 0.27}$} & \gray{2.40$_{\pm 4.55}$} & \textbf{\blue{0.48$_{\pm 0.15}$}} & \textbf{\blue{0.44$_{\pm 0.48}$}} \\
CarGoal1 & \gray{0.66$_{\pm 0.01}$} & \gray{1.21$_{\pm 0.17}$} & \textbf{0.32$_{\pm 0.26}$} & \textbf{0.33$_{\pm 0.53}$} & \gray{0.84$_{\pm 0.08}$} & \gray{1.81$_{\pm 1.91}$} & \textbf{0.27$_{\pm 0.26}$} & \textbf{0.24$_{\pm 0.44}$} & \textbf{\blue{0.43$_{\pm 0.27}$}} & \textbf{\blue{0.43$_{\pm 0.58}$}} \\
CarGoal2 & \gray{0.48$_{\pm 0.01}$} & \gray{1.25$_{\pm 0.14}$} & \textbf{0.16$_{\pm 0.21}$} & \textbf{0.71$_{\pm 1.91}$} & \gray{0.90$_{\pm 0.12}$} & \gray{5.53$_{\pm 5.01}$} & \textbf{0.16$_{\pm 0.21}$} & \textbf{0.41$_{\pm 0.77}$} & \textbf{\blue{0.19$_{\pm 0.20}$}} & \textbf{\blue{0.47$_{\pm 0.92}$}} \\
CarPush1 & \textbf{\blue{0.31$_{\pm 0.01}$}} & \textbf{\blue{0.40$_{\pm 0.10}$}} & \textbf{0.19$_{\pm 0.13}$} & \textbf{0.34$_{\pm 1.09}$} & \textbf{-0.12$_{\pm 0.70}$} & \textbf{0.59$_{\pm 0.97}$} & \textbf{0.18$_{\pm 0.15}$} & \textbf{0.41$_{\pm 1.29}$} & \textbf{0.19$_{\pm 0.14}$} & \textbf{0.24$_{\pm 0.74}$} \\
CarPush2 & \gray{0.19$_{\pm 0.01}$} & \gray{1.30$_{\pm 0.16}$} & \textbf{0.05$_{\pm 0.14}$} & \textbf{0.71$_{\pm 1.66}$} & \gray{0.12$_{\pm 0.31}$} & \gray{6.06$_{\pm 9.53}$} & \textbf{\blue{0.07$_{\pm 0.15}$}} & \textbf{\blue{0.80$_{\pm 1.41}$}} & \textbf{0.05$_{\pm 0.17}$} & \textbf{0.52$_{\pm 1.07}$} \\
SwimmerVelocity & \textbf{\blue{0.66$_{\pm 0.01}$}} & \textbf{\blue{0.96$_{\pm 0.08}$}} & \gray{0.44$_{\pm 0.22}$} & \gray{1.66$_{\pm 2.37}$} & \textbf{-0.08$_{\pm 0.02}$} & \textbf{0.00$_{\pm 0.01}$} & \gray{0.46$_{\pm 0.21}$} & \gray{1.28$_{\pm 2.60}$} & \textbf{0.49$_{\pm 0.18}$} & \textbf{0.99$_{\pm 0.37}$} \\
HopperVelocity & \textbf{0.63$_{\pm 0.06}$} & \textbf{0.61$_{\pm 0.08}$} & \textbf{0.44$_{\pm 0.26}$} & \textbf{0.75$_{\pm 0.83}$} & \textbf{0.46$_{\pm 0.38}$} & \textbf{0.50$_{\pm 0.55}$} & \gray{0.21$_{\pm 0.04}$} & \gray{1.25$_{\pm 1.20}$} & \textbf{\blue{0.66$_{\pm 0.23}$}} & \textbf{\blue{0.80$_{\pm 0.35}$}} \\
HalfCheetahVelocity & \textbf{\blue{1.00$_{\pm 0.01}$}} & \textbf{\blue{0.01$_{\pm 0.01}$}} & \textbf{0.94$_{\pm 0.04}$} & \textbf{0.77$_{\pm 0.13}$} & \textbf{0.94$_{\pm 0.04}$} & \textbf{0.94$_{\pm 0.07}$} & \textbf{0.97$_{\pm 0.02}$} & \textbf{1.00$_{\pm 1.04}$} & \textbf{0.93$_{\pm 0.07}$} & \textbf{0.62$_{\pm 0.27}$} \\
Walker2dVelocity & \textbf{0.78$_{\pm 0.09}$} & \textbf{0.06$_{\pm 0.34}$} & \textbf{\blue{0.80$_{\pm 0.04}$}} & \textbf{\blue{0.70$_{\pm 0.51}$}} & \gray{0.17$_{\pm 0.28}$} & \gray{1.08$_{\pm 1.90}$} & \textbf{0.79$_{\pm 0.01}$} & \textbf{0.17$_{\pm 0.24}$} & \textbf{0.79$_{\pm 0.06}$} & \textbf{0.14$_{\pm 0.35}$} \\
AntVelocity & \textbf{\blue{0.98$_{\pm 0.00}$}} & \textbf{\blue{0.39$_{\pm 0.12}$}} & \textbf{0.95$_{\pm 0.04}$} & \textbf{0.64$_{\pm 0.21}$} & \textbf{0.51$_{\pm 0.41}$} & \textbf{0.59$_{\pm 0.34}$} & \textbf{0.97$_{\pm 0.05}$} & \textbf{0.14$_{\pm 0.15}$} & \textbf{0.97$_{\pm 0.03}$} & \textbf{0.34$_{\pm 0.14}$} \\
\hline
SafetyGym Average (21 tasks) & \gray{0.54$_{\pm 0.21}$} & \gray{1.06$_{\pm 0.59}$} & \textbf{0.36$_{\pm 0.34}$} & \textbf{0.76$_{\pm 1.43}$} & \gray{0.41$_{\pm 0.46}$} & \gray{3.18$_{\pm 5.24}$} & \gray{0.34$_{\pm 0.19}$} & \gray{1.10$_{\pm 1.75}$} & \textbf{\blue{0.38$_{\pm 0.17}$}} & \textbf{\blue{0.57$_{\pm 0.68}$}} \\
\hline
\hline
BallRun & \gray{0.39$_{\pm 0.09}$} & \gray{1.16$_{\pm 0.19}$} & \textbf{0.18$_{\pm 0.09}$} & \textbf{0.94$_{\pm 0.73}$} & \textbf{\blue{0.40$_{\pm 0.12}$}} & \textbf{\blue{0.84$_{\pm 0.31}$}} & \textbf{0.15$_{\pm 0.01}$} & \textbf{0.00$_{\pm 0.00}$} & \textbf{0.20$_{\pm 0.02}$} & \textbf{0.06$_{\pm 0.10}$} \\
CarRun & \textbf{\blue{0.99$_{\pm 0.01}$}} & \textbf{\blue{0.65$_{\pm 0.31}$}} & \textbf{0.97$_{\pm 0.01}$} & \textbf{0.23$_{\pm 0.25}$} & \textbf{0.97$_{\pm 0.04}$} & \textbf{0.68$_{\pm 0.22}$} & \textbf{0.98$_{\pm 0.01}$} & \textbf{0.84$_{\pm 1.04}$} & \textbf{0.98$_{\pm 0.01}$} & \textbf{0.98$_{\pm 0.72}$} \\
DroneRun & \textbf{\blue{0.63$_{\pm 0.04}$}} & \textbf{\blue{0.79$_{\pm 0.68}$}} & \gray{0.45$_{\pm 0.10}$} & \gray{2.91$_{\pm 4.46}$} & \gray{0.46$_{\pm 0.18}$} & \gray{4.36$_{\pm 4.32}$} & \textbf{0.59$_{\pm 0.06}$} & \textbf{0.93$_{\pm 1.51}$} & \textbf{0.58$_{\pm 0.03}$} & \textbf{0.54$_{\pm 0.91}$} \\
AntRun & \textbf{\blue{0.72$_{\pm 0.04}$}} & \textbf{\blue{0.91$_{\pm 0.42}$}} & \textbf{0.61$_{\pm 0.13}$} & \textbf{0.78$_{\pm 0.59}$} & \textbf{0.12$_{\pm 0.09}$} & \textbf{0.04$_{\pm 0.08}$} & \gray{0.66$_{\pm 0.13}$} & \gray{1.43$_{\pm 1.25}$} & \textbf{0.68$_{\pm 0.11}$} & \textbf{0.95$_{\pm 0.29}$} \\
BallCircle & \gray{0.77$_{\pm 0.06}$} & \gray{1.07$_{\pm 0.27}$} & \textbf{0.69$_{\pm 0.10}$} & \textbf{0.59$_{\pm 0.23}$} & \textbf{\blue{0.78$_{\pm 0.05}$}} & \textbf{\blue{0.46$_{\pm 0.31}$}} & \gray{0.62$_{\pm 0.28}$} & \gray{1.17$_{\pm 1.95}$} & \textbf{0.69$_{\pm 0.11}$} & \textbf{0.97$_{\pm 0.27}$} \\
CarCircle & \textbf{\blue{0.75$_{\pm 0.06}$}} & \textbf{\blue{0.95$_{\pm 0.61}$}} & \textbf{0.70$_{\pm 0.10}$} & \textbf{0.66$_{\pm 0.27}$} & \textbf{0.72$_{\pm 0.04}$} & \textbf{0.77$_{\pm 0.44}$} & \gray{0.78$_{\pm 0.12}$} & \gray{1.88$_{\pm 3.23}$} & \textbf{0.53$_{\pm 0.17}$} & \textbf{0.51$_{\pm 0.51}$} \\
DroneCircle & \textbf{\blue{0.63$_{\pm 0.07}$}} & \textbf{\blue{0.98$_{\pm 0.27}$}} & \textbf{0.55$_{\pm 0.06}$} & \textbf{0.65$_{\pm 0.29}$} & \textbf{0.29$_{\pm 0.27}$} & \textbf{0.63$_{\pm 0.70}$} & \gray{0.60$_{\pm 0.02}$} & \gray{1.31$_{\pm 0.86}$} & \textbf{0.42$_{\pm 0.12}$} & \textbf{0.52$_{\pm 0.32}$} \\
AntCircle & \gray{0.54$_{\pm 0.20}$} & \gray{1.78$_{\pm 4.33}$} & \textbf{0.37$_{\pm 0.18}$} & \textbf{0.15$_{\pm 0.25}$} & \textbf{\blue{0.61$_{\pm 0.14}$}} & \textbf{\blue{0.75$_{\pm 0.90}$}} & \textbf{0.44$_{\pm 0.14}$} & \textbf{0.53$_{\pm 0.93}$} & \textbf{0.56$_{\pm 0.17}$} & \textbf{0.79$_{\pm 0.60}$} \\
\hline
BulletGym Average (8 tasks) & \gray{0.68$_{\pm 0.19}$} & \gray{1.04$_{\pm 1.65}$} & \textbf{0.57$_{\pm 0.25}$} & \textbf{0.86$_{\pm 1.82}$} & \gray{0.54$_{\pm 0.30}$} & \gray{1.07$_{\pm 2.04}$} & \gray{0.60$_{\pm 0.10}$} & \gray{1.01$_{\pm 1.35}$} & \textbf{\blue{0.58$_{\pm 0.09}$}} & \textbf{\blue{0.67$_{\pm 0.46}$}} \\
\hline
\hline
easysparse & \textbf{0.17$_{\pm 0.14}$} & \textbf{0.23$_{\pm 0.32}$} & \textbf{0.12$_{\pm 0.20}$} & \textbf{0.37$_{\pm 0.43}$} & \textbf{-0.06$_{\pm 0.00}$} & \textbf{0.10$_{\pm 0.05}$} & \gray{0.79$_{\pm 0.14}$} & \gray{1.34$_{\pm 1.43}$} & \textbf{\blue{0.70$_{\pm 0.19}$}} & \textbf{\blue{0.94$_{\pm 0.17}$}} \\
easymean & \textbf{0.45$_{\pm 0.11}$} & \textbf{0.54$_{\pm 0.55}$} & \textbf{0.02$_{\pm 0.07}$} & \textbf{0.21$_{\pm 0.23}$} & \textbf{-0.06$_{\pm 0.01}$} & \textbf{0.07$_{\pm 0.08}$} & \gray{0.82$_{\pm 0.06}$} & \gray{1.33$_{\pm 0.86}$} & \textbf{\blue{0.73$_{\pm 0.17}$}} & \textbf{\blue{0.94$_{\pm 0.13}$}} \\
easydense & \textbf{0.32$_{\pm 0.18}$} & \textbf{0.62$_{\pm 0.43}$} & \textbf{0.11$_{\pm 0.15}$} & \textbf{0.16$_{\pm 0.17}$} & \textbf{-0.06$_{\pm 0.00}$} & \textbf{0.07$_{\pm 0.04}$} & \gray{0.84$_{\pm 0.10}$} & \gray{1.88$_{\pm 1.20}$} & \textbf{\blue{0.70$_{\pm 0.17}$}} & \textbf{\blue{0.90$_{\pm 0.16}$}} \\
mediumsparse & \gray{0.87$_{\pm 0.11}$} & \gray{1.10$_{\pm 0.26}$} & \textbf{0.59$_{\pm 0.36}$} & \textbf{0.74$_{\pm 0.97}$} & \textbf{-0.08$_{\pm 0.00}$} & \textbf{0.07$_{\pm 0.04}$} & \gray{0.97$_{\pm 0.02}$} & \gray{1.24$_{\pm 0.68}$} & \textbf{\blue{0.94$_{\pm 0.07}$}} & \textbf{\blue{0.82$_{\pm 0.32}$}} \\
mediummean & \textbf{0.45$_{\pm 0.39}$} & \textbf{0.75$_{\pm 0.83}$} & \textbf{0.66$_{\pm 0.35}$} & \textbf{0.90$_{\pm 0.89}$} & \textbf{-0.07$_{\pm 0.01}$} & \textbf{0.02$_{\pm 0.03}$} & \textbf{\blue{0.93$_{\pm 0.12}$}} & \textbf{\blue{0.78$_{\pm 0.51}$}} & \textbf{0.92$_{\pm 0.11}$} & \textbf{0.77$_{\pm 0.16}$} \\
mediumdense & \gray{0.88$_{\pm 0.12}$} & \gray{2.41$_{\pm 0.71}$} & \textbf{0.75$_{\pm 0.29}$} & \textbf{0.65$_{\pm 0.57}$} & \textbf{-0.08$_{\pm 0.00}$} & \textbf{0.06$_{\pm 0.03}$} & \gray{0.82$_{\pm 0.30}$} & \gray{1.07$_{\pm 0.71}$} & \textbf{\blue{0.78$_{\pm 0.28}$}} & \textbf{\blue{0.68$_{\pm 0.31}$}} \\
hardsparse & \textbf{0.25$_{\pm 0.08}$} & \textbf{0.41$_{\pm 0.33}$} & \textbf{0.45$_{\pm 0.15}$} & \textbf{0.75$_{\pm 0.56}$} & \textbf{-0.05$_{\pm 0.00}$} & \textbf{0.07$_{\pm 0.04}$} & \textbf{\blue{0.54$_{\pm 0.09}$}} & \textbf{\blue{0.91$_{\pm 0.51}$}} & \textbf{0.49$_{\pm 0.13}$} & \textbf{0.76$_{\pm 0.44}$} \\
hardmean & \textbf{0.33$_{\pm 0.21}$} & \textbf{0.97$_{\pm 0.31}$} & \textbf{0.29$_{\pm 0.13}$} & \textbf{0.28$_{\pm 0.31}$} & \textbf{-0.05$_{\pm 0.00}$} & \textbf{0.06$_{\pm 0.03}$} & \gray{0.48$_{\pm 0.14}$} & \gray{1.16$_{\pm 1.19}$} & \textbf{\blue{0.46$_{\pm 0.15}$}} & \textbf{\blue{0.84$_{\pm 0.54}$}} \\
harddense & \textbf{0.08$_{\pm 0.15}$} & \textbf{0.21$_{\pm 0.42}$} & \textbf{0.36$_{\pm 0.18}$} & \textbf{0.66$_{\pm 0.92}$} & \textbf{-0.03$_{\pm 0.01}$} & \textbf{0.11$_{\pm 0.08}$} & \gray{0.47$_{\pm 0.19}$} & \gray{1.43$_{\pm 0.93}$} & \textbf{\blue{0.44$_{\pm 0.17}$}} & \textbf{\blue{0.63$_{\pm 0.24}$}} \\
\hline
MetaDrive Average (9 tasks) & \textbf{0.42$_{\pm 0.31}$} & \textbf{0.80$_{\pm 0.61}$} & \textbf{0.38$_{\pm 0.34}$} & \textbf{0.54$_{\pm 0.69}$} & \textbf{-0.06$_{\pm 0.02}$} & \textbf{0.07$_{\pm 0.06}$} & \gray{0.74$_{\pm 0.13}$} & \gray{1.24$_{\pm 0.89}$} & \textbf{\blue{0.69$_{\pm 0.16}$}} & \textbf{\blue{0.81$_{\pm 0.27}$}} \\
\hline
\bottomrule
\end{tabular}
}
\caption{
\normalfont\small 
\textbf{Results for normalized cost and normalized reward:} \(\uparrow\) indicates that a higher value is better, while \(\downarrow\) indicates that a lower value is preferable.\textbf{Bold} text denotes safe policies, \gray{gray} indicates unsafe policies, and \blue{\textbf{blue}} highlights the highest reward among the safe policies for each task.  
 Each result is obtained by running three random seeds across three distinct cost thresholds and evaluating over 20 episodes.}
\label{app:main_table}
\end{table*}

In Table~\ref{tab:main-table-two-col}, we present results for both the stochastic and deterministic versions of our algorithm, shown side by side with other baselines. The main table in the paper reports results only for the default configuration (either the stochastic or deterministic variant, depending on the domain), while this section provides a detailed comparison of both. The results indicate that the stochastic version performs better in \textit{SafetyGym} and \textit{Bullet-SafetyGym}—domains with inherently stochastic dynamics—whereas the deterministic version achieves superior performance in \textit{MetaDrive}, which features deterministic transitions.
\begin{table*}[tb]
    \centering
\resizebox{\textwidth}{!}{

\begin{tabular}{lllllllll|llll}
\hline
 & \multicolumn{2}{r}{CDT} & \multicolumn{2}{r}{CAPS(IQL)} & \multicolumn{2}{r}{CCAC} & \multicolumn{2}{r}{LSPC-Final (Our-Run)} & \multicolumn{2}{r}{\our\ -Deterministic} & \multicolumn{2}{r}{\our\ -Stochastic} \\
 & reward $\uparrow$ & cost $\downarrow$ & reward $\uparrow$ & cost $\downarrow$ & reward $\uparrow$ & cost $\downarrow$ & reward $\uparrow$ & cost $\downarrow$ & reward $\uparrow$ & cost $\downarrow$ & reward $\uparrow$ & cost $\downarrow$ \\
\hline
PointButton1 & \gray{0.53$_{\pm 0.01}$} & \gray{1.68$_{\pm 0.13}$} & \textbf{0.04$_{\pm 0.15}$} & \textbf{0.58$_{\pm 1.42}$} & \gray{0.68$_{\pm 0.14}$} & \gray{4.31$_{\pm 3.20}$} & \textbf{0.09$_{\pm 0.20}$} & \textbf{0.69$_{\pm 1.03}$} & \textbf{\blue{0.16$_{\pm 0.20}$}} & \textbf{\blue{0.97$_{\pm 1.33}$}} & \textbf{0.07$_{\pm 0.18}$} & \textbf{0.79$_{\pm 2.10}$} \\
PointButton2 & \gray{0.46$_{\pm 0.01}$} & \gray{1.57$_{\pm 0.10}$} & \textbf{0.12$_{\pm 0.20}$} & \textbf{0.96$_{\pm 2.35}$} & \gray{0.63$_{\pm 0.11}$} & \gray{4.82$_{\pm 3.61}$} & \gray{0.19$_{\pm 0.21}$} & \gray{1.39$_{\pm 2.53}$} & \textbf{\blue{0.14$_{\pm 0.21}$}} & \textbf{\blue{0.72$_{\pm 0.84}$}} & \textbf{0.12$_{\pm 0.22}$} & \textbf{0.81$_{\pm 1.36}$} \\
PointCircle1 & \textbf{\blue{0.59$_{\pm 0.00}$}} & \textbf{\blue{0.69$_{\pm 0.04}$}} & \textbf{0.50$_{\pm 0.12}$} & \textbf{0.68$_{\pm 1.51}$} & \gray{0.58$_{\pm 0.13}$} & \gray{2.22$_{\pm 1.05}$} & \gray{0.56$_{\pm 0.35}$} & \gray{3.49$_{\pm 3.55}$} & \textbf{0.47$_{\pm 0.10}$} & \textbf{0.78$_{\pm 0.87}$} & \textbf{0.55$_{\pm 0.17}$} & \textbf{0.86$_{\pm 0.21}$} \\
PointCircle2 & \gray{0.64$_{\pm 0.01}$} & \gray{1.05$_{\pm 0.08}$} & \textbf{0.51$_{\pm 0.21}$} & \textbf{0.76$_{\pm 0.47}$} & \gray{0.18$_{\pm 0.50}$} & \gray{1.42$_{\pm 1.09}$} & \gray{0.67$_{\pm 0.24}$} & \gray{4.18$_{\pm 4.56}$} & \textbf{0.45$_{\pm 0.18}$} & \textbf{0.76$_{\pm 0.41}$} & \textbf{\blue{0.58$_{\pm 0.10}$}} & \textbf{\blue{0.74$_{\pm 0.34}$}} \\
PointGoal1 & \gray{0.69$_{\pm 0.02}$} & \gray{1.12$_{\pm 0.07}$} & \textbf{0.46$_{\pm 0.28}$} & \textbf{0.63$_{\pm 0.81}$} & \gray{0.55$_{\pm 0.26}$} & \gray{1.88$_{\pm 1.72}$} & \textbf{0.26$_{\pm 0.27}$} & \textbf{0.23$_{\pm 0.45}$} & \textbf{\blue{0.61$_{\pm 0.21}$}} & \textbf{\blue{0.88$_{\pm 0.65}$}} & \textbf{0.56$_{\pm 0.25}$} & \textbf{0.70$_{\pm 0.65}$} \\
PointGoal2 & \gray{0.59$_{\pm 0.03}$} & \gray{1.34$_{\pm 0.05}$} & \textbf{0.34$_{\pm 0.23}$} & \textbf{0.80$_{\pm 0.74}$} & \gray{0.41$_{\pm 0.36}$} & \gray{3.60$_{\pm 3.37}$} & \textbf{0.23$_{\pm 0.24}$} & \textbf{0.50$_{\pm 0.79}$} & \textbf{\blue{0.38$_{\pm 0.19}$}} & \textbf{\blue{0.99$_{\pm 0.49}$}} & \textbf{\blue{0.38$_{\pm 0.21}$}} & \textbf{\blue{0.77$_{\pm 0.68}$}} \\
PointPush1 & \textbf{\blue{0.24$_{\pm 0.02}$}} & \textbf{\blue{0.48$_{\pm 0.05}$}} & \textbf{0.15$_{\pm 0.17}$} & \textbf{0.33$_{\pm 0.66}$} & \gray{0.02$_{\pm 0.33}$} & \gray{1.42$_{\pm 5.37}$} & \textbf{0.14$_{\pm 0.18}$} & \textbf{0.37$_{\pm 0.77}$} & \textbf{0.19$_{\pm 0.16}$} & \textbf{0.45$_{\pm 0.65}$} & \textbf{0.14$_{\pm 0.16}$} & \textbf{0.47$_{\pm 0.79}$} \\
PointPush2 & \textbf{\blue{0.21$_{\pm 0.04}$}} & \textbf{\blue{0.65$_{\pm 0.03}$}} & \textbf{0.12$_{\pm 0.19}$} & \textbf{0.60$_{\pm 1.12}$} & \textbf{-0.18$_{\pm 0.66}$} & \textbf{1.00$_{\pm 1.70}$} & \textbf{0.12$_{\pm 0.17}$} & \textbf{0.80$_{\pm 2.65}$} & \textbf{0.12$_{\pm 0.19}$} & \textbf{0.88$_{\pm 2.85}$} & \textbf{0.14$_{\pm 0.18}$} & \textbf{0.71$_{\pm 1.11}$} \\
CarButton1 & \gray{0.21$_{\pm 0.02}$} & \gray{1.60$_{\pm 0.12}$} & \textbf{-0.02$_{\pm 0.17}$} & \textbf{0.26$_{\pm 0.53}$} & \gray{0.45$_{\pm 0.26}$} & \gray{11.31$_{\pm 8.91}$} & \textbf{-0.03$_{\pm 0.23}$} & \textbf{0.59$_{\pm 1.42}$} & \textbf{\blue{0.04$_{\pm 0.17}$}} & \textbf{\blue{0.62$_{\pm 0.68}$}} & \textbf{-0.06$_{\pm 0.18}$} & \textbf{0.34$_{\pm 0.76}$} \\
CarButton2 & \gray{0.13$_{\pm 0.01}$} & \gray{1.58$_{\pm 0.02}$} & \textbf{\blue{-0.07$_{\pm 0.18}$}} & \textbf{\blue{0.59$_{\pm 2.07}$}} & \gray{0.53$_{\pm 0.26}$} & \gray{9.82$_{\pm 9.00}$} & \textbf{-0.13$_{\pm 0.28}$} & \textbf{0.93$_{\pm 2.06}$} & \textbf{-0.16$_{\pm 0.25}$} & \textbf{0.36$_{\pm 0.51}$} & \textbf{-0.20$_{\pm 0.26}$} & \textbf{0.27$_{\pm 0.62}$} \\
CarCircle1 & \gray{0.60$_{\pm 0.01}$} & \gray{1.73$_{\pm 0.04}$} & \gray{0.55$_{\pm 0.16}$} & \gray{1.54$_{\pm 1.63}$} & \gray{0.49$_{\pm 0.14}$} & \gray{5.13$_{\pm 3.53}$} & \gray{0.46$_{\pm 0.25}$} & \gray{1.86$_{\pm 3.14}$} & \textbf{0.37$_{\pm 0.16}$} & \textbf{0.54$_{\pm 0.55}$} & \textbf{\blue{0.50$_{\pm 0.12}$}} & \textbf{\blue{0.48$_{\pm 0.43}$}} \\
CarCircle2 & \gray{0.66$_{\pm 0.00}$} & \gray{2.53$_{\pm 0.03}$} & \gray{0.50$_{\pm 0.18}$} & \gray{1.56$_{\pm 2.32}$} & \gray{0.54$_{\pm 0.08}$} & \gray{2.75$_{\pm 4.36}$} & \gray{0.43$_{\pm 0.27}$} & \gray{2.40$_{\pm 4.55}$} & \textbf{0.43$_{\pm 0.18}$} & \textbf{0.54$_{\pm 0.72}$} & \textbf{\blue{0.48$_{\pm 0.15}$}} & \textbf{\blue{0.44$_{\pm 0.48}$}} \\
CarGoal1 & \gray{0.66$_{\pm 0.01}$} & \gray{1.21$_{\pm 0.17}$} & \textbf{0.32$_{\pm 0.26}$} & \textbf{0.33$_{\pm 0.53}$} & \gray{0.84$_{\pm 0.08}$} & \gray{1.81$_{\pm 1.91}$} & \textbf{0.27$_{\pm 0.26}$} & \textbf{0.24$_{\pm 0.44}$} & \textbf{\blue{0.49$_{\pm 0.22}$}} & \textbf{\blue{0.64$_{\pm 0.69}$}} & \textbf{0.43$_{\pm 0.27}$} & \textbf{0.43$_{\pm 0.58}$} \\
CarGoal2 & \gray{0.48$_{\pm 0.01}$} & \gray{1.25$_{\pm 0.14}$} & \textbf{0.16$_{\pm 0.21}$} & \textbf{0.71$_{\pm 1.91}$} & \gray{0.90$_{\pm 0.12}$} & \gray{5.53$_{\pm 5.01}$} & \textbf{0.16$_{\pm 0.21}$} & \textbf{0.41$_{\pm 0.77}$} & \textbf{\blue{0.25$_{\pm 0.23}$}} & \textbf{\blue{0.89$_{\pm 1.23}$}} & \textbf{0.19$_{\pm 0.20}$} & \textbf{0.47$_{\pm 0.92}$} \\
CarPush1 & \textbf{\blue{0.31$_{\pm 0.01}$}} & \textbf{\blue{0.40$_{\pm 0.10}$}} & \textbf{0.19$_{\pm 0.13}$} & \textbf{0.34$_{\pm 1.09}$} & \textbf{-0.12$_{\pm 0.70}$} & \textbf{0.59$_{\pm 0.97}$} & \textbf{0.18$_{\pm 0.15}$} & \textbf{0.41$_{\pm 1.29}$} & \textbf{0.20$_{\pm 0.12}$} & \textbf{0.25$_{\pm 0.50}$} & \textbf{0.19$_{\pm 0.14}$} & \textbf{0.24$_{\pm 0.74}$} \\
CarPush2 & \gray{0.19$_{\pm 0.01}$} & \gray{1.30$_{\pm 0.16}$} & \textbf{0.05$_{\pm 0.14}$} & \textbf{0.71$_{\pm 1.66}$} & \gray{0.12$_{\pm 0.31}$} & \gray{6.06$_{\pm 9.53}$} & \textbf{0.07$_{\pm 0.15}$} & \textbf{0.80$_{\pm 1.41}$} & \textbf{\blue{0.08$_{\pm 0.15}$}} & \textbf{\blue{0.53$_{\pm 0.75}$}} & \textbf{0.05$_{\pm 0.17}$} & \textbf{0.52$_{\pm 1.07}$} \\
SwimmerVelocity & \textbf{\blue{0.66$_{\pm 0.01}$}} & \textbf{\blue{0.96$_{\pm 0.08}$}} & \gray{0.44$_{\pm 0.22}$} & \gray{1.66$_{\pm 2.37}$} & \textbf{-0.08$_{\pm 0.02}$} & \textbf{0.00$_{\pm 0.01}$} & \gray{0.46$_{\pm 0.21}$} & \gray{1.28$_{\pm 2.60}$} & \textbf{0.38$_{\pm 0.20}$} & \textbf{0.93$_{\pm 0.33}$} & \textbf{0.49$_{\pm 0.18}$} & \textbf{0.99$_{\pm 0.37}$} \\
HopperVelocity & \textbf{0.63$_{\pm 0.06}$} & \textbf{0.61$_{\pm 0.08}$} & \textbf{0.44$_{\pm 0.26}$} & \textbf{0.75$_{\pm 0.83}$} & \textbf{0.46$_{\pm 0.38}$} & \textbf{0.50$_{\pm 0.55}$} & \gray{0.21$_{\pm 0.04}$} & \gray{1.25$_{\pm 1.20}$} & \textbf{0.32$_{\pm 0.22}$} & \textbf{0.51$_{\pm 0.76}$} & \textbf{\blue{0.66$_{\pm 0.23}$}} & \textbf{\blue{0.80$_{\pm 0.35}$}} \\
HalfCheetahVelocity & \textbf{\blue{1.00$_{\pm 0.01}$}} & \textbf{\blue{0.01$_{\pm 0.01}$}} & \textbf{0.94$_{\pm 0.04}$} & \textbf{0.77$_{\pm 0.13}$} & \textbf{0.94$_{\pm 0.04}$} & \textbf{0.94$_{\pm 0.07}$} & \textbf{0.97$_{\pm 0.02}$} & \textbf{1.00$_{\pm 1.04}$} & \textbf{0.99$_{\pm 0.01}$} & \textbf{0.98$_{\pm 0.03}$} & \textbf{0.93$_{\pm 0.07}$} & \textbf{0.62$_{\pm 0.27}$} \\
Walker2dVelocity & \textbf{0.78$_{\pm 0.09}$} & \textbf{0.06$_{\pm 0.34}$} & \textbf{\blue{0.80$_{\pm 0.04}$}} & \textbf{\blue{0.70$_{\pm 0.51}$}} & \gray{0.17$_{\pm 0.28}$} & \gray{1.08$_{\pm 1.90}$} & \textbf{0.79$_{\pm 0.01}$} & \textbf{0.17$_{\pm 0.24}$} & \textbf{0.77$_{\pm 0.02}$} & \textbf{0.08$_{\pm 0.23}$} & \textbf{0.79$_{\pm 0.06}$} & \textbf{0.14$_{\pm 0.35}$} \\
AntVelocity & \textbf{\blue{0.98$_{\pm 0.00}$}} & \textbf{\blue{0.39$_{\pm 0.12}$}} & \textbf{0.95$_{\pm 0.04}$} & \textbf{0.64$_{\pm 0.21}$} & \textbf{0.51$_{\pm 0.41}$} & \textbf{0.59$_{\pm 0.34}$} & \textbf{0.97$_{\pm 0.05}$} & \textbf{0.14$_{\pm 0.15}$} & \textbf{\blue{0.98$_{\pm 0.05}$}} & \textbf{\blue{0.42$_{\pm 0.24}$}} & \textbf{0.97$_{\pm 0.03}$} & \textbf{0.34$_{\pm 0.14}$} \\
\hline
SafetyGym Average & \gray{0.54$_{\pm 0.21}$} & \gray{1.06$_{\pm 0.59}$} & \textbf{0.36$_{\pm 0.34}$} & \textbf{0.76$_{\pm 1.43}$} & \gray{0.41$_{\pm 0.46}$} & \gray{3.18$_{\pm 5.24}$} & \gray{0.34$_{\pm 0.19}$} & \gray{1.10$_{\pm 1.75}$} & \textbf{0.36$_{\pm 0.16}$} & \textbf{0.65$_{\pm 0.73}$} & \textbf{\blue{0.38$_{\pm 0.17}$}} & \textbf{\blue{0.57$_{\pm 0.68}$}} \\
\hline
\hline
BallRun & \gray{0.39$_{\pm 0.09}$} & \gray{1.16$_{\pm 0.19}$} & \textbf{0.18$_{\pm 0.09}$} & \textbf{0.94$_{\pm 0.73}$} & \textbf{\blue{0.40$_{\pm 0.12}$}} & \textbf{\blue{0.84$_{\pm 0.31}$}} & \textbf{0.15$_{\pm 0.01}$} & \textbf{0.00$_{\pm 0.00}$} & \textbf{0.22$_{\pm 0.13}$} & \textbf{0.87$_{\pm 1.35}$} & \textbf{0.20$_{\pm 0.02}$} & \textbf{0.06$_{\pm 0.10}$} \\
CarRun & \textbf{\blue{0.99$_{\pm 0.01}$}} & \textbf{\blue{0.65$_{\pm 0.31}$}} & \textbf{0.97$_{\pm 0.01}$} & \textbf{0.23$_{\pm 0.25}$} & \textbf{0.97$_{\pm 0.04}$} & \textbf{0.68$_{\pm 0.22}$} & \textbf{0.98$_{\pm 0.01}$} & \textbf{0.84$_{\pm 1.04}$} & \textbf{\blue{0.99$_{\pm 0.01}$}} & \textbf{\blue{0.34$_{\pm 0.29}$}} & \textbf{0.98$_{\pm 0.01}$} & \textbf{0.98$_{\pm 0.72}$} \\
DroneRun & \textbf{\blue{0.63$_{\pm 0.04}$}} & \textbf{\blue{0.79$_{\pm 0.68}$}} & \gray{0.45$_{\pm 0.10}$} & \gray{2.91$_{\pm 4.46}$} & \gray{0.46$_{\pm 0.18}$} & \gray{4.36$_{\pm 4.32}$} & \textbf{0.59$_{\pm 0.06}$} & \textbf{0.93$_{\pm 1.51}$} & \textbf{0.55$_{\pm 0.04}$} & \textbf{0.86$_{\pm 1.86}$} & \textbf{0.58$_{\pm 0.03}$} & \textbf{0.54$_{\pm 0.91}$} \\
AntRun & \textbf{\blue{0.72$_{\pm 0.04}$}} & \textbf{\blue{0.91$_{\pm 0.42}$}} & \textbf{0.61$_{\pm 0.13}$} & \textbf{0.78$_{\pm 0.59}$} & \textbf{0.12$_{\pm 0.09}$} & \textbf{0.04$_{\pm 0.08}$} & \gray{0.66$_{\pm 0.13}$} & \gray{1.43$_{\pm 1.25}$} & \textbf{0.57$_{\pm 0.14}$} & \textbf{0.61$_{\pm 0.34}$} & \textbf{0.68$_{\pm 0.11}$} & \textbf{0.95$_{\pm 0.29}$} \\
BallCircle & \gray{0.77$_{\pm 0.06}$} & \gray{1.07$_{\pm 0.27}$} & \textbf{0.69$_{\pm 0.10}$} & \textbf{0.59$_{\pm 0.23}$} & \textbf{\blue{0.78$_{\pm 0.05}$}} & \textbf{\blue{0.46$_{\pm 0.31}$}} & \gray{0.62$_{\pm 0.28}$} & \gray{1.17$_{\pm 1.95}$} & \textbf{0.59$_{\pm 0.12}$} & \textbf{0.86$_{\pm 0.15}$} & \textbf{0.69$_{\pm 0.11}$} & \textbf{0.97$_{\pm 0.27}$} \\
CarCircle & \textbf{\blue{0.75$_{\pm 0.06}$}} & \textbf{\blue{0.95$_{\pm 0.61}$}} & \textbf{0.70$_{\pm 0.10}$} & \textbf{0.66$_{\pm 0.27}$} & \textbf{0.72$_{\pm 0.04}$} & \textbf{0.77$_{\pm 0.44}$} & \gray{0.78$_{\pm 0.12}$} & \gray{1.88$_{\pm 3.23}$} & \textbf{0.42$_{\pm 0.19}$} & \textbf{0.88$_{\pm 0.33}$} & \textbf{0.53$_{\pm 0.17}$} & \textbf{0.51$_{\pm 0.51}$} \\
DroneCircle & \textbf{\blue{0.63$_{\pm 0.07}$}} & \textbf{\blue{0.98$_{\pm 0.27}$}} & \textbf{0.55$_{\pm 0.06}$} & \textbf{0.65$_{\pm 0.29}$} & \textbf{0.29$_{\pm 0.27}$} & \textbf{0.63$_{\pm 0.70}$} & \gray{0.60$_{\pm 0.02}$} & \gray{1.31$_{\pm 0.86}$} & \textbf{0.41$_{\pm 0.10}$} & \textbf{0.75$_{\pm 0.31}$} & \textbf{0.42$_{\pm 0.12}$} & \textbf{0.52$_{\pm 0.32}$} \\
AntCircle & \gray{0.54$_{\pm 0.20}$} & \gray{1.78$_{\pm 4.33}$} & \textbf{0.37$_{\pm 0.18}$} & \textbf{0.15$_{\pm 0.25}$} & \textbf{\blue{0.61$_{\pm 0.14}$}} & \textbf{\blue{0.75$_{\pm 0.90}$}} & \textbf{0.44$_{\pm 0.14}$} & \textbf{0.53$_{\pm 0.93}$} & \textbf{0.40$_{\pm 0.16}$} & \textbf{0.59$_{\pm 1.33}$} & \textbf{0.56$_{\pm 0.17}$} & \textbf{0.79$_{\pm 0.60}$} \\
\hline
BulletGym Average & \gray{0.68$_{\pm 0.19}$} & \gray{1.04$_{\pm 1.65}$} & \textbf{0.57$_{\pm 0.25}$} & \textbf{0.86$_{\pm 1.82}$} & \gray{0.54$_{\pm 0.30}$} & \gray{1.07$_{\pm 2.04}$} & \gray{0.60$_{\pm 0.10}$} & \gray{1.01$_{\pm 1.35}$} & \textbf{0.52$_{\pm 0.11}$} & \textbf{0.72$_{\pm 0.75}$} & \textbf{\blue{0.58$_{\pm 0.09}$}} & \textbf{\blue{0.67$_{\pm 0.46}$}} \\
\hline
\hline
easysparse & \textbf{0.17$_{\pm 0.14}$} & \textbf{0.23$_{\pm 0.32}$} & \textbf{0.12$_{\pm 0.20}$} & \textbf{0.37$_{\pm 0.43}$} & \textbf{-0.06$_{\pm 0.00}$} & \textbf{0.10$_{\pm 0.05}$} & \gray{0.79$_{\pm 0.14}$} & \gray{1.34$_{\pm 1.43}$} & \textbf{\blue{0.70$_{\pm 0.19}$}} & \textbf{\blue{0.94$_{\pm 0.17}$}} & \textbf{0.38$_{\pm 0.13}$} & \textbf{0.41$_{\pm 0.65}$} \\
easymean & \textbf{0.45$_{\pm 0.11}$} & \textbf{0.54$_{\pm 0.55}$} & \textbf{0.02$_{\pm 0.07}$} & \textbf{0.21$_{\pm 0.23}$} & \textbf{-0.06$_{\pm 0.01}$} & \textbf{0.07$_{\pm 0.08}$} & \gray{0.82$_{\pm 0.06}$} & \gray{1.33$_{\pm 0.86}$} & \textbf{\blue{0.73$_{\pm 0.17}$}} & \textbf{\blue{0.94$_{\pm 0.13}$}} & \textbf{0.39$_{\pm 0.26}$} & \textbf{0.95$_{\pm 1.91}$} \\
easydense & \textbf{0.32$_{\pm 0.18}$} & \textbf{0.62$_{\pm 0.43}$} & \textbf{0.11$_{\pm 0.15}$} & \textbf{0.16$_{\pm 0.17}$} & \textbf{-0.06$_{\pm 0.00}$} & \textbf{0.07$_{\pm 0.04}$} & \gray{0.84$_{\pm 0.10}$} & \gray{1.88$_{\pm 1.20}$} & \textbf{\blue{0.70$_{\pm 0.17}$}} & \textbf{\blue{0.90$_{\pm 0.16}$}} & \textbf{0.34$_{\pm 0.21}$} & \textbf{0.51$_{\pm 1.24}$} \\
mediumsparse & \gray{0.87$_{\pm 0.11}$} & \gray{1.10$_{\pm 0.26}$} & \textbf{0.59$_{\pm 0.36}$} & \textbf{0.74$_{\pm 0.97}$} & \textbf{-0.08$_{\pm 0.00}$} & \textbf{0.07$_{\pm 0.04}$} & \gray{0.97$_{\pm 0.02}$} & \gray{1.24$_{\pm 0.68}$} & \textbf{\blue{0.94$_{\pm 0.07}$}} & \textbf{\blue{0.82$_{\pm 0.32}$}} & \textbf{0.65$_{\pm 0.15}$} & \textbf{0.14$_{\pm 0.16}$} \\
mediummean & \textbf{0.45$_{\pm 0.39}$} & \textbf{0.75$_{\pm 0.83}$} & \textbf{0.66$_{\pm 0.35}$} & \textbf{0.90$_{\pm 0.89}$} & \textbf{-0.07$_{\pm 0.01}$} & \textbf{0.02$_{\pm 0.03}$} & \textbf{\blue{0.93$_{\pm 0.12}$}} & \textbf{\blue{0.78$_{\pm 0.51}$}} & \textbf{0.92$_{\pm 0.11}$} & \textbf{0.77$_{\pm 0.16}$} & \textbf{0.67$_{\pm 0.19}$} & \textbf{0.32$_{\pm 0.37}$} \\
mediumdense & \gray{0.88$_{\pm 0.12}$} & \gray{2.41$_{\pm 0.71}$} & \textbf{0.75$_{\pm 0.29}$} & \textbf{0.65$_{\pm 0.57}$} & \textbf{-0.08$_{\pm 0.00}$} & \textbf{0.06$_{\pm 0.03}$} & \gray{0.82$_{\pm 0.30}$} & \gray{1.07$_{\pm 0.71}$} & \textbf{\blue{0.78$_{\pm 0.28}$}} & \textbf{\blue{0.68$_{\pm 0.31}$}} & \textbf{\blue{0.78$_{\pm 0.19}$}} & \textbf{\blue{0.36$_{\pm 0.28}$}} \\
hardsparse & \textbf{0.25$_{\pm 0.08}$} & \textbf{0.41$_{\pm 0.33}$} & \textbf{0.45$_{\pm 0.15}$} & \textbf{0.75$_{\pm 0.56}$} & \textbf{-0.05$_{\pm 0.00}$} & \textbf{0.07$_{\pm 0.04}$} & \textbf{\blue{0.54$_{\pm 0.09}$}} & \textbf{\blue{0.91$_{\pm 0.51}$}} & \textbf{0.49$_{\pm 0.13}$} & \textbf{0.76$_{\pm 0.44}$} & \textbf{0.40$_{\pm 0.17}$} & \textbf{0.83$_{\pm 0.89}$} \\
hardmean & \textbf{0.33$_{\pm 0.21}$} & \textbf{0.97$_{\pm 0.31}$} & \textbf{0.29$_{\pm 0.13}$} & \textbf{0.28$_{\pm 0.31}$} & \textbf{-0.05$_{\pm 0.00}$} & \textbf{0.06$_{\pm 0.03}$} & \gray{0.48$_{\pm 0.14}$} & \gray{1.16$_{\pm 1.19}$} & \textbf{\blue{0.46$_{\pm 0.15}$}} & \textbf{\blue{0.84$_{\pm 0.54}$}} & \textbf{0.38$_{\pm 0.16}$} & \textbf{0.81$_{\pm 0.94}$} \\
harddense & \textbf{0.08$_{\pm 0.15}$} & \textbf{0.21$_{\pm 0.42}$} & \textbf{0.36$_{\pm 0.18}$} & \textbf{0.66$_{\pm 0.92}$} & \textbf{-0.03$_{\pm 0.01}$} & \textbf{0.11$_{\pm 0.08}$} & \gray{0.47$_{\pm 0.19}$} & \gray{1.43$_{\pm 0.93}$} & \textbf{\blue{0.44$_{\pm 0.17}$}} & \textbf{\blue{0.63$_{\pm 0.24}$}} & \textbf{0.29$_{\pm 0.16}$} & \textbf{0.38$_{\pm 0.69}$} \\
\hline
MetaDrive Average & \textbf{0.42$_{\pm 0.31}$} & \textbf{0.80$_{\pm 0.61}$} & \textbf{0.38$_{\pm 0.34}$} & \textbf{0.54$_{\pm 0.69}$} & \textbf{-0.06$_{\pm 0.02}$} & \textbf{0.07$_{\pm 0.06}$} & \gray{0.74$_{\pm 0.13}$} & \gray{1.24$_{\pm 0.89}$} & \textbf{\blue{0.69$_{\pm 0.16}$}} & \textbf{\blue{0.81$_{\pm 0.27}$}} & \textbf{0.48$_{\pm 0.18}$} & \textbf{0.52$_{\pm 0.79}$} \\
\hline
\hline
\end{tabular}

}
    \caption{
\normalfont
\textbf{Results for normalized cost and normalized reward:} \(\uparrow\) indicates that a higher value is better, while \(\downarrow\) indicates that a lower value is preferable.\textbf{Bold} text denotes safe policies, \gray{gray} indicates unsafe policies, and \blue{\textbf{blue}} highlights the highest reward among the safe policies for each task.  
 Each result is obtained by running three random seeds across three distinct cost thresholds and evaluating over 20 episodes. }
    \label{tab:main-table-two-col}
\end{table*}

\subsection{Comparison with Non-Adaptive Baselines}
In the main results (Main Paper Table~\ref{tab:main-table}), we compared our method only with approaches that are capable of generalizing across different cost budgets. In Table~\ref{tab:main-table-other-baselines}, we compare our results with algorithms that require separate training for each cost budget. We use the same cost budgets as those in the DSRL benchmark, as shown in Table~\ref{tab:hyperparameters}. The results indicate that our approach outperforms baseline methods, even without retraining for each individual budget.
\begin{table*}[tb]
    \centering
    
\resizebox{\textwidth}{!}{
\begin{tabular}{lllllllll|ll}
\hline
 & \multicolumn{2}{c}{BC-Safe} & \multicolumn{2}{c}{BCQ-Lag} & \multicolumn{2}{c}{COptiDICE} & \multicolumn{2}{c}{CPQ} & \multicolumn{2}{c}{\our\ (Ours)} \\
 & reward $\uparrow$ & cost $\downarrow$ & reward $\uparrow$ & cost $\downarrow$ & reward $\uparrow$ & cost $\downarrow$ & reward $\uparrow$ & cost $\downarrow$ & reward $\uparrow$ & cost $\downarrow$ \\
\hline
PointButton1 & \textbf{0.06$_{\pm 0.04}$} & \textbf{0.52$_{\pm 0.21}$} & \gray{0.24$_{\pm 0.04}$} & \gray{1.73$_{\pm 1.11}$} & \gray{0.13$_{\pm 0.02}$} & \gray{1.35$_{\pm 0.91}$} & \gray{0.69$_{\pm 0.05}$} & \gray{3.20$_{\pm 1.57}$} & \textbf{\blue{0.07$_{\pm 0.18}$}} & \textbf{\blue{0.79$_{\pm 2.10}$}} \\
PointButton2 & \gray{0.16$_{\pm 0.04}$} & \gray{1.10$_{\pm 0.84}$} & \gray{0.40$_{\pm 0.03}$} & \gray{2.66$_{\pm 1.47}$} & \gray{0.15$_{\pm 0.03}$} & \gray{1.51$_{\pm 0.96}$} & \gray{0.58$_{\pm 0.07}$} & \gray{4.30$_{\pm 2.35}$} & \textbf{\blue{0.12$_{\pm 0.22}$}} & \textbf{\blue{0.81$_{\pm 1.36}$}} \\
PointCircle1 & \textbf{0.41$_{\pm 0.08}$} & \textbf{0.16$_{\pm 0.11}$} & \gray{0.54$_{\pm 0.17}$} & \gray{2.38$_{\pm 1.30}$} & \gray{0.86$_{\pm 0.01}$} & \gray{5.51$_{\pm 2.93}$} & \textbf{0.43$_{\pm 0.07}$} & \textbf{0.75$_{\pm 1.86}$} & \textbf{\blue{0.55$_{\pm 0.17}$}} & \textbf{\blue{0.86$_{\pm 0.21}$}} \\
PointCircle2 & \textbf{0.48$_{\pm 0.08}$} & \textbf{0.99$_{\pm 0.35}$} & \gray{0.66$_{\pm 0.13}$} & \gray{2.60$_{\pm 0.71}$} & \gray{0.85$_{\pm 0.01}$} & \gray{8.61$_{\pm 4.62}$} & \gray{0.24$_{\pm 0.40}$} & \gray{3.58$_{\pm 3.09}$} & \textbf{\blue{0.58$_{\pm 0.10}$}} & \textbf{\blue{0.74$_{\pm 0.34}$}} \\
PointGoal1 & \textbf{0.43$_{\pm 0.12}$} & \textbf{0.54$_{\pm 0.24}$} & \textbf{\blue{0.71$_{\pm 0.02}$}} & \textbf{\blue{0.98$_{\pm 0.46}$}} & \gray{0.49$_{\pm 0.05}$} & \gray{1.66$_{\pm 1.05}$} & \textbf{0.57$_{\pm 0.08}$} & \textbf{0.35$_{\pm 0.37}$} & \textbf{0.56$_{\pm 0.25}$} & \textbf{0.70$_{\pm 0.65}$} \\
PointGoal2 & \textbf{0.29$_{\pm 0.09}$} & \textbf{0.78$_{\pm 0.27}$} & \gray{0.67$_{\pm 0.06}$} & \gray{3.18$_{\pm 1.79}$} & \gray{0.38$_{\pm 0.03}$} & \gray{1.92$_{\pm 1.15}$} & \gray{0.40$_{\pm 0.15}$} & \gray{1.31$_{\pm 0.71}$} & \textbf{\blue{0.38$_{\pm 0.21}$}} & \textbf{\blue{0.77$_{\pm 0.68}$}} \\
PointPush1 & \textbf{0.13$_{\pm 0.05}$} & \textbf{0.43$_{\pm 0.29}$} & \textbf{\blue{0.33$_{\pm 0.04}$}} & \textbf{\blue{0.86$_{\pm 0.45}$}} & \textbf{0.13$_{\pm 0.02}$} & \textbf{0.83$_{\pm 0.52}$} & \textbf{0.20$_{\pm 0.08}$} & \textbf{0.83$_{\pm 0.44}$} & \textbf{0.14$_{\pm 0.16}$} & \textbf{0.47$_{\pm 0.79}$} \\
PointPush2 & \textbf{0.11$_{\pm 0.04}$} & \textbf{0.80$_{\pm 0.59}$} & \textbf{\blue{0.23$_{\pm 0.03}$}} & \textbf{\blue{0.99$_{\pm 0.57}$}} & \gray{0.02$_{\pm 0.07}$} & \gray{1.18$_{\pm 0.74}$} & \gray{0.11$_{\pm 0.14}$} & \gray{1.04$_{\pm 0.61}$} & \textbf{0.14$_{\pm 0.18}$} & \textbf{0.71$_{\pm 1.11}$} \\
CarButton1 & \textbf{\blue{0.07$_{\pm 0.03}$}} & \textbf{\blue{0.85$_{\pm 0.39}$}} & \gray{0.04$_{\pm 0.05}$} & \gray{1.63$_{\pm 0.59}$} & \gray{-0.08$_{\pm 0.09}$} & \gray{1.68$_{\pm 1.29}$} & \gray{0.42$_{\pm 0.05}$} & \gray{9.66$_{\pm 5.71}$} & \textbf{-0.06$_{\pm 0.18}$} & \textbf{0.34$_{\pm 0.76}$} \\
CarButton2 & \textbf{\blue{-0.01$_{\pm 0.02}$}} & \textbf{\blue{0.63$_{\pm 0.30}$}} & \gray{0.06$_{\pm 0.05}$} & \gray{2.13$_{\pm 1.19}$} & \gray{-0.07$_{\pm 0.06}$} & \gray{1.59$_{\pm 1.10}$} & \gray{0.37$_{\pm 0.11}$} & \gray{12.51$_{\pm 8.54}$} & \textbf{-0.20$_{\pm 0.26}$} & \textbf{0.27$_{\pm 0.62}$} \\
CarCircle1 & \gray{0.37$_{\pm 0.10}$} & \gray{1.38$_{\pm 0.44}$} & \gray{0.73$_{\pm 0.02}$} & \gray{5.25$_{\pm 2.76}$} & \gray{0.70$_{\pm 0.02}$} & \gray{5.72$_{\pm 3.04}$} & \gray{0.02$_{\pm 0.15}$} & \gray{2.29$_{\pm 2.13}$} & \textbf{\blue{0.50$_{\pm 0.12}$}} & \textbf{\blue{0.48$_{\pm 0.43}$}} \\
CarCircle2 & \gray{0.54$_{\pm 0.08}$} & \gray{3.38$_{\pm 1.30}$} & \gray{0.72$_{\pm 0.04}$} & \gray{6.58$_{\pm 3.02}$} & \gray{0.77$_{\pm 0.03}$} & \gray{7.99$_{\pm 4.23}$} & \gray{0.44$_{\pm 0.10}$} & \gray{2.69$_{\pm 2.61}$} & \textbf{\blue{0.48$_{\pm 0.15}$}} & \textbf{\blue{0.44$_{\pm 0.48}$}} \\
CarGoal1 & \textbf{0.24$_{\pm 0.08}$} & \textbf{0.28$_{\pm 0.11}$} & \textbf{\blue{0.47$_{\pm 0.05}$}} & \textbf{\blue{0.78$_{\pm 0.50}$}} & \textbf{0.35$_{\pm 0.06}$} & \textbf{0.54$_{\pm 0.33}$} & \gray{0.79$_{\pm 0.07}$} & \gray{1.42$_{\pm 0.81}$} & \textbf{0.43$_{\pm 0.27}$} & \textbf{0.43$_{\pm 0.58}$} \\
CarGoal2 & \textbf{0.14$_{\pm 0.05}$} & \textbf{0.51$_{\pm 0.26}$} & \gray{0.30$_{\pm 0.05}$} & \gray{1.44$_{\pm 0.99}$} & \textbf{\blue{0.25$_{\pm 0.04}$}} & \textbf{\blue{0.91$_{\pm 0.41}$}} & \gray{0.65$_{\pm 0.20}$} & \gray{3.75$_{\pm 2.00}$} & \textbf{0.19$_{\pm 0.20}$} & \textbf{0.47$_{\pm 0.92}$} \\
CarPush1 & \textbf{0.14$_{\pm 0.03}$} & \textbf{0.33$_{\pm 0.23}$} & \textbf{\blue{0.23$_{\pm 0.03}$}} & \textbf{\blue{0.43$_{\pm 0.19}$}} & \textbf{\blue{0.23$_{\pm 0.04}$}} & \textbf{\blue{0.50$_{\pm 0.40}$}} & \textbf{-0.03$_{\pm 0.24}$} & \textbf{0.95$_{\pm 0.53}$} & \textbf{0.19$_{\pm 0.14}$} & \textbf{0.24$_{\pm 0.74}$} \\
CarPush2 & \textbf{\blue{0.05$_{\pm 0.02}$}} & \textbf{\blue{0.45$_{\pm 0.19}$}} & \gray{0.15$_{\pm 0.02}$} & \gray{1.38$_{\pm 0.68}$} & \gray{0.09$_{\pm 0.02}$} & \gray{1.07$_{\pm 0.69}$} & \gray{0.24$_{\pm 0.06}$} & \gray{4.25$_{\pm 2.44}$} & \textbf{\blue{0.05$_{\pm 0.17}$}} & \textbf{\blue{0.52$_{\pm 1.07}$}} \\
SwimmerVelocity & \gray{0.51$_{\pm 0.20}$} & \gray{1.07$_{\pm 0.07}$} & \gray{0.48$_{\pm 0.33}$} & \gray{6.58$_{\pm 3.95}$} & \gray{0.63$_{\pm 0.06}$} & \gray{7.58$_{\pm 1.77}$} & \gray{0.13$_{\pm 0.06}$} & \gray{2.66$_{\pm 0.96}$} & \textbf{\blue{0.49$_{\pm 0.18}$}} & \textbf{\blue{0.99$_{\pm 0.37}$}} \\
HopperVelocity & \textbf{0.36$_{\pm 0.13}$} & \textbf{0.67$_{\pm 0.27}$} & \gray{0.78$_{\pm 0.09}$} & \gray{5.02$_{\pm 3.43}$} & \gray{0.13$_{\pm 0.06}$} & \gray{1.51$_{\pm 1.54}$} & \gray{0.14$_{\pm 0.09}$} & \gray{2.11$_{\pm 2.29}$} & \textbf{\blue{0.66$_{\pm 0.23}$}} & \textbf{\blue{0.80$_{\pm 0.35}$}} \\
HalfCheetahVelocity & \textbf{0.88$_{\pm 0.03}$} & \textbf{0.54$_{\pm 0.63}$} & \gray{1.05$_{\pm 0.07}$} & \gray{18.21$_{\pm 8.29}$} & \textbf{0.65$_{\pm 0.01}$} & \textbf{0.00$_{\pm 0.00}$} & \textbf{0.29$_{\pm 0.14}$} & \textbf{0.74$_{\pm 0.19}$} & \textbf{\blue{0.93$_{\pm 0.07}$}} & \textbf{\blue{0.62$_{\pm 0.27}$}} \\
Walker2dVelocity & \textbf{\blue{0.79$_{\pm 0.05}$}} & \textbf{\blue{0.04$_{\pm 0.32}$}} & \textbf{\blue{0.79$_{\pm 0.01}$}} & \textbf{\blue{0.17$_{\pm 0.06}$}} & \textbf{0.12$_{\pm 0.01}$} & \textbf{0.74$_{\pm 0.07}$} & \textbf{0.04$_{\pm 0.05}$} & \textbf{0.21$_{\pm 0.09}$} & \textbf{\blue{0.79$_{\pm 0.06}$}} & \textbf{\blue{0.14$_{\pm 0.35}$}} \\
AntVelocity & \textbf{\blue{0.98$_{\pm 0.01}$}} & \textbf{\blue{0.29$_{\pm 0.10}$}} & \gray{1.02$_{\pm 0.01}$} & \gray{4.15$_{\pm 1.63}$} & \gray{1.00$_{\pm 0.00}$} & \gray{3.28$_{\pm 2.01}$} & \textbf{-1.01$_{\pm 0.00}$} & \textbf{0.00$_{\pm 0.00}$} & \textbf{0.97$_{\pm 0.03}$} & \textbf{0.34$_{\pm 0.14}$} \\
\hline
SafetyGym Average & \textbf{0.34$_{\pm 0.31}$} & \textbf{0.75$_{\pm 0.80}$} & \gray{0.50$_{\pm 0.32}$} & \gray{3.29$_{\pm 4.97}$} & \gray{0.37$_{\pm 0.32}$} & \gray{2.65$_{\pm 3.24}$} & \gray{0.27$_{\pm 0.37}$} & \gray{2.79$_{\pm 3.86}$} & \textbf{\blue{0.38$_{\pm 0.17}$}} & \textbf{\blue{0.57$_{\pm 0.68}$}} \\
\hline
\hline
BallRun & \gray{0.27$_{\pm 0.14}$} & \gray{1.46$_{\pm 0.39}$} & \gray{0.76$_{\pm 0.01}$} & \gray{3.91$_{\pm 0.35}$} & \gray{0.59$_{\pm 0.00}$} & \gray{3.52$_{\pm 0.00}$} & \gray{0.22$_{\pm 0.00}$} & \gray{1.27$_{\pm 0.12}$} & \textbf{\blue{0.20$_{\pm 0.02}$}} & \textbf{\blue{0.06$_{\pm 0.10}$}} \\
CarRun & \textbf{0.94$_{\pm 0.00}$} & \textbf{0.22$_{\pm 0.02}$} & \textbf{0.94$_{\pm 0.01}$} & \textbf{0.15$_{\pm 0.91}$} & \textbf{0.87$_{\pm 0.00}$} & \textbf{0.00$_{\pm 0.00}$} & \gray{0.95$_{\pm 0.01}$} & \gray{1.79$_{\pm 0.18}$} & \textbf{\blue{0.98$_{\pm 0.01}$}} & \textbf{\blue{0.98$_{\pm 0.72}$}} \\
DroneRun & \textbf{0.28$_{\pm 0.25}$} & \textbf{0.74$_{\pm 0.97}$} & \gray{0.72$_{\pm 0.12}$} & \gray{5.54$_{\pm 0.81}$} & \gray{0.67$_{\pm 0.02}$} & \gray{4.15$_{\pm 0.10}$} & \gray{0.33$_{\pm 0.10}$} & \gray{3.52$_{\pm 0.58}$} & \textbf{\blue{0.58$_{\pm 0.03}$}} & \textbf{\blue{0.54$_{\pm 0.91}$}} \\
AntRun & \gray{0.65$_{\pm 0.15}$} & \gray{1.09$_{\pm 0.84}$} & \gray{0.76$_{\pm 0.07}$} & \gray{5.11$_{\pm 2.39}$} & \textbf{0.61$_{\pm 0.01}$} & \textbf{0.94$_{\pm 0.69}$} & \textbf{0.03$_{\pm 0.02}$} & \textbf{0.02$_{\pm 0.09}$} & \textbf{\blue{0.68$_{\pm 0.11}$}} & \textbf{\blue{0.95$_{\pm 0.29}$}} \\
BallCircle & \textbf{0.52$_{\pm 0.08}$} & \textbf{0.65$_{\pm 0.17}$} & \gray{0.69$_{\pm 0.11}$} & \gray{2.36$_{\pm 1.04}$} & \gray{0.70$_{\pm 0.04}$} & \gray{2.61$_{\pm 0.79}$} & \textbf{0.64$_{\pm 0.01}$} & \textbf{0.76$_{\pm 0.00}$} & \textbf{\blue{0.69$_{\pm 0.11}$}} & \textbf{\blue{0.97$_{\pm 0.27}$}} \\
CarCircle & \textbf{0.50$_{\pm 0.22}$} & \textbf{0.84$_{\pm 0.67}$} & \gray{0.63$_{\pm 0.19}$} & \gray{1.89$_{\pm 1.37}$} & \gray{0.49$_{\pm 0.05}$} & \gray{3.14$_{\pm 2.98}$} & \textbf{\blue{0.71$_{\pm 0.02}$}} & \textbf{\blue{0.33$_{\pm 0.00}$}} & \textbf{0.53$_{\pm 0.17}$} & \textbf{0.51$_{\pm 0.51}$} \\
DroneCircle & \textbf{\blue{0.56$_{\pm 0.18}$}} & \textbf{\blue{0.57$_{\pm 0.27}$}} & \gray{0.80$_{\pm 0.07}$} & \gray{3.07$_{\pm 0.89}$} & \gray{0.26$_{\pm 0.03}$} & \gray{1.02$_{\pm 0.46}$} & \gray{-0.22$_{\pm 0.05}$} & \gray{1.28$_{\pm 0.97}$} & \textbf{0.42$_{\pm 0.12}$} & \textbf{0.52$_{\pm 0.32}$} \\
AntCircle & \textbf{0.40$_{\pm 0.16}$} & \textbf{0.96$_{\pm 2.67}$} & \gray{0.58$_{\pm 0.25}$} & \gray{2.87$_{\pm 3.08}$} & \gray{0.17$_{\pm 0.10}$} & \gray{5.04$_{\pm 6.70}$} & \gray{-} & \gray{-} & \textbf{\blue{0.56$_{\pm 0.17}$}} & \textbf{\blue{0.79$_{\pm 0.60}$}} \\
\hline
BulletGym Average & \textbf{0.52$_{\pm 0.27}$} & \textbf{0.82$_{\pm 1.27}$} & \gray{0.74$_{\pm 0.25}$} & \gray{3.11$_{\pm 3.55}$} & \gray{0.55$_{\pm 0.24}$} & \gray{2.55$_{\pm 3.62}$} & \gray{0.33$_{\pm 0.29}$} & \gray{1.12$_{\pm 1.85}$} & \textbf{\blue{0.58$_{\pm 0.09}$}} & \textbf{\blue{0.67$_{\pm 0.46}$}} \\
\hline
\hline
easysparse & \textbf{0.11$_{\pm 0.08}$} & \textbf{0.21$_{\pm 0.02}$} & \gray{0.78$_{\pm 0.00}$} & \gray{5.01$_{\pm 0.06}$} & \gray{0.96$_{\pm 0.02}$} & \gray{5.44$_{\pm 0.27}$} & \textbf{-0.06$_{\pm 0.00}$} & \textbf{0.07$_{\pm 0.02}$} & \textbf{\blue{0.70$_{\pm 0.19}$}} & \textbf{\blue{0.94$_{\pm 0.17}$}} \\
easymean & \textbf{0.04$_{\pm 0.03}$} & \textbf{0.29$_{\pm 0.02}$} & \gray{0.71$_{\pm 0.06}$} & \gray{3.44$_{\pm 0.35}$} & \gray{0.66$_{\pm 0.16}$} & \gray{3.97$_{\pm 1.47}$} & \textbf{-0.07$_{\pm 0.00}$} & \textbf{0.07$_{\pm 0.01}$} & \textbf{\blue{0.73$_{\pm 0.17}$}} & \textbf{\blue{0.94$_{\pm 0.13}$}} \\
easydense & \textbf{0.11$_{\pm 0.07}$} & \textbf{0.14$_{\pm 0.01}$} & \textbf{0.26$_{\pm 0.00}$} & \textbf{0.47$_{\pm 0.01}$} & \gray{0.50$_{\pm 0.10}$} & \gray{2.54$_{\pm 0.53}$} & \textbf{-0.06$_{\pm 0.00}$} & \textbf{0.03$_{\pm 0.01}$} & \textbf{\blue{0.70$_{\pm 0.17}$}} & \textbf{\blue{0.90$_{\pm 0.16}$}} \\
mediumsparse & \textbf{0.33$_{\pm 0.34}$} & \textbf{0.30$_{\pm 0.32}$} & \gray{0.44$_{\pm 0.00}$} & \gray{1.16$_{\pm 0.02}$} & \gray{0.71$_{\pm 0.37}$} & \gray{2.49$_{\pm 1.90}$} & \textbf{-0.08$_{\pm 0.02}$} & \textbf{0.07$_{\pm 0.03}$} & \textbf{\blue{0.94$_{\pm 0.07}$}} & \textbf{\blue{0.82$_{\pm 0.32}$}} \\
mediummean & \textbf{0.31$_{\pm 0.06}$} & \textbf{0.21$_{\pm 0.00}$} & \gray{0.78$_{\pm 0.12}$} & \gray{1.53$_{\pm 0.21}$} & \gray{0.76$_{\pm 0.34}$} & \gray{2.05$_{\pm 0.92}$} & \textbf{-0.08$_{\pm 0.00}$} & \textbf{0.05$_{\pm 0.02}$} & \textbf{\blue{0.92$_{\pm 0.11}$}} & \textbf{\blue{0.77$_{\pm 0.16}$}} \\
mediumdense & \textbf{0.24$_{\pm 0.00}$} & \textbf{0.17$_{\pm 0.00}$} & \gray{0.58$_{\pm 0.21}$} & \gray{1.89$_{\pm 1.19}$} & \gray{0.69$_{\pm 0.13}$} & \gray{2.24$_{\pm 0.65}$} & \textbf{-0.07$_{\pm 0.00}$} & \textbf{0.07$_{\pm 0.01}$} & \textbf{\blue{0.78$_{\pm 0.28}$}} & \textbf{\blue{0.68$_{\pm 0.31}$}} \\
hardsparse & \gray{0.17$_{\pm 0.05}$} & \gray{3.25$_{\pm 0.10}$} & \gray{0.50$_{\pm 0.04}$} & \gray{1.02$_{\pm 0.05}$} & \gray{0.37$_{\pm 0.10}$} & \gray{2.05$_{\pm 0.27}$} & \textbf{-0.05$_{\pm 0.00}$} & \textbf{0.06$_{\pm 0.01}$} & \textbf{\blue{0.49$_{\pm 0.13}$}} & \textbf{\blue{0.76$_{\pm 0.44}$}} \\
hardmean & \textbf{0.13$_{\pm 0.00}$} & \textbf{0.40$_{\pm 0.00}$} & \gray{0.47$_{\pm 0.13}$} & \gray{2.56$_{\pm 0.72}$} & \gray{0.32$_{\pm 0.19}$} & \gray{2.47$_{\pm 2.00}$} & \textbf{-0.05$_{\pm 0.00}$} & \textbf{0.06$_{\pm 0.02}$} & \textbf{\blue{0.46$_{\pm 0.15}$}} & \textbf{\blue{0.84$_{\pm 0.54}$}} \\
harddense & \textbf{0.15$_{\pm 0.06}$} & \textbf{0.22$_{\pm 0.01}$} & \gray{0.35$_{\pm 0.03}$} & \gray{1.40$_{\pm 0.14}$} & \gray{0.24$_{\pm 0.21}$} & \gray{1.68$_{\pm 2.15}$} & \textbf{-0.04$_{\pm 0.01}$} & \textbf{0.08$_{\pm 0.01}$} & \textbf{\blue{0.44$_{\pm 0.17}$}} & \textbf{\blue{0.63$_{\pm 0.24}$}} \\
\hline
MetaDrive Average & \textbf{0.18$_{\pm 0.27}$} & \textbf{0.58$_{\pm 0.35}$} & \gray{0.54$_{\pm 0.35}$} & \gray{2.05$_{\pm 2.70}$} & \gray{0.58$_{\pm 0.32}$} & \gray{2.77$_{\pm 2.87}$} & \textbf{-0.06$_{\pm 0.01}$} & \textbf{0.06$_{\pm 0.04}$} & \textbf{\blue{0.69$_{\pm 0.16}$}} & \textbf{\blue{0.81$_{\pm 0.27}$}} \\
\hline
\hline
\end{tabular}

}
\caption{
\normalfont
\textbf{Comparison with Non-Adaptive Baselines}: This table presents results similar to our main results table (Table~\ref{tab:main-table}), but includes additional baseline methods that do not generalize to varying cost budgets. The results for these baselines are taken from the DSRL benchmark~\cite{liu_datasets_2024}.
}
    \label{tab:main-table-other-baselines}
\end{table*}

\begin{figure}[t]
    \centering	
     \vspace{-0.2cm}
    \includegraphics[width=0.5\linewidth]{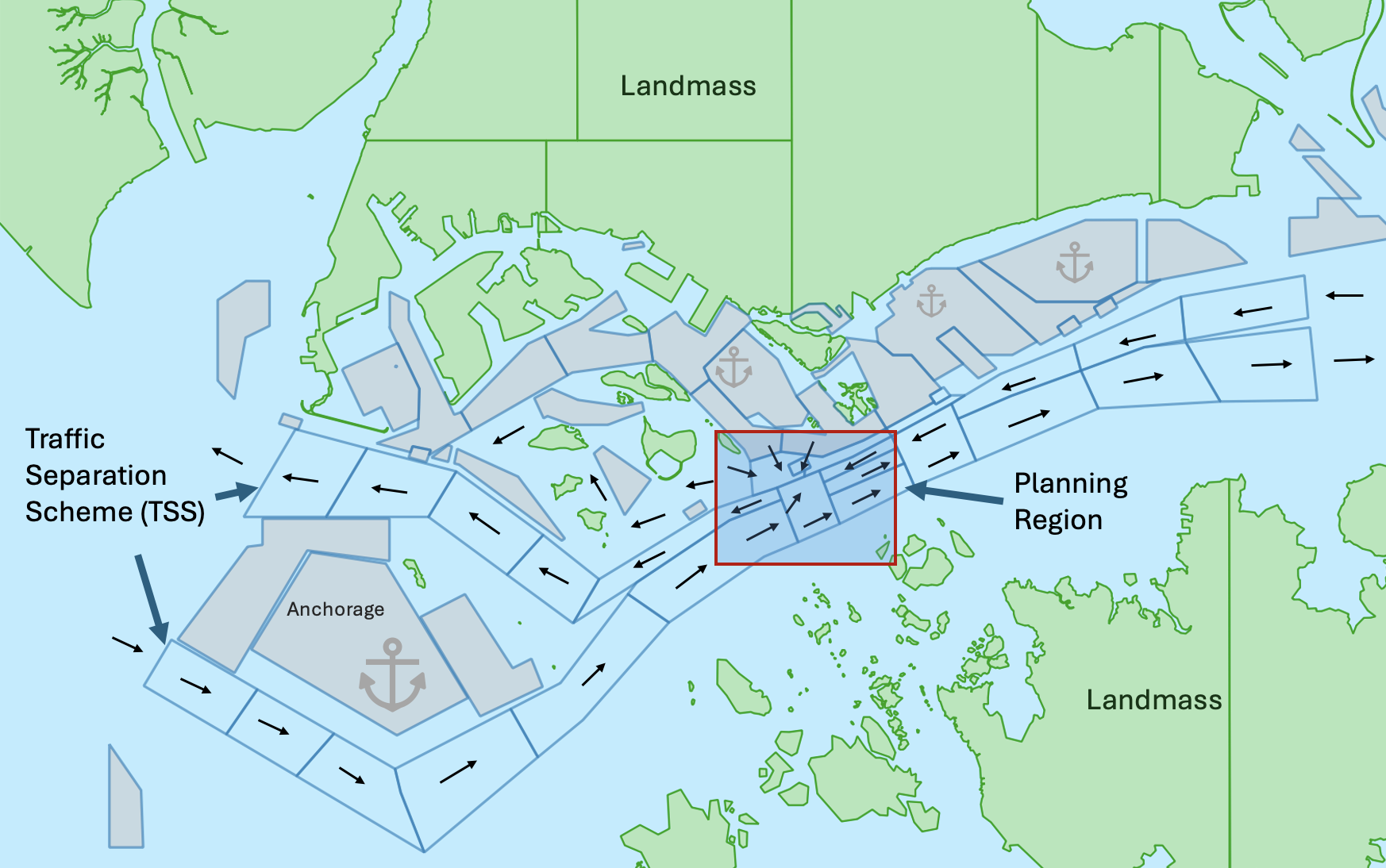}
    \caption{The red region is used as the environment area. The gray areas indicate anchorage zones, the green areas represent landmasses, and the arrows and regions with dark blue borders represent the Maritime Traffic Separation Scheme (TSS). The high density of crossing points in the red area makes it a more challenging region for navigation, providing a suitable setting for testing advanced planning techniques.} 
    \vspace{-0.2cm}
    \label{fig:maritime_tss}
\end{figure}

\subsection{Maritime Navigation task}\label{app:maritime}
The Maritime Navigation task is based entirely on the ShipNaviSim framework~\cite{pham_shipnavisim_2025}, which was constructed using two years of historical AIS data from a high-traffic hotspot in the Singapore Strait. The planning region corresponds to the area with the highest vessel density and collision risk—where many ship routes intersect (Figure~\ref{fig:maritime_tss}).
ShipNaviSim uses AIS records containing both static information (vessel width, length, type, ID) and dynamic information (latitude, longitude, speed, heading, course-over-ground). Trajectories were generated from tankers and cargo vessels, which are considered the riskiest due to their large size (200–300 m) and limited maneuverability. All trajectories were interpolated at 10-second intervals.
The resulting dataset includes approximately 125,000 trajectories (about 14 million transitions), with average lengths of 100–150 steps. ShipNaviSim uses 80\% of this data for training and 20\% for evaluation.

\noindent\textbf{Observation Space}
ShipNaviSim~\cite{pham_shipnavisim_2025} defines the observation space from the perspective of the ego vessel. At each timestep, the ego vessel has access to a historical sequence of its own trajectory together with those of the 10 nearest surrounding ships, each observed over a configurable number of past steps. For every historical point, the available features include the $x$ and $y$ coordinates, speed $v$, and heading angle $h$ for both the ego and nearby vessels. The goal location of the ego vessel is also provided. This historical context is used to capture multi-ship interactions and inform navigation decisions. For simplicity, the original implementation processes the flattened observation vector using a standard neural network architecture.

\noindent\textbf{Action Space}
In ShipNaviSim, the action space is a 3-dimensional continuous “delta” action space, represented as $\langle d_x, d_y, d_h \rangle$, corresponding to changes in $x$, $y$, and heading $h$ between consecutive steps. The vessel’s speed at the next timestep is computed from the distance traveled, using $v_{t+1} = \sqrt{d_x^2 + d_y^2} / \delta_T$ with a timestep of $\delta_T = 10$ seconds. Because actions encode differences between consecutive states, a simple inverse-kinematics model can be used to infer actions directly from state-only trajectories in the dataset.

\noindent\textbf{Evaluation Metrics}
We evaluates navigation behavior using several vessel-specific metrics that compare agent trajectories with human expert data.

\noindent\textbf{Goal-Conditioned ADE (GC-ADE)}
GC-ADE measures the average 2D displacement error between the agent’s trajectory and the historical trajectory. For trajectories $\tau_m$ and $\tau_p$ of lengths $T_m$ and $T_p$, the metric is computed over $\min(T_m, T_p)$ steps:
\[
\text{GC-ADE} = \frac{1}{\min(T_m, T_p)}
\sqrt{
\sum_{t=1}^{\min(T_m, T_p)}
(x_t^m - x_t^p)^2 + (y_t^m - y_t^p)^2
}
\]

\noindent\textbf{Success Rate}
Success rate is defined as the percentage of episodes in which the ego vessel reaches its assigned goal, with success determined by entering a 200-meter radius around the goal.

\noindent\textbf{Close-quarters Rate}
Close-quarters rate measures the proportion of timesteps in an episode where the ego vessel comes within 3 cable lengths (555 meters) of another vessel—a distance regarded by domain experts as a close-quarters situation. This metric serves as a coarse indicator of traffic density and potential navigational risk; it does not necessarily imply that a collision was imminent.

\noindent\textbf{Reward.} The reward function is defined as a sparse signal: the agent receives a reward of $100$ only when it successfully reaches the goal. No intermediate rewards are provided, which encourages the agent to focus on completing the task while avoiding unsafe situations.

\noindent\textbf{Cost.} The cost function is defined to capture close-quarters situations. Specifically, for each transition, if the next state $s'$ results in the agent coming too close to another vessel (i.e., a close-quarters scenario), the cost is set to $c(s,a)=1$; otherwise, $c(s,a)=0$. This formulation directly penalizes unsafe maneuvers and encourages the agent to maintain safe distances from other vessels throughout its trajectory.

\subsection{Ablations on Key Hyperparameters}\label{app:hyperparam}
In our approach, the IQL instantiation is primarily governed by three hyperparameters: the expectile for training the cost critic ($\tauc$), the expectile for the reward critic ($\taur$), and the policy extraction temperature associated with AWR ($\beta$). Among these, the expectile values are particularly important for performance, as also observed in the original IQL work~\cite{kostrikov_offline_2021}. A lower cost critic expectile ($\tauc \le 0.5$) encourages a more conservative cost estimate that prioritizes minimizing cost, and our experiments indicate that better results are achieved with smaller $\tauc$. Conversely, the reward critic expectile is chosen to be $\taur \ge 0.5$, aligning with the recommendations from the IQL paper~\cite{kostrikov_offline_2021}. Ablation results for these expectile values are presented in Table~\ref{tab:expectile}.

\subsection{Feasible Action Set Analysis}\label{exp:budget}
 In Figure~\ref{fig:cost_value_dist} we present the persentatage of feasible actions to evaluate how restrictive or inclusive our learned feasible action space ie.
 Using the learned cost critic, we estimate the percentage of \( (s, a) \) pairs from the dataset that are feasible based on the remaining budget \( \ct \), given by \( \frac{\sum_{(s,a) \in \mathcal{D}} \mathbf{I}[\qcs(s,a) \leq \ct]}{|\mathcal{D}|} \). If the remaining budget is $\geq 15$, almost all the samples becomes feasible for most of the environments.  For example, CarPush1 and MataDrive-mediummean, feasibility reaches 99\% by budget 5, and ~65\% of the actions become feasible in AntCircle. This indicate the feasible action set is not over conservative. 

\begin{table*}[b]
    \centering

\begin{minipage}{0.48\textwidth}  

\resizebox{\textwidth}{!}{
\begin{tabular}{lrrrrrr}
\toprule
Expectile $\tau_r$ & \multicolumn{2}{c}{0.5} & \multicolumn{2}{c}{0.6} & \multicolumn{2}{c}{0.7} \\
 & cost $\downarrow$ & reward $\uparrow$ & cost $\downarrow$ & reward $\uparrow$ & cost $\downarrow$ & reward $\uparrow$ \\
\midrule
HopperVelocityGymnasium & 0.83 & 0.79 & 0.80 & 0.66 & 0.85 & 0.60 \\
PointPush1Gymnasium & 0.40 & 0.17 & 0.46 & 0.14 & 0.34 & 0.17 \\
PointPush2Gymnasium & 0.66 & 0.12 & 0.71 & 0.14 & 0.68 & 0.12 \\
SwimmerVelocityGymnasium & 0.82 & 0.36 & 0.98 & 0.49 & 0.99 & 0.38 \\
AntRun & 0.95 & 0.68 & 1.00 & 0.69 & 0.98 & 0.68 \\
BallRun & 0.07 & 0.20 & 3.30 & 0.24 & 2.43 & 0.22 \\
CarRun & 0.98 & 0.98 & 0.30 & 0.98 & 0.26 & 0.98 \\
DroneRun & 0.54 & 0.58 & 2.75 & 0.66 & 0.66 & 0.60 \\
\bottomrule
\end{tabular}
}
    \end{minipage}
    \hfill  
    \begin{minipage}{0.48\textwidth}  
\resizebox{\textwidth}{!}{
\begin{tabular}{lrrrrrr}
\toprule
Cost Expectile $\tau_c$ & \multicolumn{2}{c}{0.2} & \multicolumn{2}{c}{0.3} & \multicolumn{2}{c}{0.4} \\
 & cost $\downarrow$ & reward $\uparrow$ & cost $\downarrow$ & reward $\uparrow$ & cost $\downarrow$ & reward $\uparrow$ \\
\midrule
HopperVelocityGymnasium & 1.19 & 0.39 & 0.80 & 0.66 & 0.21 & 0.43 \\
PointPush1Gymnasium & 0.72 & 0.21 & 0.46 & 0.14 & 0.25 & 0.11 \\
PointPush2Gymnasium & 1.07 & 0.15 & 0.71 & 0.14 & 0.58 & 0.12 \\
SwimmerVelocityGymnasium & 2.19 & 0.30 & 0.98 & 0.49 & 1.19 & 0.57 \\
AntRun & 0.95 & 0.68 & 0.47 & 0.62 & 0.33 & 0.62 \\
BallRun & 0.07 & 0.20 & 3.09 & 0.06 & 3.79 & 0.01 \\
CarRun & 0.98 & 0.98 & 0.58 & 0.98 & 0.21 & 0.90 \\
DroneRun & 0.54 & 0.58 & 0.11 & 0.16 & 0.00 & -0.01 \\
\bottomrule
\end{tabular}
}
    \end{minipage}
    \caption{Ablation on the Reward Expectile $\taur$ Cost Expectile $\tauc$}
    \label{tab:expectile}
\end{table*}


\begin{figure}[tbh]
    \centering	
    \caption{\textbf{Percentage of Feasible Action Samples:}
    }
    \includegraphics[width=\linewidth]{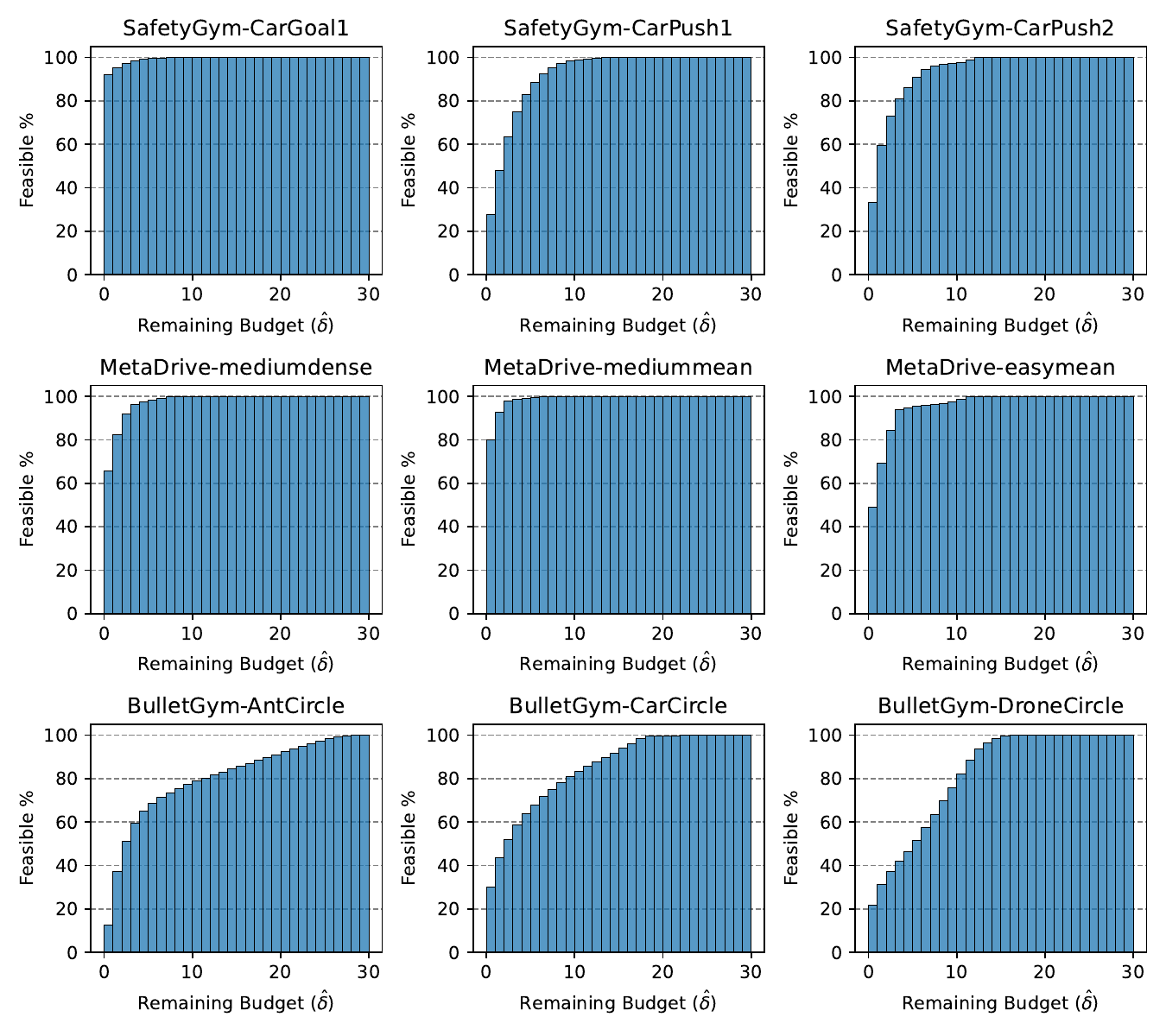}
        \vskip -5pt
    \label{fig:cost_value_dist}
\end{figure}

\subsection{Compare our results with the best policy obtained by LSPC}\label{app:lspc}
Table~\ref{tab:lspc-best-comparison} reports results using the best policy selected during training. In the original LSPC~\cite{koirala_latent_2025} work, policies were evaluated every 20k steps over a 1M-step training run, which allows selection of a high-performing policy but assumes access to a simulator—often unavailable in offline RL settings. To ensure a fair comparison, we reproduced this evaluation strategy for both our approach and LSPC, selecting and reporting the best-performing policy during training. As shown in the table, our method remains competitive under these conditions, with blue indicating the best results. All environments exhibit safe policies under this evaluation protocol.

\newcommand{\bluebold}[1]{\textbf{\textcolor{blue}{#1}}}
\begin{table}[tb]
    \centering
\begin{tabular}{lllll}
\hline
 & \multicolumn{2}{c}{LSPC-Best-Policy} & \multicolumn{2}{c}{BCAP-Best-Policy(ours)} \\
\textbf{Task} & reward $\uparrow$ & cost $\downarrow$ & reward $\uparrow$ & cost $\downarrow$ \\
\hline
PointButton1 & \textbf{0.10±0.19} & \textbf{0.72±1.17} & \bluebold{0.16±0.21} & \bluebold{0.81±1.07} \\
PointButton2 & \textbf{0.17±0.20} & \textbf{1.01±1.52} & \bluebold{0.23±0.26} & \bluebold{0.89±1.06} \\
PointCircle1 & \textbf{0.32±0.18} & \textbf{0.01±0.07} & \bluebold{0.47±0.19} & \bluebold{0.97±0.84} \\
PointCircle2 & \textbf{0.44±0.21} & \textbf{0.00±0.00} & \bluebold{0.54±0.19} & \bluebold{0.81±0.82} \\
PointGoal1 & \textbf{0.41±0.31} & \textbf{0.46±0.76} & \bluebold{0.68±0.17} & \bluebold{0.98±0.93} \\
PointGoal2 & \bluebold{0.21±0.23} & \bluebold{0.58±0.95} & \textbf{0.16±0.38} & \textbf{0.92±1.11} \\
PointPush1 & \textbf{0.14±0.17} & \textbf{0.42±0.79} & \bluebold{0.24±0.17} & \bluebold{0.45±0.59} \\
PointPush2 & \textbf{0.13±0.16} & \textbf{0.56±1.30} & \bluebold{0.20±0.17} & \bluebold{0.72±0.95} \\
CarButton1 & \textbf{-0.05±0.19} & \textbf{0.58±1.48} & \bluebold{0.07±0.15} & \bluebold{0.70±1.05} \\
CarButton2 & \bluebold{-0.13±0.24} & \bluebold{0.84±1.89} & \textbf{-0.07±0.25} & \textbf{0.72±1.23} \\
CarCircle1 & \textbf{0.34±0.14} & \textbf{0.22±0.52} & \bluebold{0.38±0.16} & \bluebold{0.53±0.55} \\
CarCircle2 & \textbf{0.36±0.21} & \textbf{0.59±1.49} & \bluebold{0.51±0.16} & \bluebold{0.98±1.02} \\
CarGoal1 & \textbf{0.29±0.27} & \textbf{0.36±0.64} & \bluebold{0.50±0.24} & \bluebold{0.73±0.97} \\
CarGoal2 & \textbf{0.15±0.21} & \textbf{0.53±1.20} & \bluebold{0.30±0.22} & \bluebold{0.89±1.00} \\
CarPush1 & \textbf{0.20±0.13} & \textbf{0.29±0.73} & \bluebold{0.26±0.14} & \bluebold{0.39±0.94} \\
CarPush2 & \textbf{0.09±0.16} & \textbf{0.66±1.09} & \bluebold{0.13±0.16} & \bluebold{0.73±1.00} \\
SwimmerVelocity & \bluebold{0.56±0.13} & \bluebold{0.26±0.63} & \textbf{0.31±0.25} & \textbf{0.86±1.88} \\
HopperVelocity & \textbf{0.24±0.02} & \textbf{0.34±0.39} & \bluebold{0.72±0.16} & \bluebold{0.61±0.45} \\
HalfCheetahVelocity & \bluebold{0.99±0.02} & \bluebold{0.59±0.73} & \textbf{0.96±0.02} & \textbf{0.87±0.85} \\
Walker2dVelocity & \textbf{0.80±0.05} & \textbf{0.29±0.45} & \bluebold{0.80±0.01} & \bluebold{0.06±0.16} \\
AntVelocity & \textbf{0.98±0.05} & \textbf{0.29±0.23} & \bluebold{0.99±0.01} & \bluebold{0.68±0.40} \\
\hline
BallRun & \bluebold{0.21±0.03} & \bluebold{0.00±0.00} & \textbf{0.19±0.23} & \textbf{0.32±0.67} \\
CarRun & \textbf{0.98±0.00} & \textbf{0.44±0.46} & \bluebold{0.99±0.00} & \bluebold{0.72±0.43} \\
DroneRun & \bluebold{0.59±0.03} & \bluebold{0.40±1.16} & \textbf{0.58±0.01} & \textbf{0.31±0.60} \\
AntRun & \textbf{0.43±0.13} & \textbf{0.48±0.49} & \bluebold{0.72±0.01} & \bluebold{0.59±0.42} \\
BallCircle & \textbf{0.49±0.16} & \textbf{0.04±0.09} & \bluebold{0.76±0.05} & \bluebold{0.93±0.43} \\
CarCircle & \bluebold{0.73±0.03} & \bluebold{0.13±0.32} & \textbf{0.72±0.05} & \textbf{0.68±0.62} \\
DroneCircle & \textbf{0.54±0.04} & \textbf{0.72±0.61} & \bluebold{0.54±0.04} & \bluebold{0.69±0.51} \\
AntCircle & \textbf{0.44±0.14} & \textbf{0.45±0.82} & \bluebold{0.60±0.12} & \bluebold{0.77±0.77} \\
\hline
easysparse & \bluebold{0.72±0.00} & \bluebold{0.54±0.29} & \textbf{0.70±0.19} & \textbf{0.94±0.17} \\
easymean & \textbf{0.72±0.01} & \textbf{0.53±0.29} & \bluebold{0.73±0.17} & \bluebold{0.94±0.13} \\
easydense & \bluebold{0.71±0.02} & \bluebold{0.53±0.29} & \textbf{0.70±0.17} & \textbf{0.90±0.16} \\
mediumsparse & \bluebold{0.95±0.07} & \bluebold{0.10±0.07} & \textbf{0.94±0.07} & \textbf{0.82±0.32} \\
mediummean & \textbf{0.91±0.11} & \textbf{0.44±0.29} & \bluebold{0.92±0.11} & \bluebold{0.77±0.16} \\
mediumdense & \bluebold{0.90±0.12} & \bluebold{0.37±0.35} & \textbf{0.78±0.28} & \textbf{0.68±0.31} \\
hardsparse & \bluebold{0.53±0.10} & \bluebold{0.46±0.29} & \textbf{0.49±0.13} & \textbf{0.76±0.44} \\
hardmean & \bluebold{0.52±0.09} & \bluebold{0.30±0.19} & \textbf{0.46±0.15} & \textbf{0.84±0.54} \\
harddense & \bluebold{0.52±0.10} & \bluebold{0.37±0.23} & \textbf{0.44±0.17} & \textbf{0.63±0.24} \\
\hline
\hline
\end{tabular}
    \caption{Performance of the best policy evaluated during training.}
    \label{tab:lspc-best-comparison}
\end{table}

\subsection{Grid World Environment}\label{app:grid-world}
Figure~\ref{fig:grid-world} shows the grid-world environment used for the results in Figure~\ref{fig:synthetic}. The action space consists of moving left, right, up, or down. The transition probability of the intended movement $p$ is the probability of moving in the chosen direction, with the remaining probability distributed among the other three actions. We vary $p$ and the budget to evaluate our method.

\begin{figure}[ht]
    \centering
    \includegraphics[width=0.3\linewidth]{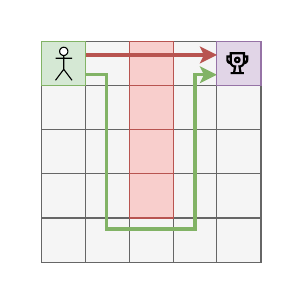}
    \caption{\textbf{Grid-World Environment:} The agent starts at the green cell and receives a reward of $-1$ per step. Reaching the target green cell terminates the episode. We vary the transition probability of the intended movement to make navigation more difficult and stochastic, resulting in a non-zero expected cost.}
    \label{fig:grid-world}
\end{figure}

\subsection{Approximated Persistent Safety}

In this section, we examine how sensitive our approach is to the quality of the learned persistent safety set. We approximate this set by learning the cost critic (cost-value function). In this experiment, we early-stop cost-critic training at 20\%, 50\%, 80\%, and 100\% of the training steps, then evaluate the results shown in Figure~\ref{tab:partial-cost-critic}. The results show that reducing the cost-critic quality leads to slightly higher constraint violations, but the effect is minor.

\begin{table}[tb]
\centering
\begin{tabular}{
>{\columncolor[HTML]{FFFFFF}}l 
>{\columncolor[HTML]{FFFFFF}}l 
>{\columncolor[HTML]{FFFFFF}}r 
>{\columncolor[HTML]{FFFFFF}}r 
>{\columncolor[HTML]{FFFFFF}}r 
>{\columncolor[HTML]{FFFFFF}}r }
\hline
 &  & \multicolumn{4}{l}{\cellcolor[HTML]{FFFFFF}\textit{\% of cost critic learning}} \\ \hline
\textit{task} & \textit{} & {\color[HTML]{000000} 20\%} & {\color[HTML]{000000} 50\%} & {\color[HTML]{000000} 80\%} & {\color[HTML]{000000} 100\%} \\
CarCircle1 & reward & 0.68 & 0.55 & 0.58 & 0.50 \\
 & cost & {\color[HTML]{C53929} 2.79} & {\color[HTML]{C53929} 1.22} & {\color[HTML]{C53929} 1.72} & 0.50 \\
CarCircle2 & reward & 0.59 & 0.60 & 0.51 & 0.47 \\
 & cost & {\color[HTML]{C53929} 2.49} & {\color[HTML]{C53929} 3.36} & {\color[HTML]{C53929} 1.03} & 0.53 \\
CarGoal1 & reward & 0.45 & 0.39 & 0.36 & 0.39 \\
 & cost & 0.60 & 0.49 & 0.45 & 0.39 \\
CarGoal2 & reward & 0.21 & 0.17 & 0.16 & 0.17 \\
 & cost & 0.74 & 0.61 & 0.55 & 0.49 \\
HalfCheetahVelocity & reward & 0.97 & 0.94 & 0.92 & 0.92 \\
 & cost & 0.70 & 0.64 & 0.77 & 0.63 \\
HopperVelocity & reward & 0.47 & 0.67 & 0.47 & 0.61 \\
 & cost & {\color[HTML]{C53929} 1.09} & 0.95 & 0.86 & 0.75 \\ \hline
\end{tabular}
\caption{Ablation by partially learning cost critic}\label{tab:partial-cost-critic}
\end{table}

\section{Extended Related Work}\label{app:related_work}
Constrained Markov decision processes (CMDPs) \cite{altman_constrained_2021} are among the most widely adopted frameworks for safe RL. Numerous methods have been proposed to address the CMDP setting in offline RL. A simple approach is behavior cloning using only safe trajectories~\cite{liu_datasets_2024}; CDT~\cite{liu_constrained_2023} extends the Decision Transformer~\cite{chen_decision_2021} to handle action constraints. The DSRL benchmark~\cite{liu_datasets_2024} provides standardized datasets, benchmarks, and implementations. BCQ-Lagrangian~\cite{liu_datasets_2024} employs a Lagrangian approach to solve constrained problems; however, this method suffers from instability due to the difficulty of tuning the Lagrange multiplier, as it jointly optimizes both reward and cost objectives.

The DICE (Distribution Correction Estimation) framework has gained significant traction in offline learning due to its ability to directly estimate stationary distribution corrections, avoiding the compounding errors in value-based bootstrapping. This approach was initially developed for off-policy evaluation and offline policy optimization, with notable algorithms such as DualDICE~\cite{nachum_dualdice_2019} and OptiDICE~\cite{lee_optidice_2021}. Beyond standard offline RL, matching state-transition occupancies has proven highly effective in offline imitation learning (IL), enabling robust policy learning from demonstrations without explicit reward signals. Methods like DemoDICE~\cite{kim_demodice_2022} and SafeDICE~\cite{jang2023safedice} leverage this framework for IL, and more recently, IOSTOM~\cite{pham2025iostom} introduced transition occupancy matching specifically for imitation learning from observations. Building upon these principles for the constrained RL setting, the COptiDICE algorithm~\cite{lee_coptidice_2022} tackles the offline CMDP problem via stationary distribution correction estimation. Meanwhile, Constrained Penalized Q-learning (CPQ)~\cite{xu_constraints_2022} formulates the problem as a min-max objective, which also leads to optimization instability. Despite their theoretical motivation, these methods have not demonstrated strong performance in prior safe offline RL benchmarks~\cite{liu_datasets_2024}. Recently proposed CAPS algorithm~\cite{chemingui_constraint-adaptive_2025} learns multiple policies and switches between them based on the remaining cost budget.

Recent work~\cite{koirala_latent_2025} learns a variational autoencoder (VAE)~\cite{kingma_introduction_2019} to approximate the most conservative policy, and then trains a separate policy in the VAE’s latent space to maximize the reward. \cite{guo_constraint-conditioned_2025} improves upon this by conditioning the VAE on the remaining budget and sampling from it during training. Both approaches utilize variational autoencoders to learn generative models as a pre-training step, which incurs significant computational overhead.

Hamilton-Jacobi (HJ) reachable sets are useful for safety verification and policy supervision~\cite{ganai_hamilton-jacobi_2024}. In RL, feasible sets depend on the policy and change during training, causing instability. RCRL~\cite{yu_reachability_2022} uses a Lagrangian approach, while FISOR~\cite{zheng_safe_2024} enforces hard constraints independently of the policy. Our method conditions feasibility on discounted future costs, enabling policy-independent estimation while handling cumulative soft constraints in CMDPs.
Alternatively, action-constrained RL enforces per-step feasibility using optimization solvers~\cite{lin_escaping_2021} or generative models~\cite{brahmanage_flowpg_2023,brahmanage_leveraging_2024}.

A related approach is Saute RL~\cite{sootla_saute_2022}, which enforces hard per-step constraints to satisfy the total cost limit. Instead of estimating future costs to select feasible actions, it labels timesteps by cumulative cost and requires an online RL setting to track the remaining budget during rollouts. In contrast, our approach prunes unsafe actions offline without requiring online samples.

\section{Implementation Details}
\subsection{Implementation Details CMDP}\label{app:cmdp-extension}
Given a CMDP, we first solve the MDP to minimize the budget and obtain solutions for \(\vcs\) and \(\qcs\). These solutions are then used to construct \ourmdp. To keep the state space finite and discrete, we discretize the budget and add it as an additional dimension to the state space. The transition function is updated based on the \ourmdp\ transition function \(\bT\). The budget update functions \(f,g\) are real-valued functions, which must be discretized to be compatible with the discrete budget, in that case we can round the real budget to the nearest budget-bin. States outside the persistent safety set \(\SbP\) are masked, and the action space is restricted to the support of \(\PiP\). The resulting extended MDP can then be solved using a standard LP solver.
In the policy extraction, we can only include the actions in the persistently safe actions $\Ap(s,\ct)$. Also, we will learn the policy for states outside the persistent safety set $s\notin \SbP$ to minimize the cost in case the agent leaves the persistent state set $\SbP$, mainly because of the approximation errors caused by discretization.

\begin{algorithm}
\caption{Construction and LP Solution of the Extended MDP}
\label{alg:extended_mdp}
\begin{algorithmic}[1]
\Require CMDP \( \mc{M} := \langle S, A, \T, r, c, \gamma, \mu_0, \cta \rangle \), budget update function $\ctupdate$
\State Compute cost value functions $\vcs$ and $\qcs$ by minimizing cost in $\mc{M}$ ({Standard LP formulation})
\State Construct discrete budget space $B$ via binning over $[0, \ctmax]$
\State Modified state space of BDMDP $\bS:= S\times B$ and persistent state space $\SbP:= \{ (s,\ct) \in \bS| \vcs(s) \leq \ct \}$
\State Define the persistent action space $\Ap(s,\ct):= \{ a \in A | \qcs(s,a) \leq \ct \}  \quad \forall (s,\ct)\in\bS$
\State Define the extended transition function:  \( \bT\big((s',\ct') \mid (s,\ct), a\big) := \T(s' \mid s,a)\cdot \1_{\{\ct' = \ctupdate(s,a,s',\ct)\}} \)
\State Solve the extended MDP via LP, only considering the persistent state set $\SbP$ and persistent action set $\Ap$ to obtain $Q_R^*$
\State Extract policy: $\pi(s,\ct) \leftarrow
\begin{cases}
    \displaystyle \arg\max_{a \in \Ap(s,\ct)} Q_R^*(s,a), & \text{if } s \in \SbP \\
    \displaystyle \arg\min_{a} \qcs(s,a), & s\notin \SbP
\end{cases}$
\State \Return $\pi$
\end{algorithmic}
\end{algorithm}

\subsection{Implementation Details Offlien RL}
In the main results, we include CDT, CAPS, LSPC, and CCAC as baselines. All baselines are evaluated using the same cost thresholds and trained with the same number of gradient steps (SafetyGym and BulletGym: 100k steps; MetaDrive: 200k steps) as proposed in the DSRL benchmark and the final policy is used for evaluation. Details are specified in Table~\ref{tab:hyperparameters}. For baselines where standard deviation is not reported in the published paper (CAPS), or they are not tested on some DSRL instances (CCAC), we run them again using their author provided code and hyperparams. The CDT results are taken from the published paper as it is tested on all domains with standard deviations.

For LSPC, we use the official source code and hyper-parameters from the original authors~\cite{koirala_latent_2025}. The original work trained the model for 1M steps, evaluated policy every 20k steps, and reported the \textit{best policy'}s results after 1M training steps. However, this training method is not standard, as during policy training, access to simulator is not assumed in standard DSRL baselines. To ensure a fair comparisons across all the methods, instead we train the LSPC model with the same number of gradient steps as the other baselines and report results from the final policy in the Main Paper Table~\ref{tab:main-table}.  For completeness sake, we also compare the best policies of both LSPC and our method over 1M training in Table~\ref{tab:lspc-best-comparison}.

\paragraph{Our Algorithm Details} Our algorithm was trained for the same number of gradient steps, using the exact hyperparameters listed in Table~\ref{tab:hyperparameters}. These hyperparameters were selected through manual experimentation, guided by Bayesian hyperparameter search, over value ranges explored in the ablation study presented in the experimental section.
The complete implementation of our algorithm is available in the Supplementary Material. Hyperparameters for different domains are presented in Table~\ref{tab:hyperparameters}, and the training algorithm is detailed in Algorithm~\ref{alg:siql}. Though our algorithm consists of three isolated loops (cost critic learning, reward critic learning, and policy extraction), since each loop depends only on its own previous steps, they can be implemented in a single loop with gradient updates. We follow this approach to keep things simple since it allow us to evaluate the policy while training.

\textbf{Metrics.} We use the standard metrics proposed in the DSRL benchmark~\cite{liu_datasets_2024}, namely \emph{normalized reward} and \emph{normalized cost}. These are computed using the following formulas:
\begin{align}
    R_{\text{normalized}} &= \frac{R_{\pi} - R_{\min}}{R_{\max} - R_{\min}}, \quad
    C_{\text{normalized}} = \frac{C_{\pi} - C_{\min}}{C_{\max} - C_{\min}}
\end{align}
Here, $R_{\min}$ and $R_{\max}$ denote the minimum and maximum empirical returns, respectively, while $C_{\min}$ and $C_{\max}$ denote the minimum and maximum empirical costs of the particular task.

\textbf{Evaluation:} The algorithm for evaluating a policy using the trained budget-conditioned policy is provided in Algorithm~\ref{alg:siql-eval}. 
Since the task is in-practice finite horizon, we adjust the cost budget based on the maximum horizon as in the Algorithm~\ref{alg:siql-eval}.
It calculates the budget required based on the remaining steps and remaining budget. This idea was used in previous works~\cite{chemingui_constraint-adaptive_2025}.

\textbf{Run Time.} We benchmarked the training and evaluation time of our algorithm on a server equipped with an NVIDIA GeForce RTX 3090 GPU. The results are reported in Table~\ref{tab:run_time}. Our method completes in just a few minutes, offering a noticeable speed-up compared to several baseline methods, which may take 2–3 hours to train under identical hardware and conditions. While we acknowledge that runtime comparisons are inherently approximate, the results suggest that our approach has the potential for substantial improvements in efficiency.

\begin{table*}[ht]
  \centering
  \begin{tabular}{cccc}
    \hline
     &   BulletGym   & SafetGym  &   MetaDrive     \\
    \hline
    Average Training Time   &   7.1 minutes    &   8.5 minutes   &   18.4 minutes    \\
    Evaluation Time    &   1.1 minute    &   4.1 minutes    &   10.8 minutes    \\
    Training Steps   &   100k    &   100k    &   200k    \\
    \hline
  \end{tabular}
  \caption{Average Training and evaluation time per instance (by evaluating on 3 distinct cost thresholds, 20 episodes)}
  \label{tab:run_time}
\end{table*}

\begin{table*}[htb]
  \centering
  \begin{tabular}{lccc}
    \hline
    \textbf{Hyperparameters} &   \textbf{BulletGym}   & \textbf{SafetGym}  &   \textbf{MetaDrive}     \\
    \textbf{Common}
    \\
    Training Steps   &   100k    &   100k    &   200k    \\
    Policy network architecture   &   [512, 512] MLP    &   [512, 512] MLP    &   [512, 512] MLP  \\
   Cost thresholds for main results &  (10, 20, 40)   &   (20, 40, 80)    &   (10, 20, 40)    \\
    Seeds for main results & (0, 10, 20) & (0, 10, 20) & (0, 10, 20)\\
    \hline
    \hline
    \textbf{BCR} \\
    Algorithm Version  &  Stochastic    &   Stochastic    &   Deterministic      \\
    Batch Size  &  512    &   512    &   512      \\
    Optimizer   &   Adam    &   Adam    &   Adam    \\
    Learning Rates $\lambda_Q^C, \lambda_V^C, \lambda_Q^R, \lambda_V^R, \lambda_{\pi_C}, \lambda_{\pi_R}$   &   0.0003    &   0.0003    &   0.0003    \\
    Target Update $\alpha$  &  0.005    &   0.005    &   0.005      \\
    Discount factors for cost and reward $\gamma_{C},\gamma$  &  0.99    &   0.99    &   0.99     \\
    Expectile for cost (IQL) $\tauc$  &  0.2    &   0.3    &   0.4     \\
    Expectile for reward (IQL) $\taur$  &  0.5    &   0.6    &   0.6     \\
    Budget sample count $n$  &  1    &   1    &   1     \\
    Policy dropout  &  0.1    &   0.1    &   0.1     \\
    AWR Temperature for both cost and reward (IQL) $\beta$ &  3    &   3    &   8     \\
    \hline
    \hline
  \end{tabular}
  \caption{Default Hyperparameters of \our\ for tasks in three environments.}
  \label{tab:hyperparameters}
\end{table*}

\begin{algorithm}[tb]
\small
   \caption{\small \our : Execution (for finite horizon)}
   \label{alg:siql-eval}
    \begin{algorithmic}
   \State Input: environment \emph{Env}, cost threshold $\cta$ and maximum episode length of the environment $H$
   
   \State $c \leftarrow 0$ \Comment{Initialize total cost to zero}
   \State $t \leftarrow 0$ \Comment{Initialize total timesteps to zero}
   
   \State $s \leftarrow$ \emph{Env}.reset() ;  \ 
    \emph{done} $\leftarrow$ False
   \Repeat
   \State $\ct \leftarrow \frac{\cta - c}{(1-\gamma)}\frac{1-\gamma^{H-t}}{H-t}$ \Comment{Calculate the dynamic budget constraint to enforce based on the remaining steps and remaining budget for finite horizon tasks}
   \State $a \sim \pi((s, \ct))$
   \State $s', r(s,a), c(s,a),\ $\emph{done} $\leftarrow$ \emph{Env}.step($a$)
   \State $s \leftarrow s'$
   \Until{\emph{done}}
\end{algorithmic}
\end{algorithm}


\end{document}